%% file: icml2021.tex
\documentclass{article}
\usepackage{microtype}
\usepackage{graphicx}
\usepackage{subfigure}
\usepackage{booktabs}
\usepackage{hyperref}

\usepackage[accepted]{icml2021}
\usepackage[algo2e,algoruled,ruled,boxed,lined]{algorithm2e}

\input{math_commands.tex}

\input{defs.tex}

\hypersetup{colorlinks=true,linkcolor=red,urlcolor=magenta,menucolor=red}
\usepackage{url}
\usepackage{enumerate}
\usepackage{enumitem}
\usepackage{wrapfig}
\usepackage[export]{adjustbox}
\usepackage{comment}
\usepackage[colorinlistoftodos]{todonotes}
\usepackage{boldline}
\usepackage{cleveref}
\crefname{equation}{Eq.}{}
\usepackage{xspace}
\xspaceaddexceptions{]\}}
\usepackage[export]{adjustbox}
\usepackage{amsfonts}       
\usepackage{nicefrac}       
\usepackage{eqparbox}
\renewcommand\algorithmiccomment[1]{%
  \hfill\#\ \eqparbox{COMMENT}{#1}%
}
\icmltitlerunning{Shortest-Path Constrained Reinforcement Learning for Sparse Reward Tasks}
\begin{document}
\twocolumn[
\icmltitle{Shortest-Path Constrained Reinforcement Learning for Sparse Reward Tasks}
\icmlsetsymbol{equal}{*}
\begin{icmlauthorlist}
\icmlauthor{Sungryull Sohn}{equal,umich,lg}
\icmlauthor{Sungtae Lee}{equal,yonsei}
\icmlauthor{Jongwook Choi}{umich}
\icmlauthor{Harm van Seijen}{ms}
\icmlauthor{Mehdi Fatemi}{ms}
\icmlauthor{Honglak Lee}{lg,umich}
\end{icmlauthorlist}
\icmlaffiliation{umich}{University of Michigan}
\icmlaffiliation{yonsei}{Yonsei University}
\icmlaffiliation{ms}{Microsoft Research}
\icmlaffiliation{lg}{LG AI Research}
\icmlcorrespondingauthor{Sungryull Sohn}{srsohn@umich.edu}
\icmlkeywords{Deep Reinforcement Learning, ICML}
\vskip 0.3in
]
\printAffiliationsAndNotice{\icmlEqualContribution}
\input{0_abstract}
\input{1_intro}
\input{2_preliminary}
\input{3_problem}
\input{4_method_two_column}
\input{5_related}
\input{6_experiment}

\input{7_conclusion}
\bibliography{icml2021}
\bibliographystyle{icml2021}
\newpage
\onecolumn
\appendix
\begin{center}
{\Large \textbf{ Appendix: %
Shortest-Path Constrained Reinforcement Learning for Sparse Reward Tasks
\\[10pt]
}}
\end{center}
\input{8_appendix_ablation}

\input{8_appendix_rnet}

\input{8_appendix_bonus}
\input{8_appendix_experiment_detail}
\clearpage
\input{8_appendix_option_view}

\input{8_appendix_proof_single_goal}
\input{8_appendix_proof_theorems}

\input{8_appendix_relwork}
\bigskip
\end{document}

%% file: math_commands.tex
\usepackage{mathtools,amsthm,thmtools,amsmath,amsfonts,bm,amssymb}

\newtheorem{theorem}{Theorem}
\newtheorem{definition}{Definition}

\newtheorem{lemma}[theorem]{Lemma}

\declaretheoremstyle[
notefont=\bfseries, notebraces={}{},
bodyfont=\normalfont\itshape,
headformat=\NAME \NOTE
]{nopar}
\declaretheorem[style=nopar,name=Theorem]{theorem*}
\declaretheorem[style=nopar,name=Lemma]{lemma*}
\declaretheorem[style=nopar,name=Corollary]{corollary*}

\newcommand{\nn}{\nonumber}
\newcommand{\mb}{\mathbf}
\newcommand{\tb}{\textbf}

\newcommand{\mbb}{\mathbb}
\newcommand{\mc}{\mathcal}

\DeclareRobustCommand\onedot{\futurelet\@let@token\@onedot}
\def\onedot{.}
\def\eg{\emph{e.g}\onedot} 
\def\ie{\emph{i.e}\onedot} 
 
\def\etc{\emph{etc}\onedot}

\newcommand{\cutsectionup}{\vspace*{-0.05in}}
\newcommand{\cutsectiondown}{\vspace*{-0.05in}}

\newcommand{\cutparagraphup}{\vspace{-7pt}}
\newcommand{\cutparagraphdown}{\vspace{-5pt}}

\newcommand{\cutitemizedown}{\vspace{-0pt}}

\newcommand{\cutproofup}{\vspace{-5pt}}
\newcommand{\cutproofdown}{\vspace{-2pt}}

\newcommand{\cutdefinitiondown}{\vspace{-5pt}}

\newcommand{\biggmid}{\,\bigg|\,}

\def\eqref#1{equation~\ref{#1}}

\def\1{\bm{1}}

\DeclareMathAlphabet{\mathsfit}{\encodingdefault}{\sfdefault}{m}{sl}
\SetMathAlphabet{\mathsfit}{bold}{\encodingdefault}{\sfdefault}{bx}{n}

\DeclareMathOperator*{\argmax}{arg\,max}
\DeclareMathOperator*{\argmin}{arg\,min}

%% file: defs.tex
\usepackage{color}
\usepackage[normalem]{ulem}
\usepackage{booktabs}

\newcommand{\Bo}[1]{}
\usepackage{marginnote}

\def\gpathsetpi{\mathcal{T}^\pi_{s, s'}}

\def\nr{\mathrm{nr}}
\def\ir{\mathrm{IR}}

\def\pathsetpi{\mathcal{T}^\pi_{s, s', \nr}}
\def\pathsetpiphi{\mathcal{T}^\pi_{\Phi, \nr}}
\def\pathsetoptpiphi{\mathcal{T}^{\pi^*}_{\Phi, \nr}}
\def\distpi{D_{\nr}^\pi(s, s')}
\def\dist{D_{\nr}(s, s')}

\def\pathsethathat{\widehat{\mathcal{T}}_{\hat{s}\rightarrow \hat{s}'}}%

\def\nrstate{\mc{S}^\ir}
\def\nrstatepair{\Phi^{\pi}}
\def\Pispss{\Pi^{\text{SP}}_{s\rightarrow s'}}
\def\Pisp{\Pi^{\text{SP}}}
\def\Piksp{\Pi^{\text{SP}}_k}
\def\Pimsp{\Pi^{\text{SP}}_m}
\def\Piinfsp{\Pi^{\text{SP}}_\infty}
\def\trans{P^{\pi}_{s, s'}}

\def\transshort{P^{*}_{s, s'}}

\def\transopthat{P^{\pi^*}_{\hat{s}, \hat{s}'}}
\def\transshorthat{P^{*}_{\hat{s}, \hat{s}'}}

\def\ksp{$k$-SP\xspace}
\def\csp{C^{\text{SP}}\xspace}

\usepackage{graphicx,scalerel}
\newcommand\wh[1]{\hstretch{2}{\hat{\hstretch{.5}{#1}}}}
\def\hatcsp{\wh{C}^{\text{SP}}}

\def\laststir{s_{\text{ir}}'}
\def\initstir{s_{\text{ir}}}

\def\rand{\textbf{Random}}
\def\icm{\textbf{ICM}}
\def\eco{\textbf{ECO}}
\def\ppo{\textbf{PPO}}
\def\oracle{\textbf{GT-Grid}{}}
\def\ucb{\textbf{GT-UCB}{}}
\def\sprl{\textbf{SPRL}{}}

\def\atari{\textit{Atari}\xspace}
\def\dmlab{\textit{DeepMind Lab}\xspace}
\def\grid{\textit{MiniGrid}\xspace}
\def\fetch{\textit{Fetch}\xspace}

\def\montezuma{\textit{Montezuma's Revenge}\xspace}
\def\freeway{\textit{Freeway}\xspace}
\def\mspacman{\textit{Ms.Pacman}\xspace}
\def\gravitar{\textit{Gravitar}\xspace}
\def\hero{\textit{HERO}\xspace}
\def\seaquest{\textit{Seaquest}\xspace}

\def\fpnp{\textit{FetchPickAndPlace-v1}\xspace}
\def\fpush{\textit{FetchPush-v1}\xspace}
\def\freach{\textit{FetchReach-v1}\xspace}
\def\fslide{\textit{FetchSlide-v1}\xspace}

\def\fours{\textit{FourRooms-7$\times$7}\xspace}
\def\fourl{\textit{FourRooms-11$\times$11}\xspace}
\def\keys{\textit{KeyDoors-7$\times$7}\xspace}
\def\keyl{\textit{KeyDoors-11$\times$11}\xspace}

\def\dmgs{\textit{GoalSmall}\xspace}
\def\dmgl{\textit{GoalLarge}\xspace}
\def\dmom{\textit{ObjectMany}\xspace}

\def\mbb{\mathbb}
\def\tb{\textbf}

\def\defeq{:=}

\renewcommand{\widehat}{\hat}

\def\eqref#1{equation~\ref{#1}}

\def\1{\bm{1}}

\DeclareMathAlphabet{\mathsfit}{\encodingdefault}{\sfdefault}{m}{sl}
\SetMathAlphabet{\mathsfit}{bold}{\encodingdefault}{\sfdefault}{bx}{n}

\definecolor{darkgreen}{rgb}{0,0.5,0}

%% file: 0_abstract.tex
\begin{abstract}
We propose the $k$-Shortest-Path (\ksp{}) constraint: a novel constraint on the agent's trajectory that improves the sample efficiency in sparse-reward MDPs.
We show that any optimal policy necessarily satisfies the \ksp{} constraint. Notably, the \ksp{} constraint prevents the policy from exploring state-action pairs along the non-\ksp{} trajectories (\eg, going back and forth). However, in practice, excluding state-action pairs may hinder the convergence of RL algorithms. To overcome this, we propose a novel cost function that penalizes the policy violating SP constraint, instead of completely excluding it.
Our numerical experiment in a tabular RL setting demonstrates that the SP constraint can significantly reduce the trajectory space of policy. As a result, our constraint enables more sample efficient learning by suppressing redundant exploration and exploitation.
Our experiments on \grid{}, \dmlab{}, \atari, and \fetch{} show that the proposed method significantly improves proximal policy optimization (PPO) and outperforms existing novelty-seeking exploration methods including count-based exploration even in continuous control tasks, indicating that it improves the sample efficiency by preventing the agent from taking redundant actions. \footnote{The code is available on https://github.com/srsohn/shortest-path-rl}
\end{abstract}

%% file: 1_intro.tex
\vspace*{-0pt}
\cutsectionup
\section{Introduction}\label{sec:i}
\cutsectiondown

Recently, deep reinforcement learning (RL) has achieved a large number of breakthroughs in many domains including video games~\citep{mnih2015human, vinyals2019grandmaster}, and board games~\citep{silver2017mastering}.
Nonetheless, a central challenge in reinforcement learning (RL) is the sample efficiency~\citep{kakade2003sample}; it has been shown that the RL algorithm requires a large number of samples for successful learning in MDP with large state and action space. 
Moreover, the success of the RL algorithm heavily hinges on the quality of collected samples; the RL algorithm tends to fail if the collected trajectory does not contain enough evaluative feedback (\eg, sparse or delayed reward). 

To circumvent this challenge, planning-based methods utilize the environment's model to improve or create a policy instead of interacting with the environment. 
Recently, combining the planning method with an efficient path search algorithm, such as Monte-Carlo tree search (MCTS)~\citep{norvig2002modern, coulom2006efficient}, has demonstrated successful results~\citep{guo2016deep, vodopivec2017monte, silver2017mastering}.
However, such tree search methods would require an accurate model of MDP and the complexity of planning may grow intractably large for a complex domain.
Model-based RL methods attempt to learn a model instead of assuming that model is given, but learning an accurate model also requires a large number of samples, which is often even harder to achieve than solving the given task. 
Model-free RL methods can be learned solely from the environment reward, without the need of a (learned) model. However, both value-based and policy-based methods suffer from poor sample efficiency, especially in sparse-reward tasks.
To tackle the sparse reward problem, researchers have proposed to learn an intrinsic bonus function that measures the novelty of the state that agent visits ~\citep{schmidhuber1991adaptive,oudeyer2009intrinsic,pathak2017curiosity,savinov2018episodic,choi2018contingency,burda2018exploration}.
However, when such an intrinsic bonus is added to the reward, it often requires a careful balancing between environment reward and bonus and scheduling of the bonus scale to guarantee the convergence to an optimal solution. 

\begin{figure*}[t]
    \centering
    \vspace*{-4pt}
    \includegraphics[draft=false, width=0.92\linewidth, valign=t]
    {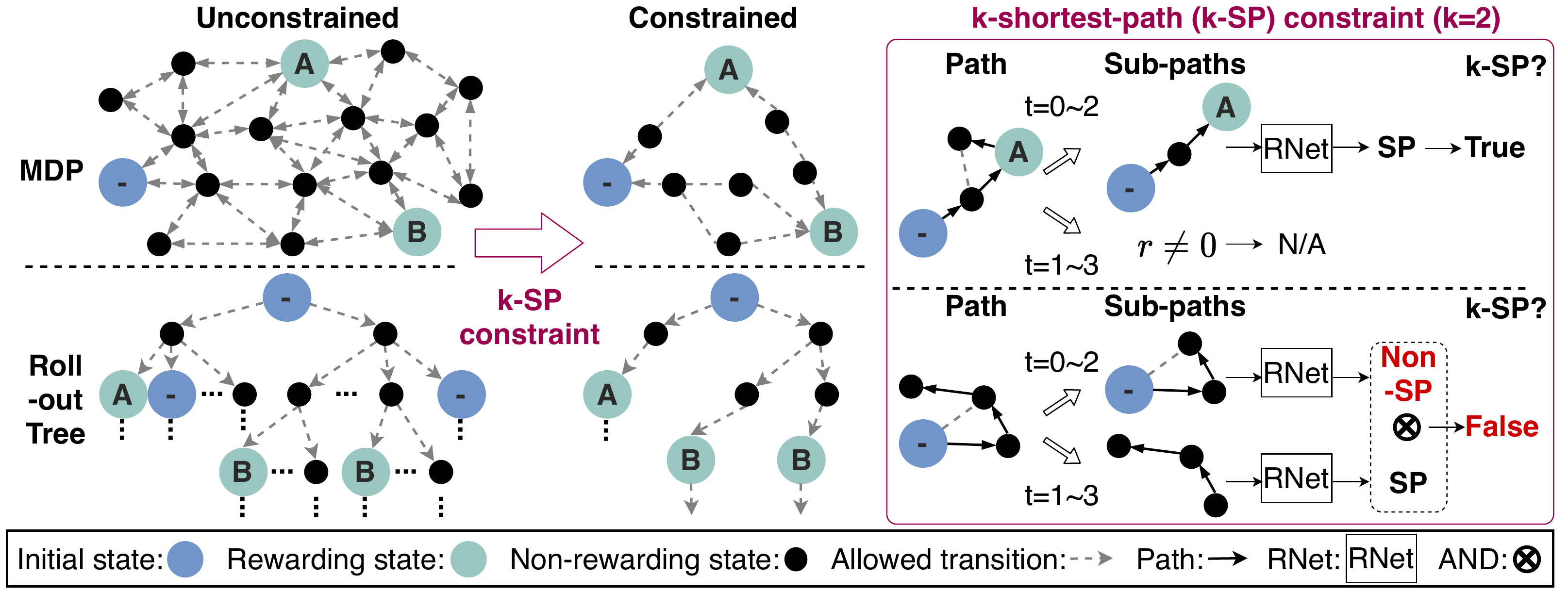}
    \vspace*{-3pt}
    \caption{
        The $k$-SP constraint improves the sample efficiency of %
        RL methods in sparse-reward tasks by pruning out suboptimal trajectories from the trajectory space.
        Intuitively, the $k$-SP constraint means that when a policy rolls out into trajectories, all of the sub-paths of length $k$ is the shortest path (under a distance metric defined in terms of policy, discount factor, and transition probability; see %
        \Cref{sec:method-k-shortest-path}
        for the formal definition). 
        \tb{(Left)} MDP and a rollout tree are given.
        \tb{(Middle)} The paths that satisfy the $k$-SP constraint. The number of admissible trajectories is drastically reduced.
        \tb{(Right)} A path rolled out by a policy satisfies the $k$-SP constraint if all sub-paths of length $k$ are shortest paths and have not received a non-zero reward. We use a reachability network to test if a given (sub-)path is the shortest path (See \Cref{sec:m} for details).
    }
    \label{fig:policy_tree}
    \vspace*{-10pt}
\end{figure*}

To tackle the aforementioned challenge of sample efficiency in sparse reward tasks, we introduce a constrained-RL framework that improves the sample efficiency of any model-free RL algorithm in sparse-reward tasks%
, under the mild assumptions on MDP (see \Cref{appendix:sp-optimal}).
Of note, though our framework will be formulated for policy-based methods,
our final form of the cost function (\Cref{eq:ksp-problem-4} in~\Cref{sec:m}) applies to both policy-based and value-based methods.
We propose a novel $k$-\emph{shortest-path}
(\ksp{}) constraint (\Cref{def:k-shortest-path-constraint}) that improves the sample efficiency of policy learning (See~\Cref{fig:policy_tree}).
The \ksp{} constraint is applied to a trajectory rolled out by a policy; all of its sub-path of length $k$ is required to be a shortest-path under the \emph{$\pi$-distance} metric which we define in~\Cref{sec:shortest-path-general}.
We prove that applying our constraint preserves the optimality for any MDP~(\Cref{thm:ksp-optimal}), except the stochastic and multi-goal MDP which requires additional assumptions.
We relax the hard constraint into a soft cost formulation~\citep{tessler2019reward}, and use a \emph{reachability network}~\citep{savinov2018episodic} (RNet) to efficiently learn the cost function in an off-policy manner.

We summarize our contributions as the following:
\textbf{(1)} %
We propose a novel constraint that can improve the sample efficiency of any model-free RL method in sparse reward tasks.
\textbf{(2)} %
We present several theoretical results including the proof that our proposed constraint preserves the optimal policy of given MDP.
\textbf{(3)} %
We present a numerical result in tabular RL setting to precisely evaluate the effectiveness of the proposed method.
\textbf{(4)} %
We propose a practical way to implement our proposed constraint and demonstrate that it provides a significant improvement on four complex deep RL domains.
\textbf{(5)} %
We demonstrate that our method significantly improves the sample efficiency of PPO, and outperforms existing novelty-seeking methods on four complex domains in sparse reward settings.

%% file: 2_preliminary.tex
\cutsectionup
\section{Preliminaries}\label{sec:pre}
\cutsectiondown

\label{sec:background-mdp}
\paragraph{Markov Decision Process (MDP).}
We model a task as an MDP tuple $\mathcal{M}=(\mc{S},\mc{A},\mc{P},\mc{R},\rho ,\gamma)$, where  $\mc{S}$ is a state set, $\mc{A}$ is an action set, $\mc{P}$ is a transition probability, $\mc{R}$ is a reward function, $\rho$ is an initial state distribution, and $\gamma \in [0,1)$ is a discount factor. 
For each state $s$, the value of a policy $\pi$ is denoted by $V^{\pi}(s)=\mbb{E}^{\pi}{ \left[ \sum_{t} \gamma^t r_t \mid s_0 = s \right] }.$
Then, the goal is to find the optimal policy $\pi^*$ that maximizes the expected return: %
\vspace*{-5pt}
\begin{align}
	\pi^*
	&=\argmax_{\pi} \mbb{E}^{\pi}_{s\sim\rho}{ \Big[ \textstyle\sum_{t} \gamma^t r_t \mid s_0 = s \Big] }\\ 
	&= \argmax_{\pi} \mbb{E}_{s\sim\rho}\left[ V^\pi(s) \right].
\end{align}
\vspace*{-20pt}

\paragraph{Constrained MDP.}
A constrained Markov Decision Process (CMDP) is an MDP with extra constraints that restrict the domain of allowed policies~\citep{altman1999constrained}.
Specifically, CMDP introduces a constraint function $C(\pi)$ that maps a policy to a scalar, and a threshold $\alpha\in\mbb{R}$.
The objective of CMDP is to maximize the expected return $R(\tau) = \sum_t \gamma^t r_t$ of a trajectory $\tau=\{s_0, a_0, r_1, s_1, a_1, r_2, s_2, \ldots \}$ subject to a constraint:
$ %
    \pi^{*} =\textstyle\argmax_{\pi} \mbb{E}_{\tau\sim\pi}\left[ R(\tau) \right],\text{ s.t. }C(\pi) \leq \alpha.
$ %
A popular choice of constraint is based on the transition cost function~\citep{tessler2019reward} $c(s, a, r, s')\in\mbb{R}$ which assigns a scalar-valued cost to each transition. Then the constraint function for a policy $\pi$ is defined as the discounted sum of the cost under the policy:
$ %
C(\pi) = \mbb{E}_{\tau\sim\pi}\left[ \textstyle\sum_t \gamma^t c(s_t, a_t, r_{t+1}, s_{t+1})\right]. 
$ %
In this work, we propose a \emph{shortest-path} constraint, that provably preserves the optimal policy of the original unconstrained MDP, while reducing the trajectory space.
We will use a cost function-based formulation to implement our constraint (see \Cref{sec:s} and~\ref{sec:m}).  %

%% file: 3_problem.tex
\section{Formulation: $k$-shortest-path Constraint}\label{sec:s}
\cutsectiondown

We define the $k$-shortest-path (k-SP) constraint to remove redundant transitions (\eg, unnecessarily going back and forth), leading to faster policy learning.
We show two important properties of our constraint:
(1) the optimal policy is preserved, and
(2) the policy search space is reduced.

In this work, we limit our focus to MDPs satisfying $R(s)+\gamma V^*(s) > 0$ for all initial states $s \ s.t \ \rho(s) > 0$ and all rewarding states that optimal policy visits with non-zero probability $s\in \{s| r(s)\neq 0 , p_{\pi^*}(s)>0\}$ where $p_{\pi}(s)$ is a probability of visiting state $s$ with policy $\pi$.
We exploit this mild assumption to prove that our constraint preserves optimality.
Intuitively, we exclude the case when the optimal strategy for the agent is at best choosing a ``lesser of evils'' (\ie, largest but negative value) which often still means a failure.
We note that this is often caused by unnatural reward function design; in principle, we can avoid this by simply offsetting the reward function by a constant $-\lvert \min_{s\in\{s| p_{\pi^*}(s)\}} V^*(s)\rvert$ for every transition, assuming the policy is \textit{proper}%
\footnote{It is an instance of potential-based reward shaping which has optimality guarantee~\citep{ng1999policy}.}.
Goal-conditioned RL~\citep{nachum2018data}, most of the well-known domains such as \textit{Atari}~\citep{bellemare2013arcade}, \dmlab{}~\citep{beattie2016deepmind}, \grid{}~\citep{gym_minigrid}, etc., satisfy this assumption. Also, for general settings with stochastic MDP and multi-goals, we require additional assumptions to prove the optimality guarantee (See~\Cref{appendix:sp-optimal} for details).

\subsection{Shortest-path Policy and Shortest-path Constraint}
\label{sec:shortest-path-general}
\vspace*{-2pt}

Let $\tau$ be a \emph{path} defined by a sequence of states: $\tau = \{s_0,\ldots,s_{\ell(\tau)}\}$,
where $\ell(\tau)$ is the \emph{length} of a path $\tau$ (\ie, $\ell(\tau) = |\tau| - 1$).
We denote the set of all paths from $s$ to $s'$ by $\mathcal{T}_{s,s'}$.
A path $\tau^*$ from $s$ to $s'$ is called a \emph{shortest path} from $s$ to $s'$ if $\ell(\tau^*)$ is minimum,
\ie, $\ell(\tau^*) = \smash{\min_{\tau \in \mathcal{T}_{s,s'}} \ell(\tau)}$.

Now we will define similar concepts (length, shortest path, \etc) \emph{with respect to} a policy.
Intuitively, a policy that rolls out shortest paths (up to some stochasticity)
to a goal state or between any state pairs should be a counterpart.
We consider a set of all admissible paths from $s$ to $s'$ under a policy $\pi$:

\begin{definition}[Path set]\label{def:path-set-pi} %
    $\gpathsetpi=\{\tau \mid s_0=s, s_{\ell(\tau)}=s', p_\pi(\tau)>0, s_t\neq s'\text{ for }\forall t<\ell(\tau) \}$.
That is, %
$\gpathsetpi$ is a set of all paths that policy $\pi$ may roll out
from $s$ and terminate once visiting $s'$.
\end{definition}

If the MDP is a single-goal task, \ie, there exists a unique (rewarding) goal state $s_g \in \mathcal{S}$
such that $s_g$ is a terminal state, and $R(s) > 0$ if and only if $s = s_g$,
any shortest path from an initial state to the goal state is the optimal path with the highest return $R(\tau)$,
and a policy that always rolls out a shortest path from an initial state to the goal state is therefore optimal (see \Cref{lem:g-sp-optimal}).\footnote{
We refer the readers to \Cref{appendix:g-sp-specialcase} for more detailed discussion and proofs for single-goal MDPs.  }
This is because all states except for $s_g$ are non-rewarding states,
but in general MDPs this is not necessarily true.
However, this motivates us to limit the domain of the shortest path to
\emph{non-rewarding states}.
We define \emph{non-rewarding paths} from $s$ to $s'$ as follows:
\begin{definition}[Non-rewarding path set]
\label{def:path-set-pi-nr} %
$\pathsetpi=\{\tau \mid \tau \in\gpathsetpi, r_{t}=0\text{ for  }\forall t<\ell(\tau) \}$.
\end{definition}
\cutdefinitiondown
In words, $\smash{\pathsetpi}$ is a set of all non-rewarding paths from $s$ to $s'$
rolled out by policy $\pi$ (\ie, $\tau\in\smash{\gpathsetpi}$) without any associated reward except the last step (\ie, $r_{t}=0\text{ for  }\forall t<\ell(\tau) $).
Now we are ready to define a notion of length with respect to a policy
and shortest path policy:
\begin{definition}[$\pi$-distance from $s$ to $s'$]
\label{def:ex-path-len}
$\distpi=\log_\gamma \big(
    \mathbb{E}_{\tau \sim \pi :~ \tau \in \pathsetpi} \left[ \gamma^{\ell(\tau)} \right]  
\big)$
\end{definition}
\begin{definition}[Shortest path distance from $s$ to $s'$]\label{def:shortest-dist}
$\dist=\min_{\pi}\distpi$.
\end{definition}
\cutdefinitiondown
We define $\pi$-distance to be the log-mean-exponential of the length $\ell(\tau)$ of non-rewarding paths
$\tau \in \smash{\pathsetpi}$. To be thorough, $pi$-distance is not a "distance" but a quasi-metric since by definition, $pi$-distance is asymmetric. 
When there exists no admissible path from $s$ to $s'$ under policy $\pi$,
the path length is defined to be $\infty$: $\smash{\distpi}=\infty$ if $\smash{\pathsetpi} = \emptyset$.
We note that when both MDP and policy are deterministic, $D^\pi(s, s')$ recovers
the natural definition of path length, $\distpi = \ell(\tau)$.

We call a policy a \emph{shortest-path policy} from $s$ to $s'$ if it rolls out a path with the smallest $\pi$-distance:
\begin{definition}[Shortest path policy from $s$ to $s'$]
\label{def:sp-policy} %
$\pi \in \Pispss=\{\pi \in \Pi  \mid \distpi=\dist \}$.
\end{definition}
\cutdefinitiondown\vspace*{1pt}
Finally, we will define the shortest-path (SP) constraint.
Let $\nrstate=\{s \mid R(s)>0\text{ or }\rho(s)>0 \}$ be the union of all initial and rewarding states, and  $\smash{\nrstatepair=\{(s, s') \mid s, s' \in \nrstate, \rho(s)>0, \pathsetpi\neq\emptyset\}}$ be the subset of $\smash{\nrstate}$ such that agent may roll out.
Then, the SP constraint is applied to the non-rewarding sub-paths between states in $\nrstatepair$: $\pathsetpiphi=\bigcup_{(s, s')\in\nrstatepair}{\pathsetpi}$. 
We note that these definitions are used in the proofs (\Cref{appendix:sp-optimal}).
Now, we define the shortest-path constraint as follows:
\begin{definition}[Shortest-path constraint] \label{def:shortest-path-constraint}
A policy $\pi$ satisfies the shortest-path (SP) constraint if %
$\pi \in \Pisp$, where
$
\Pisp =
\{
    \pi \mid \text{For all~} s, s'\in\pathsetpiphi, 
    \text{it holds~}
    \pi\in\Pispss
\}.
$
\end{definition}
\cutdefinitiondown\vspace*{1pt}

Intuitively, the SP constraint forces a policy to transition between initial and rewarding states via shortest paths. %
The SP constraint would be particularly effective in sparse-reward settings, where the distance between rewarding states is large.

Given these definitions, we can show that an optimal policy indeed satisfies the SP constraint
in a general MDP setting.
In other words, the shortest path constraint should not change optimality:
\begin{theorem}\label{thm:sp-optimal}
For any MDP, an optimal policy $\pi^*$ satisfies the shortest-path constraint: $\pi^*\in\Pisp$.
\end{theorem}
\cutproofup
\vspace*{-3pt}
\begin{proof} See Appendix~\ref{appendix:sp-optimal} for the proof. \end{proof}
\cutproofdown

\subsection{Relaxation: $k$-shortest-path Constraint}
\label{sec:method-k-shortest-path}
Implementing the shortest-path constraint is, however, intractable since it requires a distance predictor $\dist$. Note that the distance predictor addresses the optimization problem,
which might be as difficult as solving the given task.
To circumvent this challenge, we consider its more tractable version, namely a \emph{$k$-shortest path} constraint, which reduces the shortest-path problem $\dist$ to a binary decision problem --- is the state $s'$ reachable from $s$ within $k$ steps? --- also known as $k$-reachability~\citep{savinov2018episodic}.
The $k$-shortest path constraint is defined as follows:
\vspace*{-1em}
\begin{definition}[$k$-shortest-path constraint]
\label{def:k-shortest-path-constraint}
A policy $\pi$ satisfies the $k$-shortest-path constraint if %
$\pi \in \Piksp$, where
\begin{align}
\Piksp =
\{
    \pi \mid &\text{For all~} s, s'\in\pathsetpiphi,  \distpi \leq k,\nonumber
    \\
    &\text{it holds~}
    \pi\in\Pispss
\}.\label{def:ksp}
\end{align}
\end{definition}
\vspace*{-0.5em}
Note that the SP constraint (\Cref{def:shortest-path-constraint}) is relaxed by adding a condition $\distpi\leq k$.
In other words, the \ksp{} constraint is imposed only for $s,s'$-path whose length is not greater than $k$.
From \Cref{def:ksp}, we can prove an important property and then \Cref{thm:ksp-optimal} (optimality):
\begin{lemma}\label{lem:ksp-reduce}
For an MDP $\mc{M}$, $\Pimsp\subset\Piksp$ if $k < m$.
\end{lemma}
\cutproofup
\begin{proof} It is true since $\{(s,s') \mid \distpi\leq k\} \subset \{(s,s') \mid \distpi\leq m\}$ for $k<m$. \end{proof}
\cutproofdown
\begin{theorem}\label{thm:ksp-optimal}
For an MDP $\mc{M}$ and any $k\in\mbb{R}$, an optimal policy $\pi^*$ is a $k$-shortest-path policy.
\end{theorem}
\cutproofup
\begin{proof} 
\Cref{thm:sp-optimal} tells $\pi^* \in \Pisp$.
\Cref{def:ksp} tells $\Pisp = \Piinfsp$ and
\Cref{lem:ksp-reduce} tells $\Piinfsp \subset \Piksp$.
Collectively, we have
$\pi^*   \in     \Pisp  =      \Piinfsp      \subset       \Piksp.$
\end{proof}
\vspace*{-0.8em}
In conclusion, \Cref{thm:ksp-optimal} states that %
the \ksp{} constraint does not change the optimality of policy,
and \Cref{lem:ksp-reduce} states a larger $k$ results in a larger reduction in policy search space.
Thus, it motivates us to apply the \ksp{} constraint in policy search to more efficiently find an optimal policy.
For the numerical experiment on measuring the 
reduction in the policy roll-outs space, please refer to \Cref{sec:analysis}.

%% file: 4_method_two_column.tex
\cutsectionup
\section{Shortest-Path Reinforcement Learning (SPRL)}\label{sec:m}
\cutsectiondown
\begin{figure*}[h]
    \centering
    \vspace*{-2pt}
    \subfigure[]{
    \includegraphics[draft=false, height=0.19\linewidth, valign=b]
    {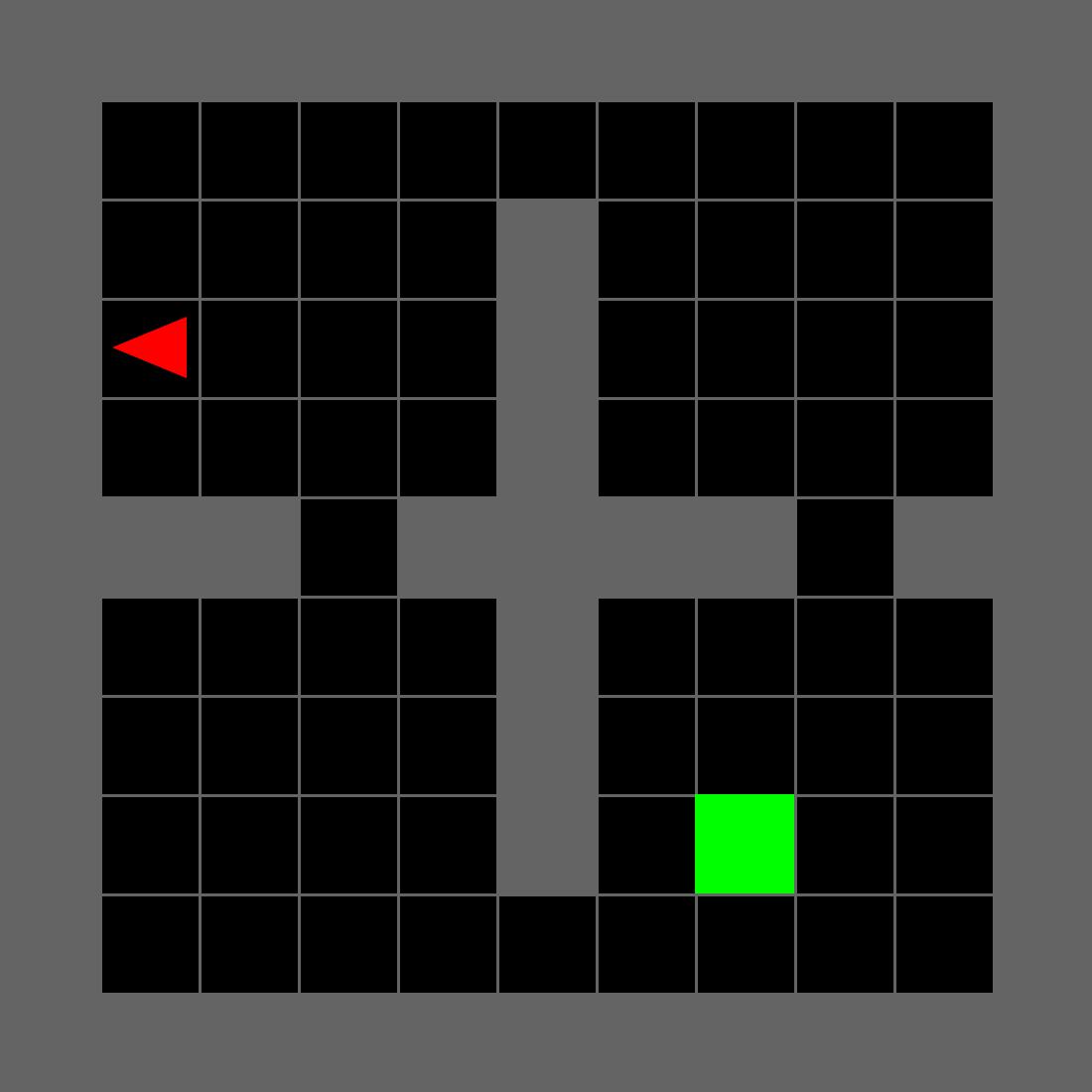}}
    \hspace{+3pt}
    \subfigure[]{
    \includegraphics[draft=false, height=0.19\linewidth, valign=b]{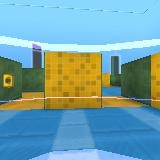}}
    \hspace{+3pt}
    \subfigure[]{
    \includegraphics[draft=false, height=0.19\linewidth, valign=b]{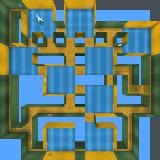}}
    \hspace{+3pt}
    \subfigure[]{
    \includegraphics[draft=false, height=0.19\linewidth, valign=b]{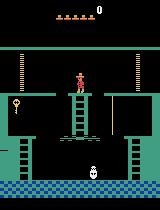}}
    \hspace{+3pt}
    \subfigure[]{
    \includegraphics[draft=false, height=0.19\linewidth, valign=b]{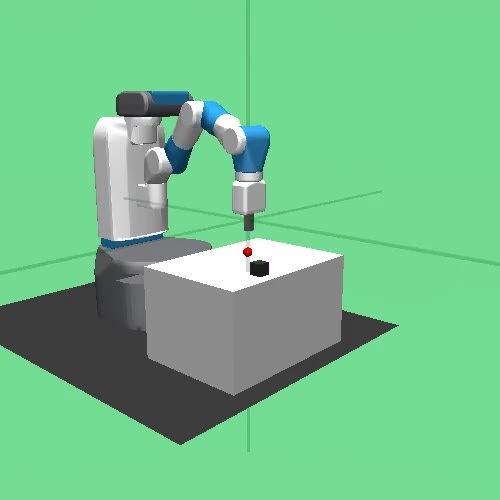}}
    \vspace*{-10pt}
    \caption{
        An example observation of (a) \fourl{}, %
        (b)~\dmgl{} in \dmlab{}, (c) the maze layout (not available to the agent) of \dmgl{}, (d) \montezuma in \atari, and (e) \fpush{} in \fetch{}.
    }
    \label{fig:dmlab_maze_layout}
    \vspace*{-10pt}
\end{figure*}
In~\Cref{sec:s}, we defined our \ksp constraint and proved that it preserves the optimality. In this section, we derive how the \ksp constraint on policy can be estimated by the \ksp cost that can be computed from the agent's on-policy trajectory. The proposed intrinsic cost term can be subtracted from the extrinsic reward and optimized together via any model-free RL method.
\cutparagraphup
\paragraph{$k$-shortest-path Cost.}
The objective of RL with the \ksp{} constraint $\Piksp$ can be written as:
\begin{align}
    \pi^* &= \textstyle\argmax_{\pi} \mbb{E}^{\pi}\left[ R(\tau) \right],\ \  \text{s.t.}\ \ \pi\in\Piksp,\label{eq:ksp-problem}
\end{align}
where 
$\Piksp = \{\pi \mid \forall (s, s'\in\pathsetpiphi),  \distpi \leq k, \text{it holds~}\pi\in\Pispss\}$ (\Cref{def:k-shortest-path-constraint}). 
We want to formulate the constraint $\pi \in \Piksp$ in the form of constrained MDP~(\Cref{sec:background-mdp}),
\ie, as $C(\pi) \le \alpha$.
We begin by re-writing the \ksp{} constraint into a cost-based form: %
\begin{align}
    \hspace*{-5pt}
    &\Piksp = \{\pi \mid \csp_k(\pi) = 0\}, \text{ where }\\ 
    &\csp_k(\pi)=\sum_{(s,s'\in\pathsetpiphi):\distpi\leq k} \mbb{I}\left[ \dist < \distpi \right].
    \label{eq:cost-form1}
\end{align}
Note that $\mbb{I}\left[ \dist < \distpi \right]=0 \leftrightarrow \dist = \distpi$ since $\dist\leq \distpi$ from~\Cref{def:shortest-dist}.
Similar to~\citet{tessler2019reward}, we apply the constraint to the on-policy trajectory $\tau=(s_0, s_1, \ldots)$ with discounting by replacing $(s, s')$ with $(s_t, s_{t+l})$ where $[t, t+l]$ represents each segment of $\tau$ with length $l$:
\begin{align}
\csp_k(\pi)
&\simeq  \mbb{E}_{\tau\sim\pi} \left[ \csp_k(\tau) \right], \\
\csp_k(\tau)
&= \textstyle\sum_{(t, l): t\geq 0, l\leq k}~
    \gamma^t
    \cdot \left( \prod_{j=t}^{t+l-1}{\mathbb{I} \left[
    r_j=0\right]} \right) \nonumber\\
&\qquad \cdot \mathbb{I} \left[ D_\nr(s_t, s_{t+l})< D^\pi_\nr(s_t, s_{t+l})  \right] \label{eq:csp-before}\\
&\leq
    \textstyle\sum_{(t, l): t\geq 0, l\leq k}~
    \gamma^t
    \cdot \left( \prod_{j=t}^{t+l-1}{\mathbb{I} \left[
    r_j=0\right]} \right) \nonumber\\
&\qquad \cdot\mathbb{I} \left[
	    D_\nr(s_t, s_{t+l}) < k
	\right]\label{eq:cmdp-derivation-cont}\\
&\triangleq \hatcsp_k(\pi),
\end{align}
where the Inequality.~(\ref{eq:cmdp-derivation-cont}) holds because $D^\pi_\nr(s_t, s_{t+l})=\log_\gamma\left( \mathbb{E}_{\tau \in \mathcal{T}_{s_t, s_{t+l},\nr}^{\pi} }\left[ \gamma^{|\tau|} \right]\right)<k$ from Jensen's inequality.
Note that it is sufficient to consider only the cases $l = k$
(because for $l < k$, given $D_\nr(s_t, s_{t+k}) < k$, we have $D(s_t, s_{t+l}) \le l < k$).
Then, we simplify $\hatcsp_k(\tau)$ as
\begin{align}
&\hatcsp_k(\tau) %
= \textstyle\sum_{t}
    \gamma^t \mathbb{I} \left[D_\nr(s_t, s_{t+k}) < k \right] \prod_{j=t}^{t+k-1}{\mathbb{I} \left[
    r_j=0\right]}\\
&=\textstyle\sum_{t}
    \gamma^t 
    \mathbb{I} [t \ge k]\mathbb{I} \left[
        D_\nr(s_{t-k}, s_t) < k
    \right] \prod_{j=t-k}^{t-1}{\mathbb{I} \left[
    r_j=0\right]}.
    \label{eq:cmdp-derivation}
\end{align}
Finally, the per-time step cost $c_t$ is given as:
\begin{align}
    c_t = \mathbb{I} [t \ge k]
    \cdot \mathbb{I} \left[
	    D_\nr(s_{t-k}, s_t) < k
	\right]
	\cdot \prod_{j=t-k}^{t-1}{\mathbb{I} \left[
    r_j=0\right]},
	\label{eq:cost}
\end{align}
where $\hatcsp_k(\pi) = \mathbb{E}_{\tau \sim \pi}\left[ \sum_t \gamma^t c_t \right]$.
Note that $\hatcsp_k(\pi)$ is an upper bound of $\csp_k(\pi)$,
which will be minimized by the bound to make as little violation of the shortest-path constraint as possible.
Intuitively speaking, $c_t$ penalizes the agent from taking a non-$k$-shortest path at each step, so minimizing such penalties will make the policy satisfy the $k$-shortest-path constraint. 
In \Cref{eq:cost}, $c_t$ depends on the previous $k$ steps; hence, the resulting CMDP becomes a $(k+1)$-th order MDP. In practice, however, we empirically found that feeding only the current time-step observation to the policy performs better than stacking the previous $k$-steps of observations (See \Cref{appendix:stack} for details). Thus, we did not stack the observation in all the experiments.
We use the Lagrange multiplier method to convert the objective (\ref{eq:ksp-problem}) into an equivalent unconstrained problem as follows:
\begin{align}
	\min_{\lambda>0}\max_{\theta} L(\lambda, \theta) = \min_{\lambda>0}\max_{\theta}\mbb{E}_{\tau\sim\pi_{\theta}}\Big[ \textstyle\sum_t{\gamma^t \left( r_t - \lambda c_t\right)} \Big],   \label{eq:ksp-problem-3}
\end{align}
where $L$ is the Lagrangian, $\theta$ is the parameter of policy $\pi$, and $\lambda> 0$ is the Lagrangian multiplier. While in ordinary Lagrange multiplier method $\lambda$ is unique since Theorem~\ref{thm:ksp-optimal} shows that the shortest-path constraint preserves the optimality, we can control how much weight we will give to the constraint, i.e., we are free to set any $\lambda>0$.

Thus, we simply consider  $\lambda$ as a tunable \emph{positive} hyperparameter, and  simplify the min-max problem~(\ref{eq:ksp-problem-3}) to an RL objective with costs $c_t$ being added:
\vspace*{-3pt}
\begin{align}
	\max_{\theta}\mbb{E}_{\tau\sim\pi_{\theta}}\Big[ \textstyle\sum_t{ \gamma^t \left( r_t - \lambda c_t\right)} \Big].\label{eq:ksp-problem-4}
\end{align}

\begin{figure*}[!h]
    \centering
    \hspace*{-5pt}
    \includegraphics[draft=false, height=7.3\baselineskip, valign=t]{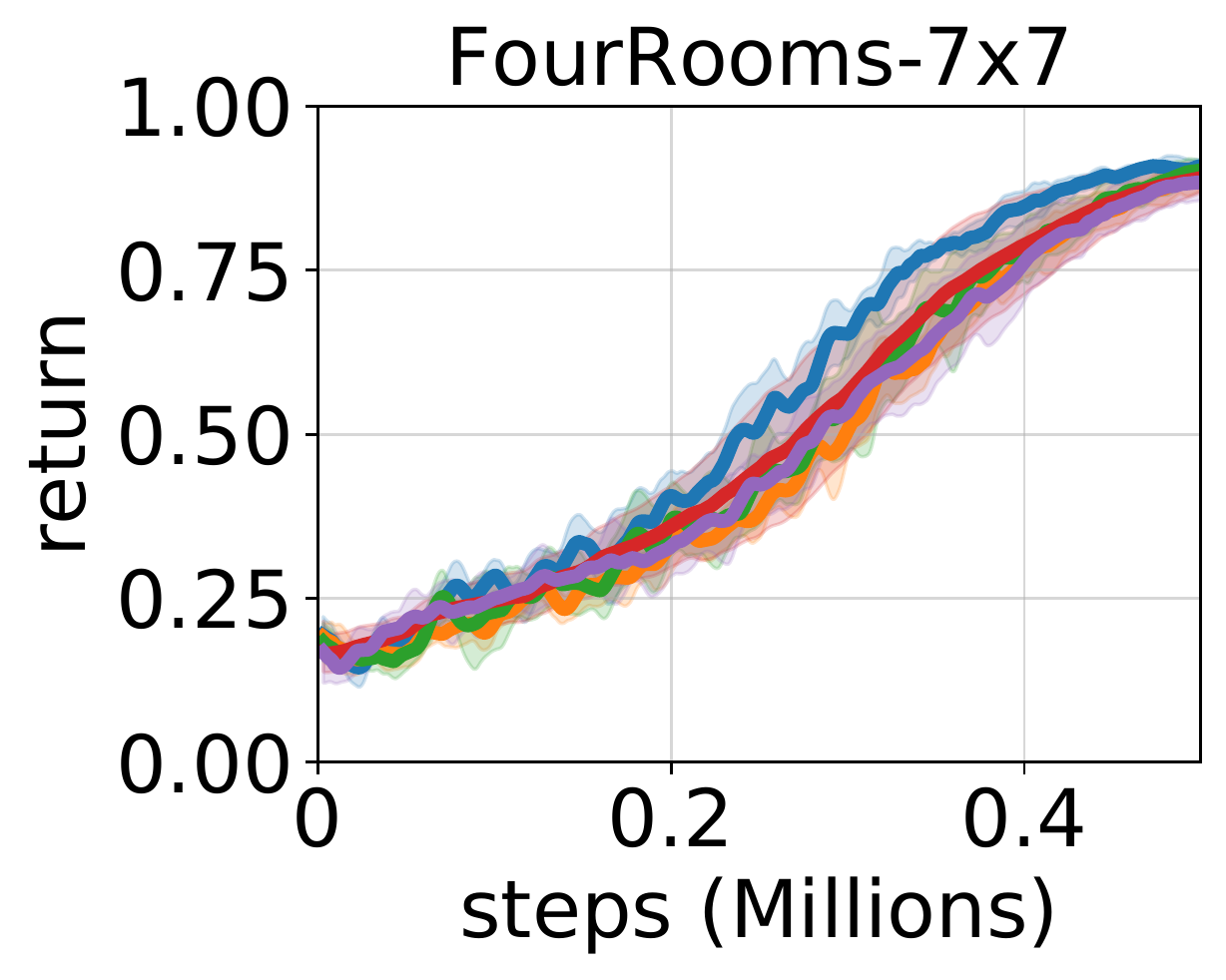}
    \hspace{-6pt}
    \includegraphics[draft=false, height=7.3\baselineskip, valign=t]{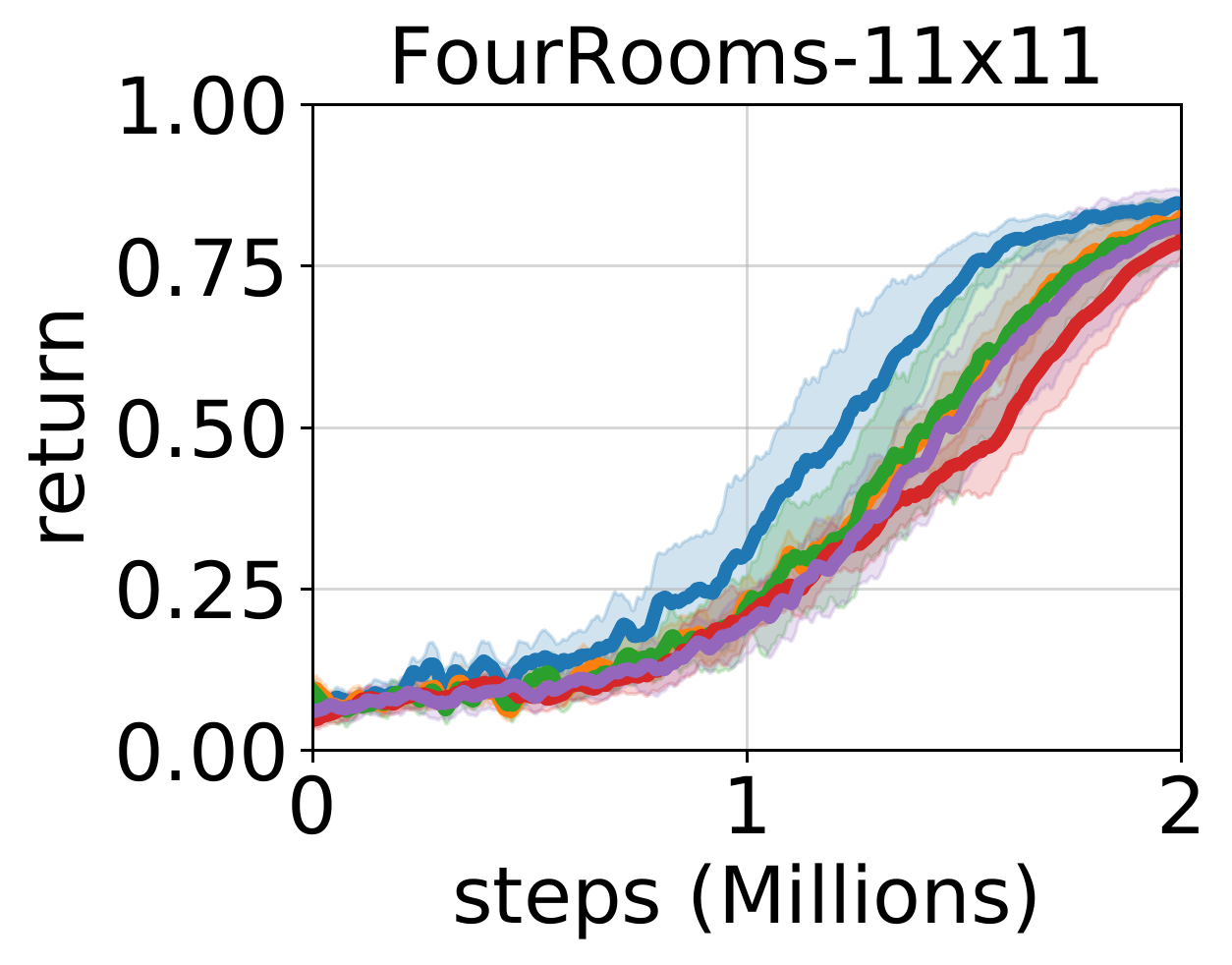}
    \hspace{-6pt}
    \includegraphics[draft=false, height=7.3\baselineskip, valign=t]{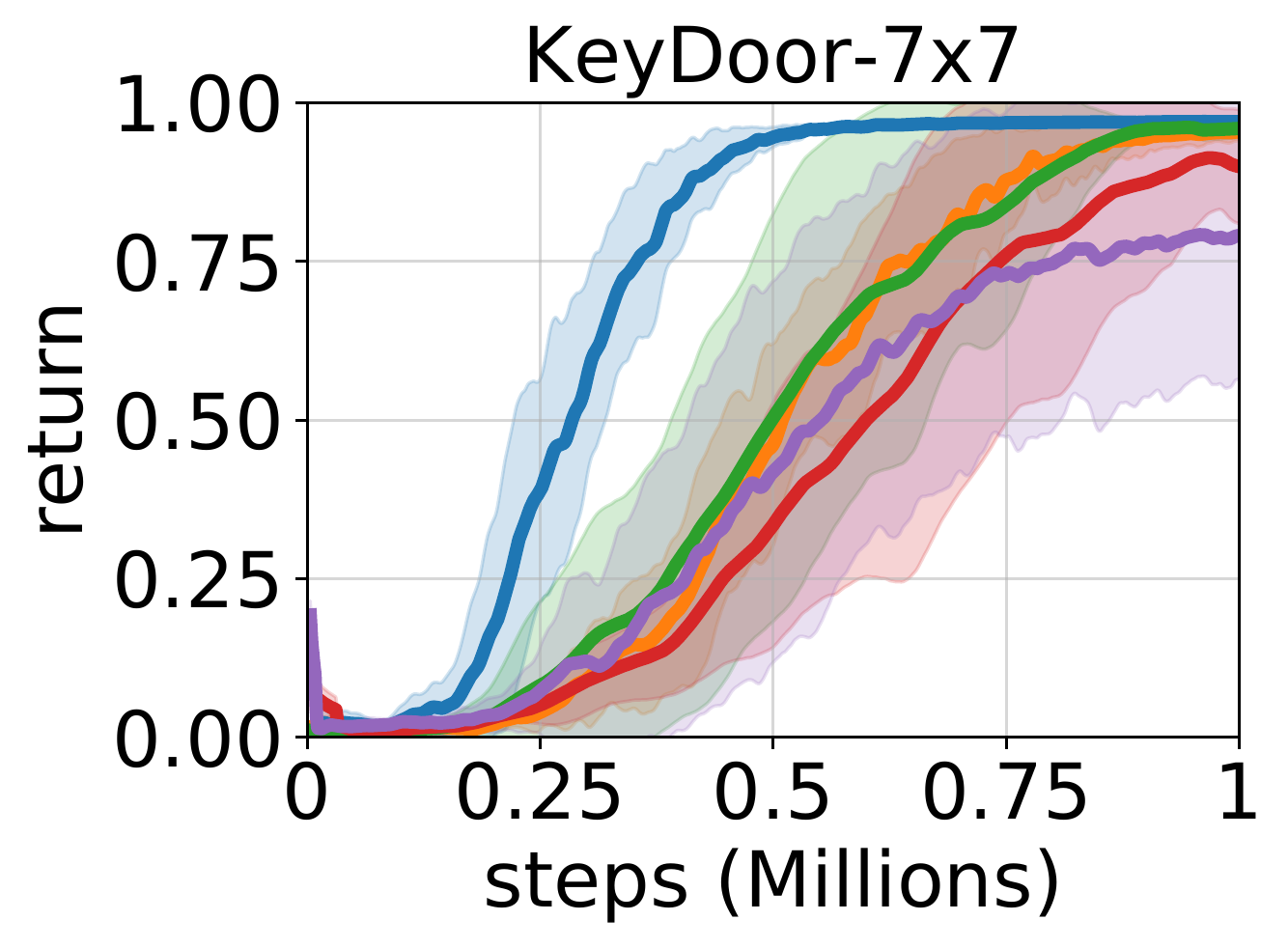}
    \hspace{-6pt}
    \includegraphics[draft=false, height=7.3\baselineskip, valign=t]{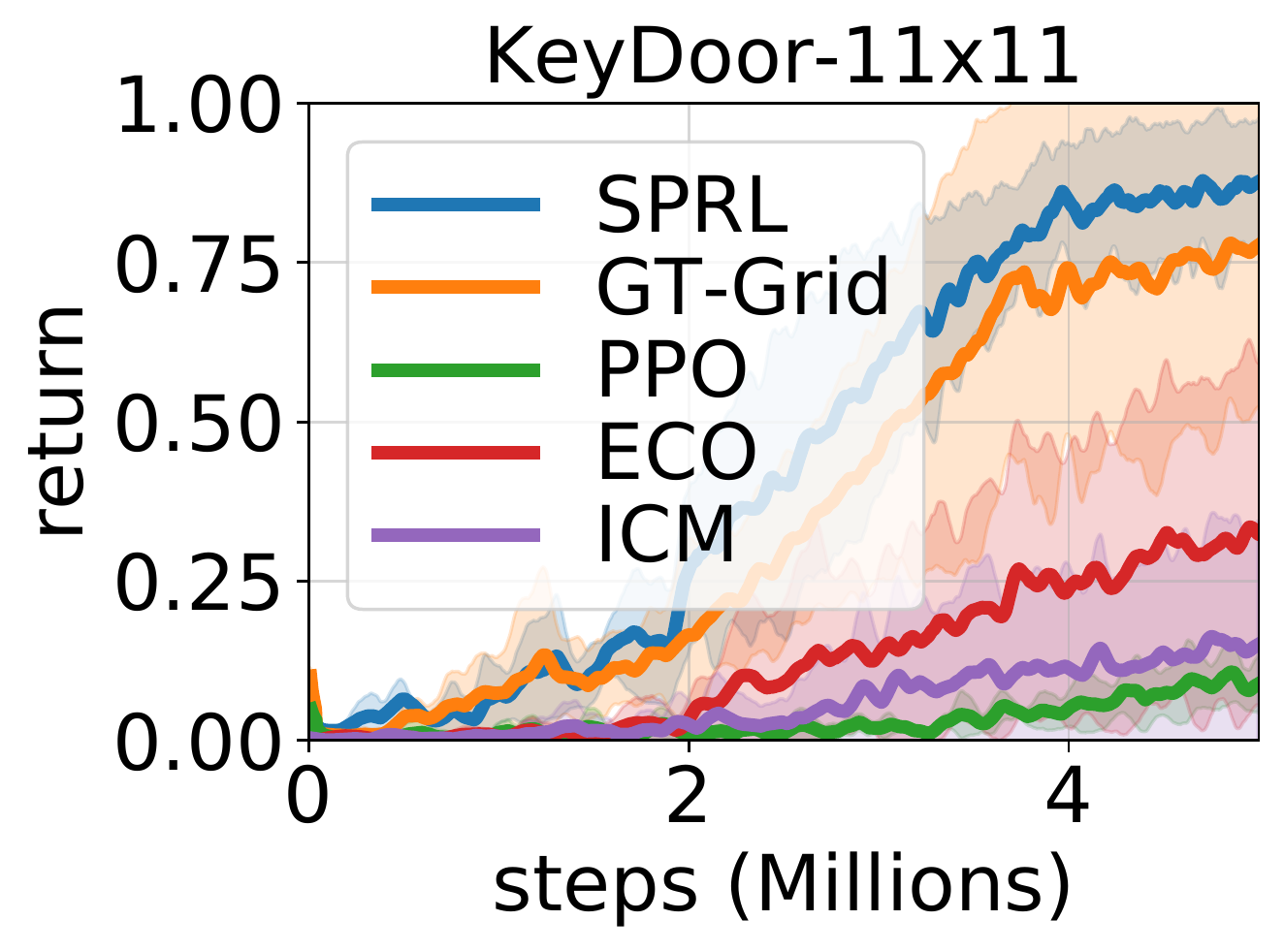}
    \hspace{-5pt}
    \vspace*{-10pt}
    \caption{
        Progress of average episode reward %
        on \grid{} tasks. We report the mean (solid curve) and standard error (shadowed area) of the performance over six random seeds. %
    }
    \label{fig:minigrid}
\vspace*{-5pt}
\end{figure*}
\cutparagraphup
\paragraph{Practical Implementation of the Cost Function.}
We implement the binary distance discriminator $\mbb{I}( D_\nr(s_{t-k}, s_{t}) < k)$ in \Cref{eq:cost} 
using the \emph{$k$-reachability network}~\citep{savinov2018episodic}. The $k$-reachability network $\text{Rnet}_k(s, s')$ is trained to output 1 if the state $s'$ is reachable from the state $s$ with less than or equal to $k$ consecutive actions, and 0 otherwise. %
Formally, we take the functional form:
$
	\text{Rnet}_{k}(s, s') \simeq \mbb{I}\left( D_\nr(s,s') < k+1 \right).
$
We then estimate the cost term $c_t$ using $(k-1)$-reachability network as follows:
\vspace*{-5pt}
\begin{align}
    &c_t= \mbb{I}\left[ D_\nr(s_{t-k},s_{t})< k\right] \mbb{I}\left[t\geq k\right]
    \prod_{j=t-k}^{t-1}{\mathbb{I} \left[
    r_j=0\right]}\\
	 &= \text{Rnet}_{k-1}(s_{t-k}, s_t) 
	  \mbb{I}\left[t\geq k\right]
	  \prod_{j=t-k}^{t-1}{\mathbb{I} \left[
    r_j=0\right]}.\label{eq:cost-2}
\vspace*{-5pt}
\end{align}
Intuitively speaking, if the agent takes a $k$-shortest path, then the distance between $s_{t-k}$ and $s_t$ is $k$, hence $c_t = 0$.
If it is not a $k$-shortest path, $c_t > 0$ since the distance between $s_{t-k}$ and $s_t$ will be less than $k$.
In practice, due to the error in the reachability network, we add a small tolerance $\Delta t\in\mbb{N}$ to ignore outliers.
It leads to an empirical version of the cost as follows:
\vspace*{-10pt}
\begin{align}
   c_t \simeq & \text{Rnet}_{k-1}(s_{t-k-\Delta t}, s_t)\cdot \prod_{j=t-k-\Delta t}^{t-1}{\mathbb{I} \left[
    r_j=0\right]} \nonumber\\
    &\cdot \mbb{I}( t \geq k+\Delta t). 
   \label{eq:cost-tol}
\vspace*{-3pt}
\end{align}
In our experiment, we found that a small tolerance $\Delta t \simeq k/5$ works well in general.
Similar to~\citet{savinov2018episodic}, we used the following contrastive loss for training the reachability network:
\setlength{\abovedisplayskip}{4pt}
\setlength{\belowdisplayskip}{4pt}
\begin{align} %
    \mathcal{L}_{\text{Rnet}} = &- \log\left(\text{Rnet}_{k-1}(s_{\text{anc}}, s_+) \right) \\
    &- \log\left(1 - \text{Rnet}_{k-1}(s_{\text{anc}}, s_-) \right),
    \label{eq:rnet-loss}
\end{align}
where $s_{\text{anc}}, s_+, s_-$ are the anchor, positive, and negative samples, respectively (See \Cref{sec:appendix-reachability-training} for the detail of training).

%% file: 5_related.tex
\cutsectionup
\section{Related Work}\label{sec:r}
\cutsectiondown
\paragraph{Shortest-path Problem and Planning.}
Many early works~\citep{bellman1958routing, ford1956network, bertsekas1991analysis,bertsekas1995neuro} have discussed (stochastic) shortest path problems in the context of MDP. They viewed the shortest path problem as a planning problem and proposed a dynamic programming-based algorithm similar to the value iteration~\citep{sutton2018reinforcement} to solve it.
Our main idea is inspired by (but not based on) this viewpoint. Specifically, our method does not directly solve the shortest path problem via planning; hence, our method does not require a forward model for planning. Our method only exploits the optimality guarantee of the shortest-path under the $\pi$-distance to prune out sub-optimal policies (\ie, non-shortest paths).
\cutparagraphdown

\cutparagraphup
\paragraph{Distance Metric in Goal-conditioned RL.}
In goal-conditioned RL, there has been a recent surge of interest in learning a distance metric in state (or goal) space to construct a high-level MDP graph and perform planning to find a shortest-path to the goal state. ~\citet{huang2019mapping, laskin2020sparse} used the universal value function (UVF)~\citep{schaul2015universal} with a constant step penalty as a distance function. ~\citet{zhang2018composable, laskin2020sparse} used the success rate of transition between nodes as distance and searched for the longest path to find the plan with the highest success rate. SPTM~\citep{savinov2018semi} defined a binary distance based on the reachability network (RNet) to connect nearby nodes in the graph.
However, the proposed distance metrics and methods can be used only for the goal-conditioned task and lack the theoretical guarantee in general MDP, while our theory and framework are applicable to general MDP (see~\Cref{sec:shortest-path-general}).
\cutparagraphdown

\cutparagraphup
\paragraph{Reachability Network.}
The reachability network (RNet) was first proposed by~\citet{savinov2018episodic} as a way to measure the novelty of a state for \emph{exploration}. Intuitively, if the current state is not reachable from previous states in episodic memory, it is considered to be novel.
SPTM~\citep{savinov2018semi} used RNet to predict the local connectivity (\ie, binary distance) between observations in memory for \emph{graph-based planning} in a navigation task.
\citet{zhang2020generating} used the k-adjacency network, which is analogous to the RNet, to improve the subgoal generation of hierarchical reinforcement learning (HRL) by constraining the goal space into adjacent states from the current state.
On the other hand, we use RNet for \emph{constraining the policy} (\ie, removing the sub-optimal policies from policy space). Thus, in ours and the other three compared works, RNet is being employed for fundamentally different purposes.
\cutparagraphup
\paragraph{Sparse Reward Problem.}
One of the most famous approaches to tackle the sparse reward problem in RL is intrinsic motivation~\citep{bellemare2016unifying, pathak2017curiosity, savinov2018episodic, burda2018large}. By adding intrinsic reward, they aim to transform the original sparse reward problem into a dense reward problem.~\citet{bellemare2016unifying} is one of the pioneer works that formulated the intrinsic reward based on the pseudo count of state visitation to measure the state novelty.~\citet{pathak2017curiosity, burda2018large} designed an intrinsic reward using prediction error.~\citet{savinov2018episodic} defined the intrinsic reward based on whether the current state is ``reachable'' from the previously visited states, where the reachability between a pair of states is predicted by a neural network that is trained via temporal contrastive learning.
~\citet{florensa2017reverse} created a curriculum based on the starting positions in training to tackle the sparse reward problem. The positions near the goal are considered easy and the positions far from the goal are considered hard. ~\citet{riedmiller2018learning} formed the sparse reward problem into multiple low-level tasks. After designing an auxiliary reward function and learning a policy for every task, they learned a high-level policy that decides the sequence of low-level policies.~\citet{ecoffet2019go} proposed to learn a policy that can go back to previously visited states, such that the agent can perform a directed exploration around the promising state. Our \sprl{} is not concerned with measuring the state novelty but aims to shrink the policy search space to improve the sample efficiency of RL algorithms. The exploration is promoted as a byproduct of the reduced policy search space.

\cutparagraphdown
\cutparagraphup
\paragraph{More Related Works.}
Please refer to \Cref{appendix:extended-related-works} for further discussions about other related works.%

%% file: 6_experiment.tex
\section{Experiments}\label{sec:e}
\begin{figure*}[!h]
    \centering
    \vspace*{0pt}
    \hspace*{-10pt}
    \includegraphics[draft=false, width=0.24\linewidth, valign=t]
    {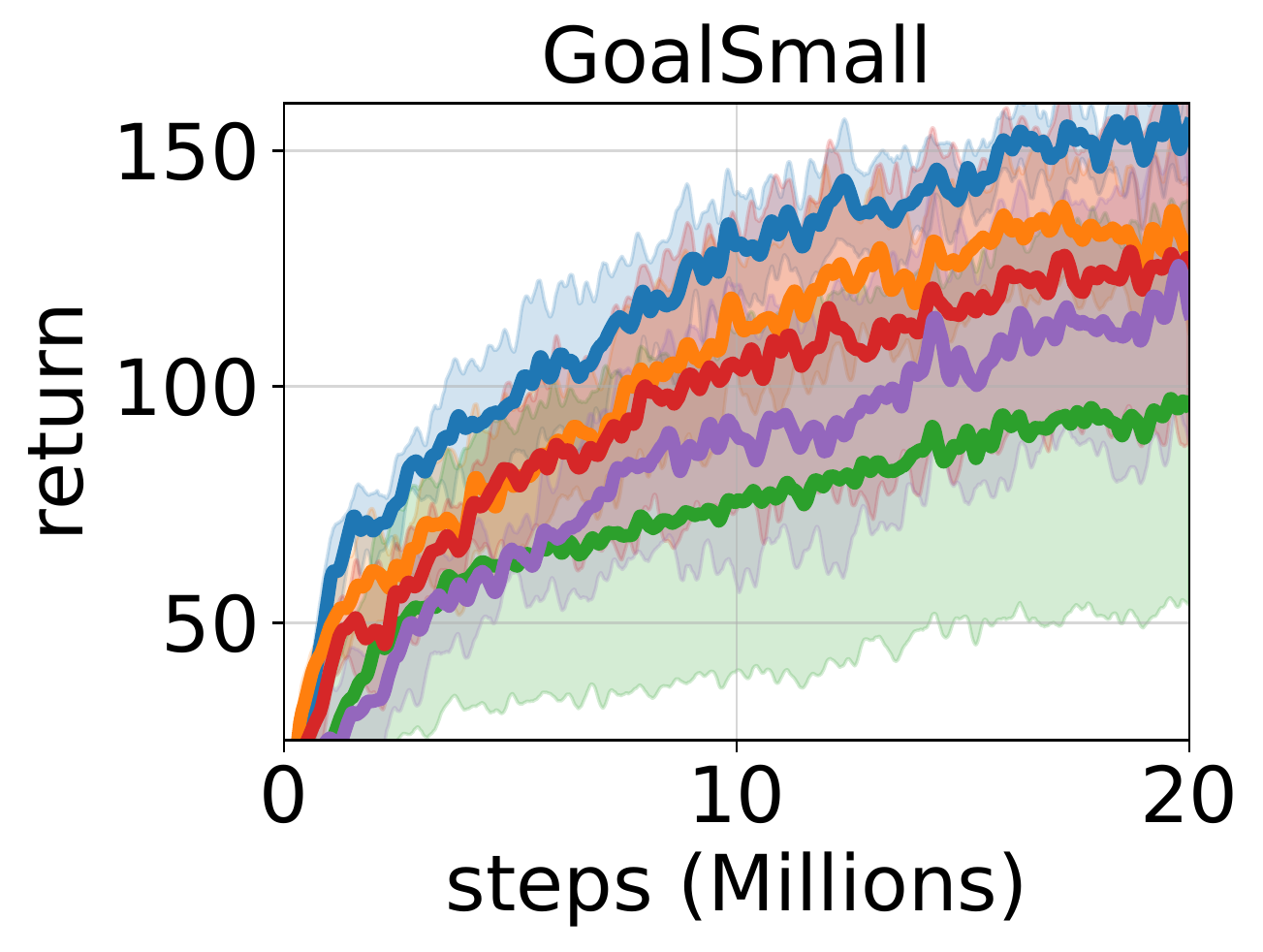}
    \hspace{-5pt}
    \includegraphics[draft=false, width=0.24\linewidth, valign=t]
    {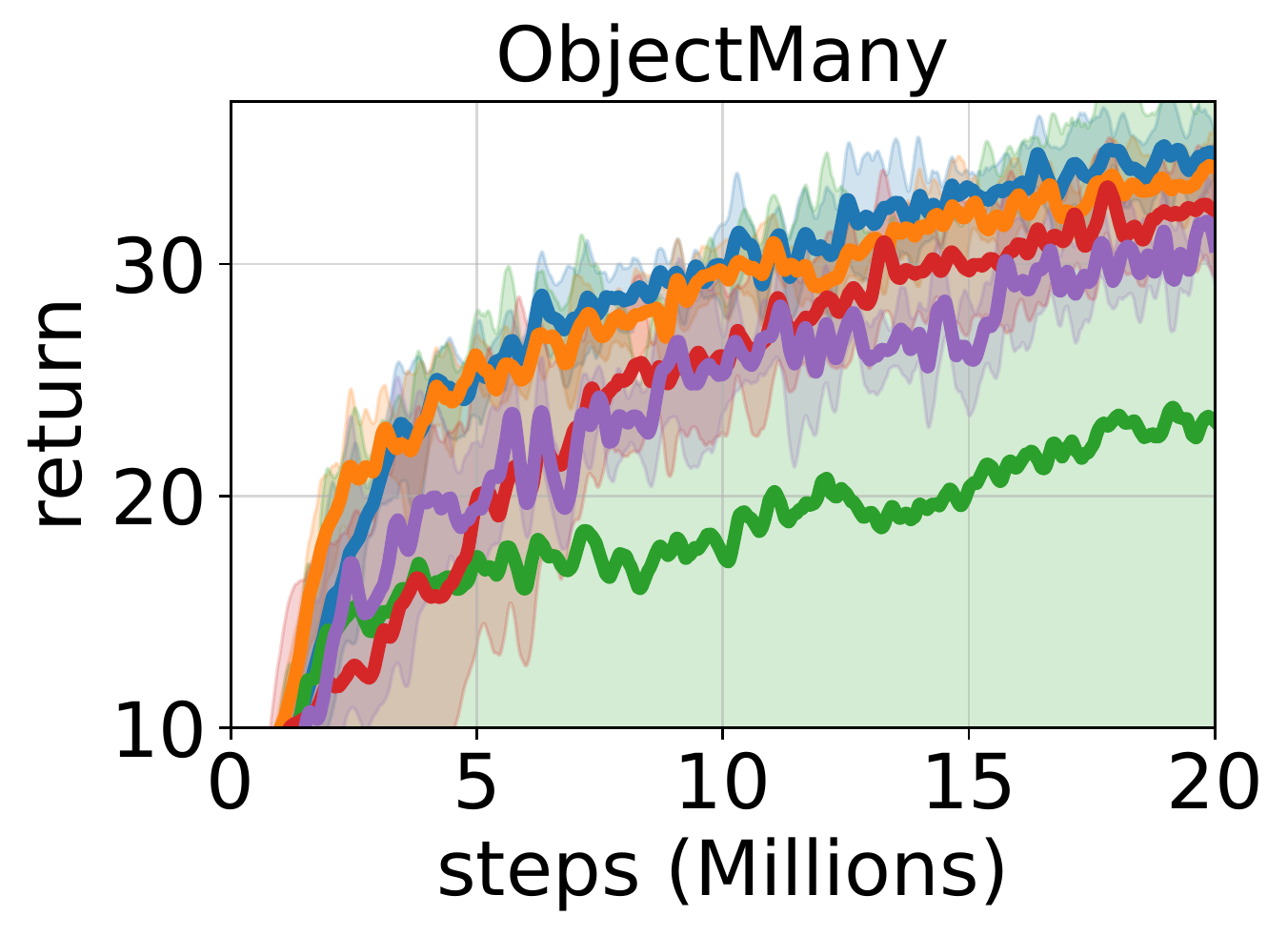}
    \hspace{-5pt}
    \includegraphics[draft=false, width=0.24\linewidth, valign=t]
    {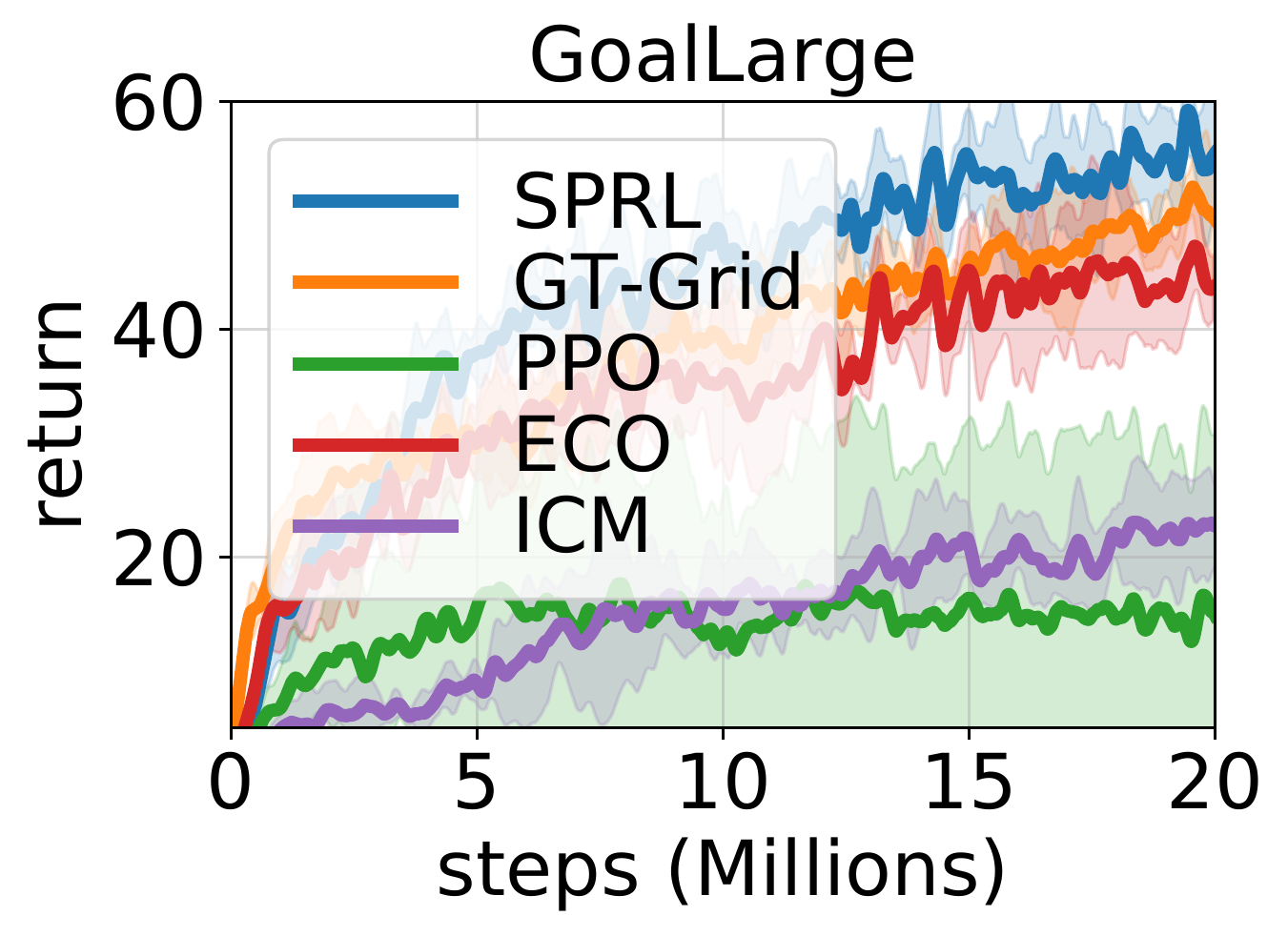}
    \hspace{-10pt}
    \vspace*{-10pt}
    \caption{
        Progress of average episode reward %
        on \dmlab{} tasks. We report the mean (solid curve) and standard error (shadowed area) of the performance over four random seeds. %
    }
    \label{fig:dmlab}
\vspace*{-5pt}
\end{figure*}
\begin{figure*}[!h]
    \centering
    \vspace*{0pt}
    \hspace*{-1pt}
    \includegraphics[draft=false, height=7.2\baselineskip, valign=t]{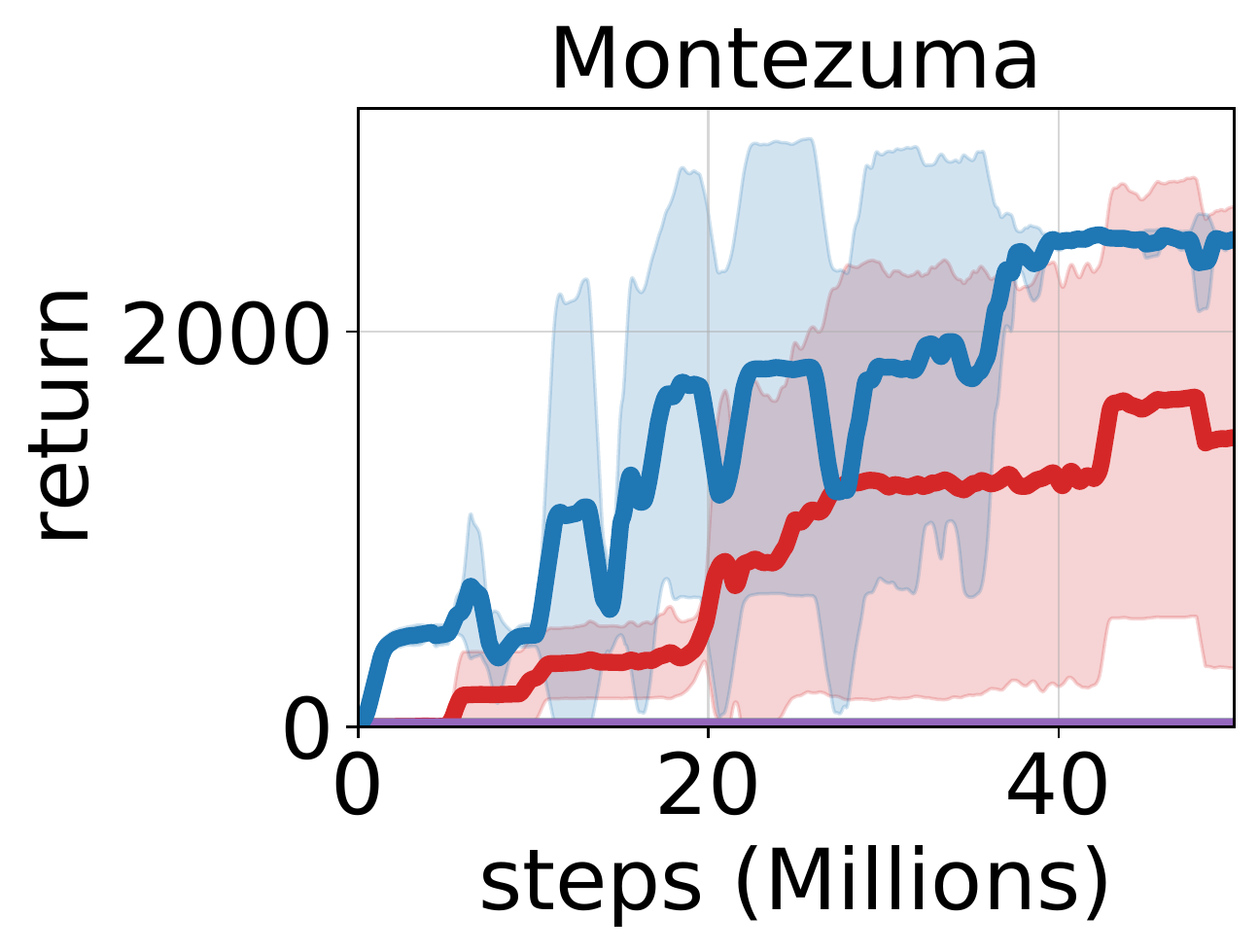}
    \hspace{-2pt}
    \includegraphics[draft=false, height=7.2\baselineskip, valign=t]{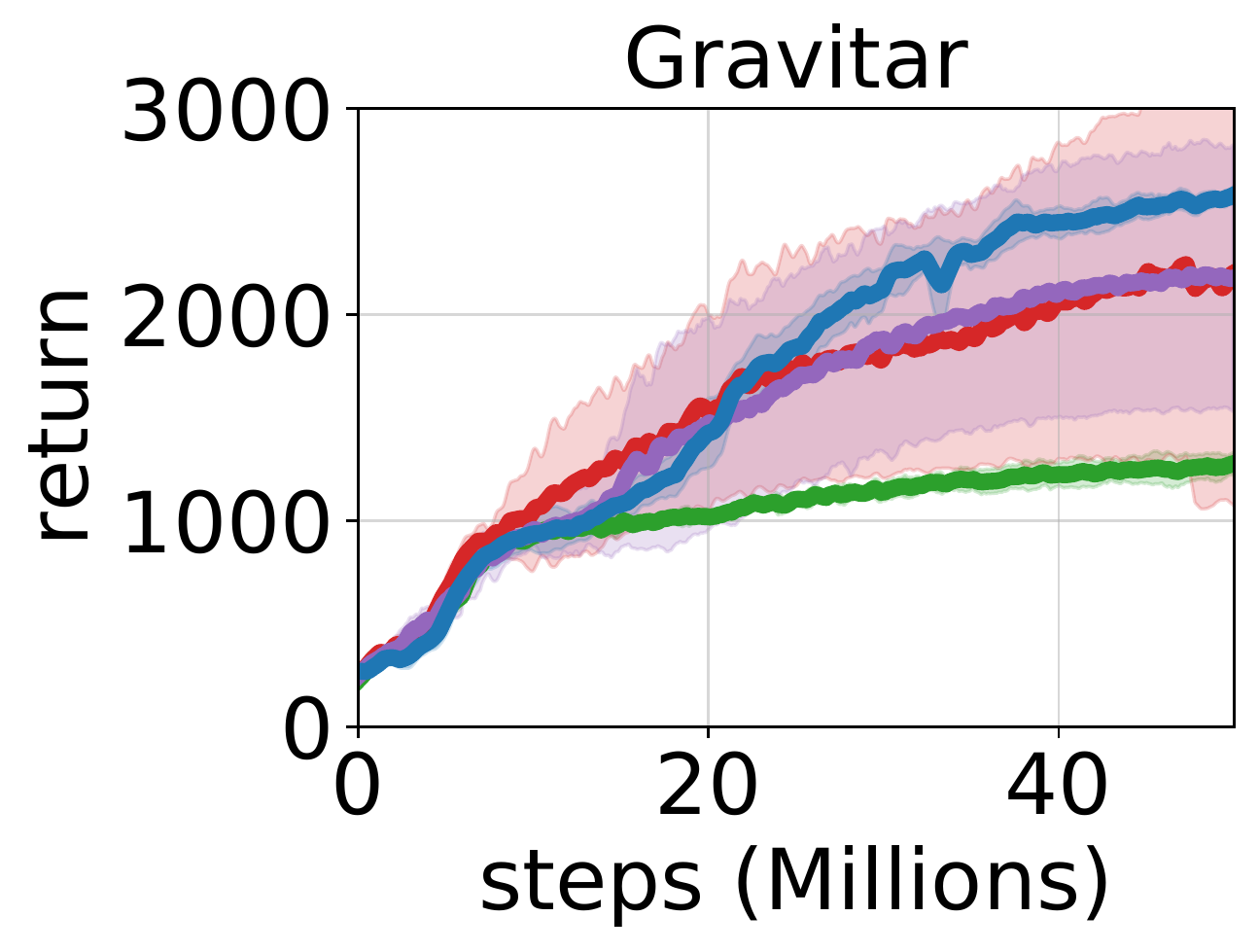}
    \hspace{-2pt}
    \includegraphics[draft=false, height=7.2\baselineskip, valign=t]{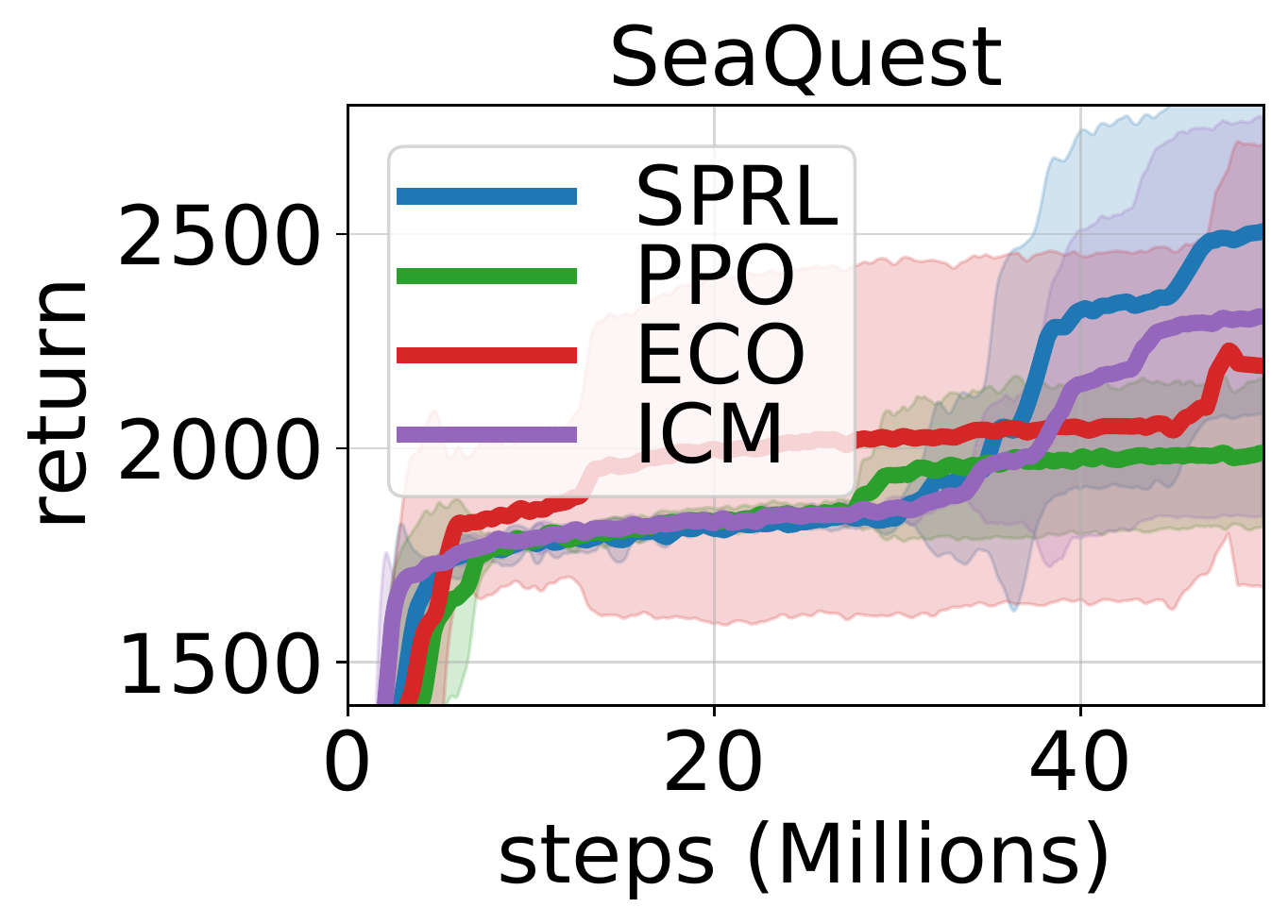}
    \\
    \hspace*{-1pt}
    \includegraphics[draft=false, height=7.2\baselineskip, valign=t]{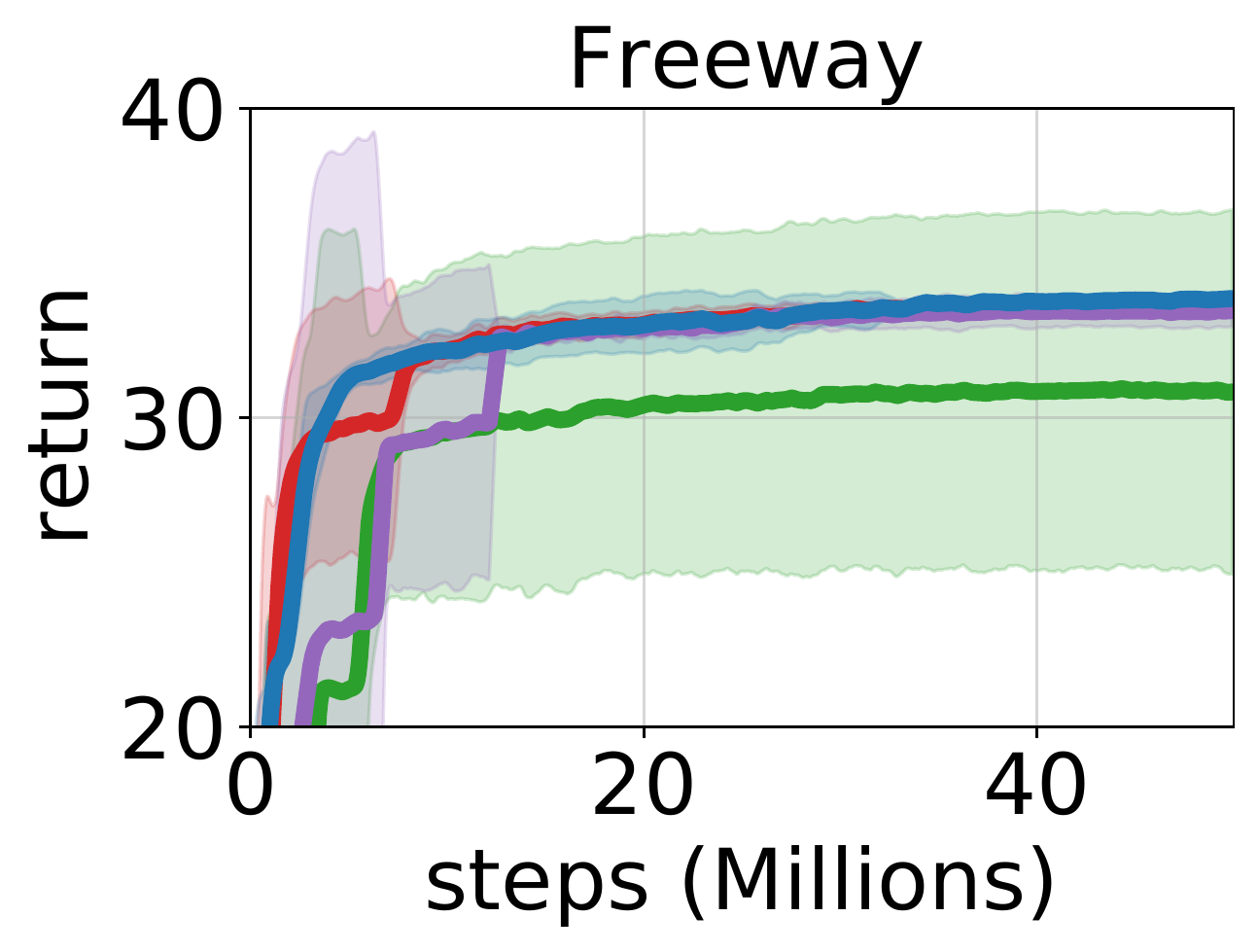}
    \hspace{-2pt}
    \includegraphics[draft=false, height=7.2\baselineskip, valign=t]{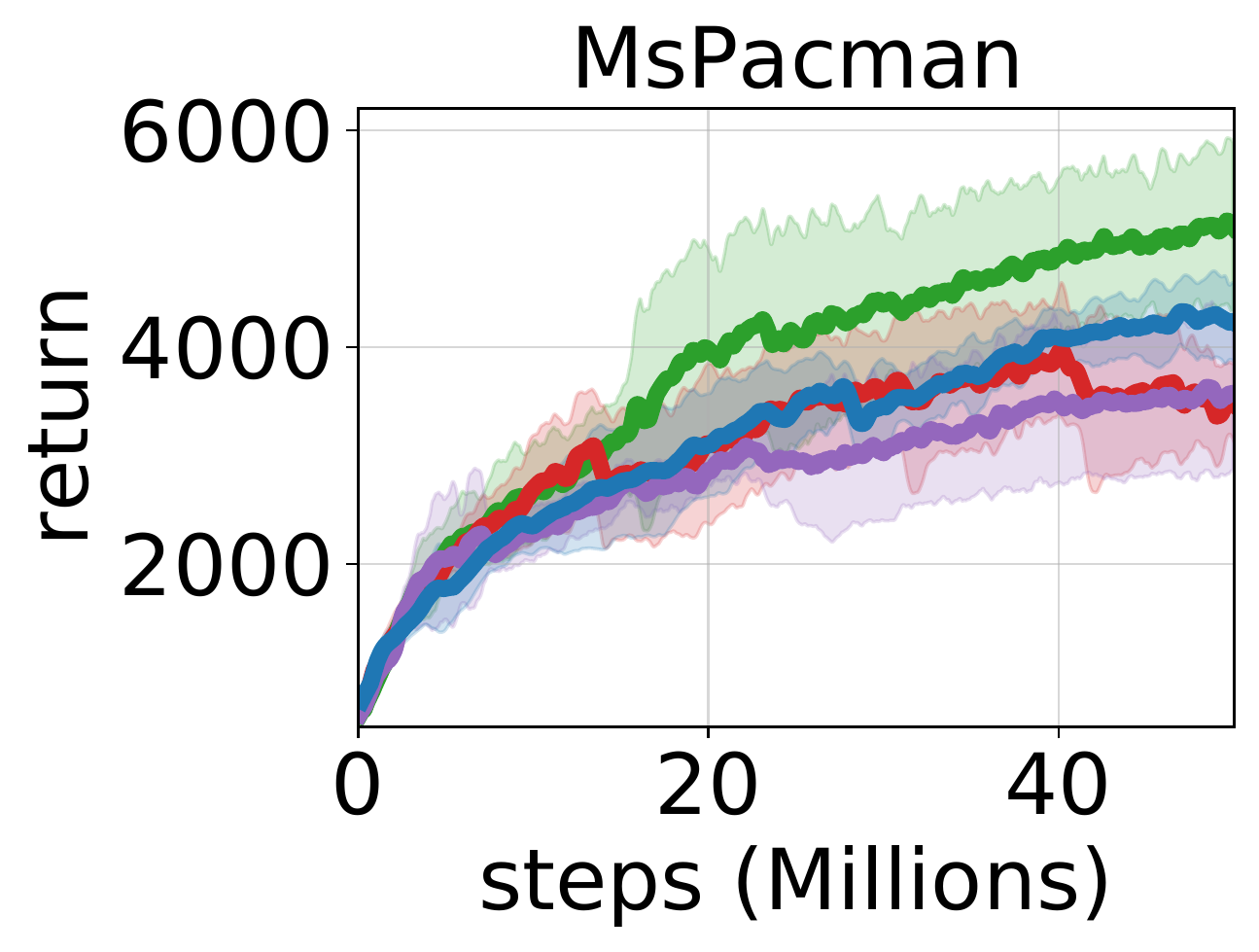}
    \hspace{-2pt}
    \includegraphics[draft=false, height=7.2\baselineskip, valign=t]{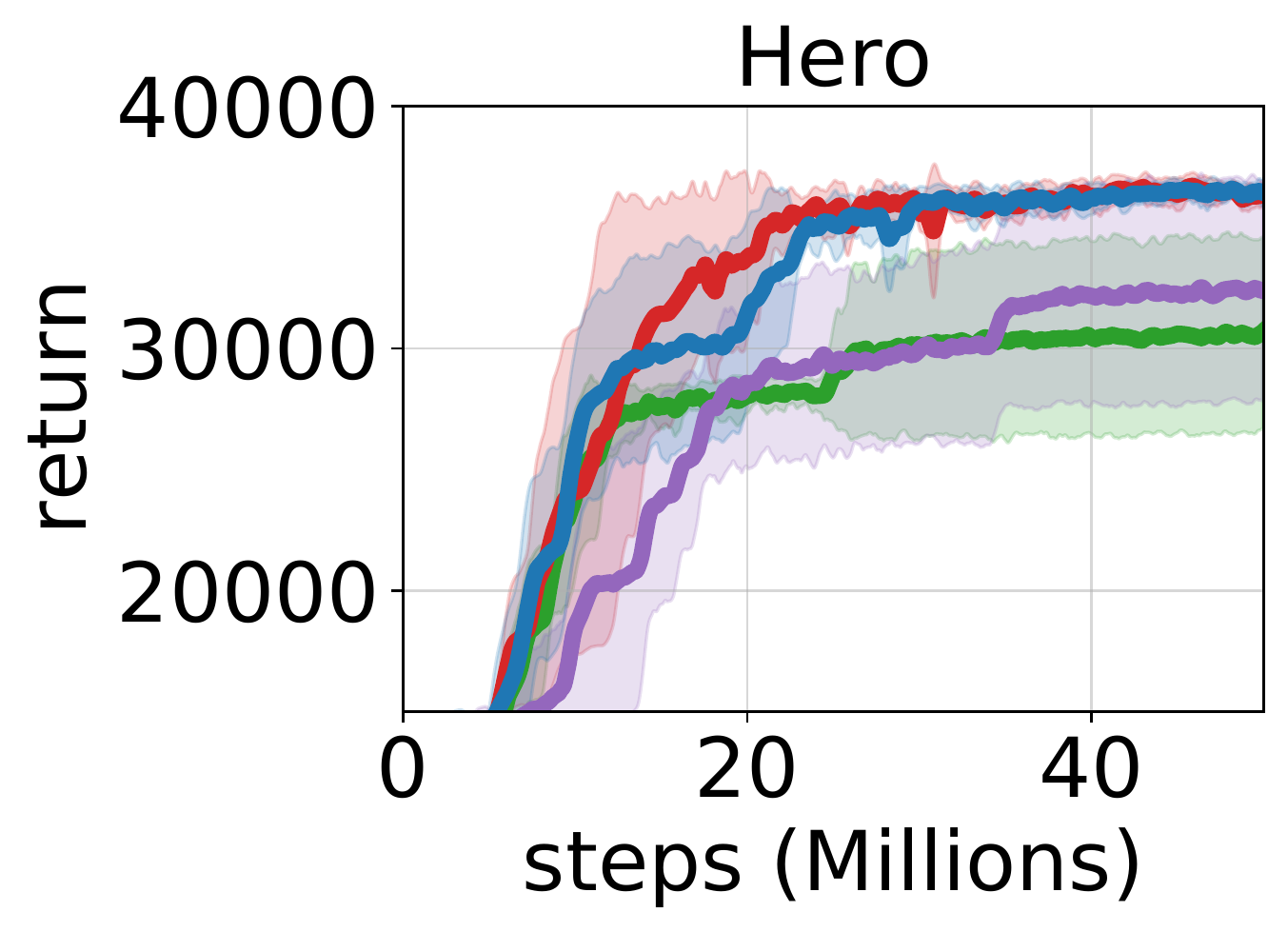}
    \vspace*{-10pt}
    \caption{
        Progress of average episode reward %
        on \atari tasks. We report the mean (solid curve) and standard error (shadowed area) of the performance over four random seeds. The performances of PPO and ICM in \montezuma are both zero, hence invisible in the figure.
    }
    \label{fig:atari}
\vspace{-10pt}
\end{figure*}
\begin{figure*}[!h]
    \centering
    \vspace*{0pt}
    \hspace*{-1pt}
    \includegraphics[draft=false, height=7.2\baselineskip, valign=t]{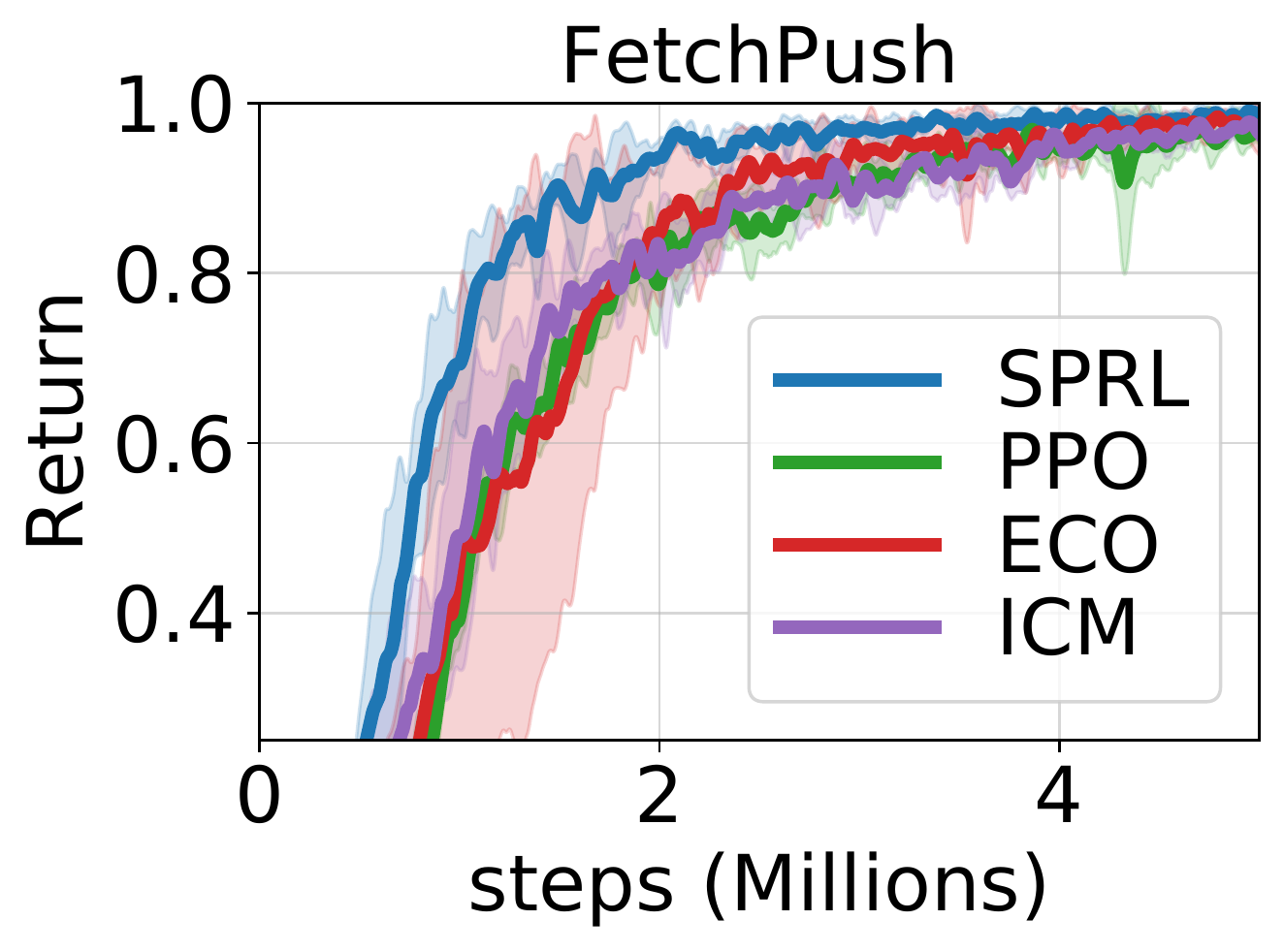}
    \hspace{-2pt}
    \includegraphics[draft=false, height=7.2\baselineskip, valign=t]{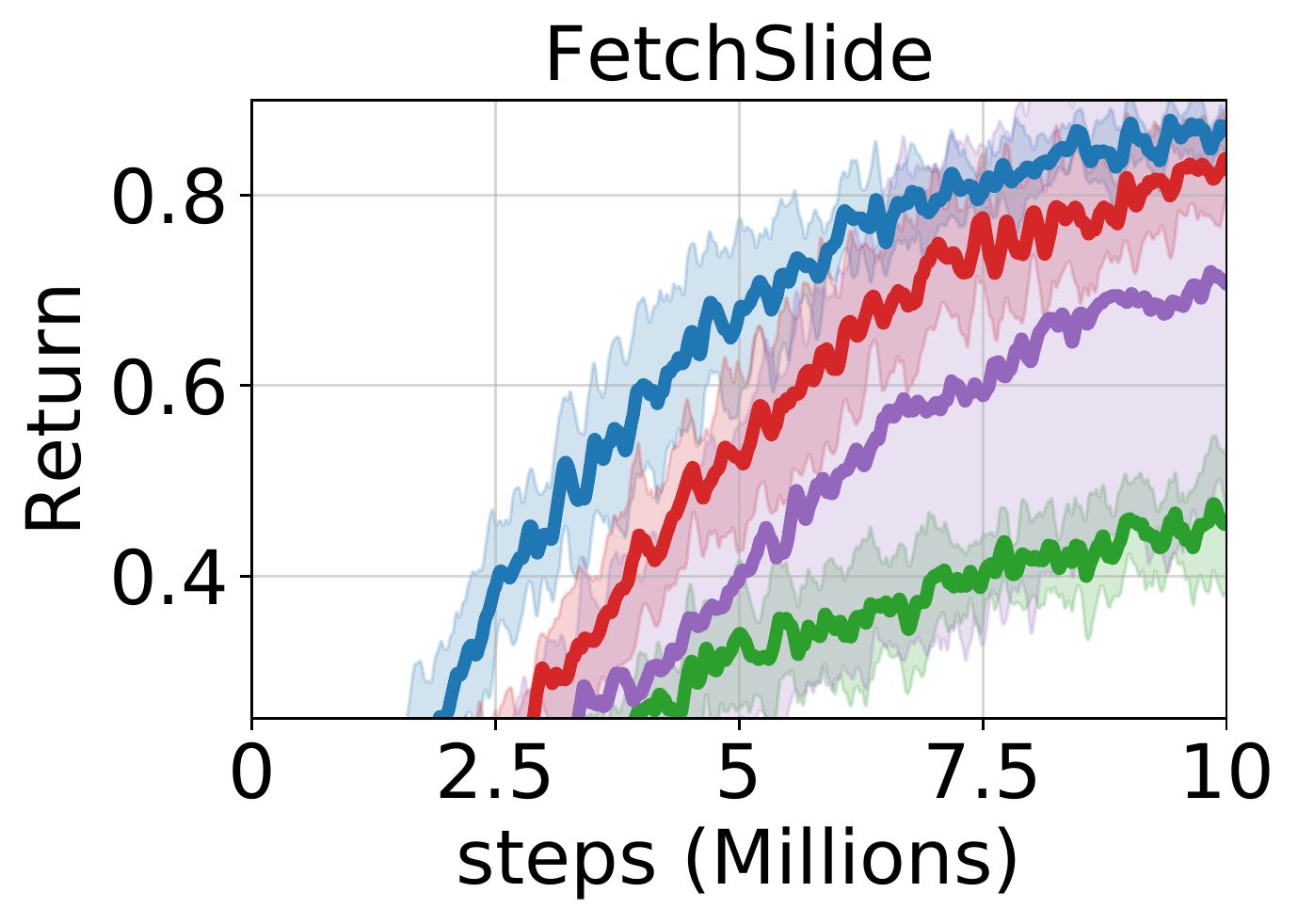}
    \hspace{-2pt}
    \includegraphics[draft=false, height=7.2\baselineskip, valign=t]{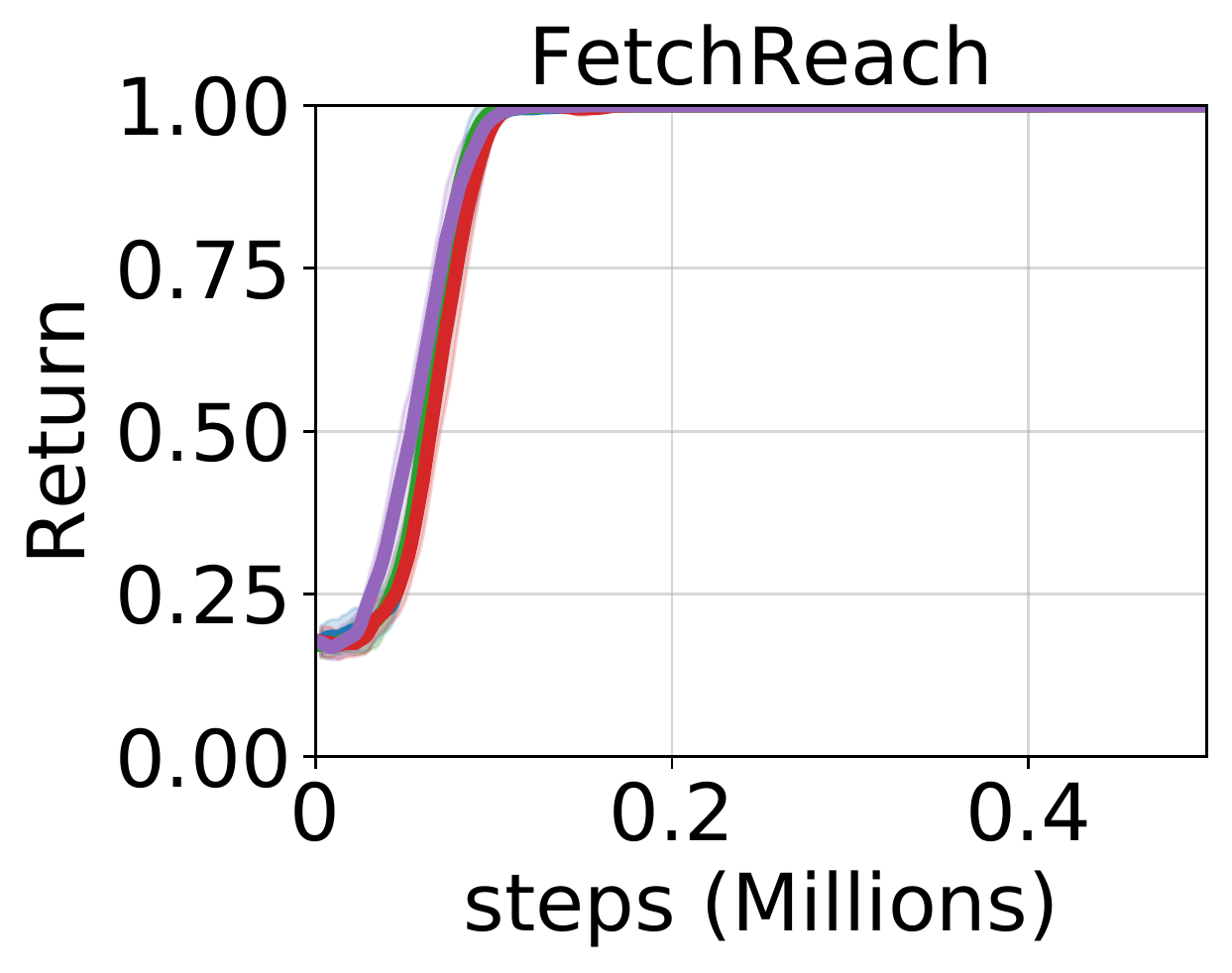}
    \hspace*{-2pt}
    \includegraphics[draft=false, height=7.2\baselineskip, valign=t]{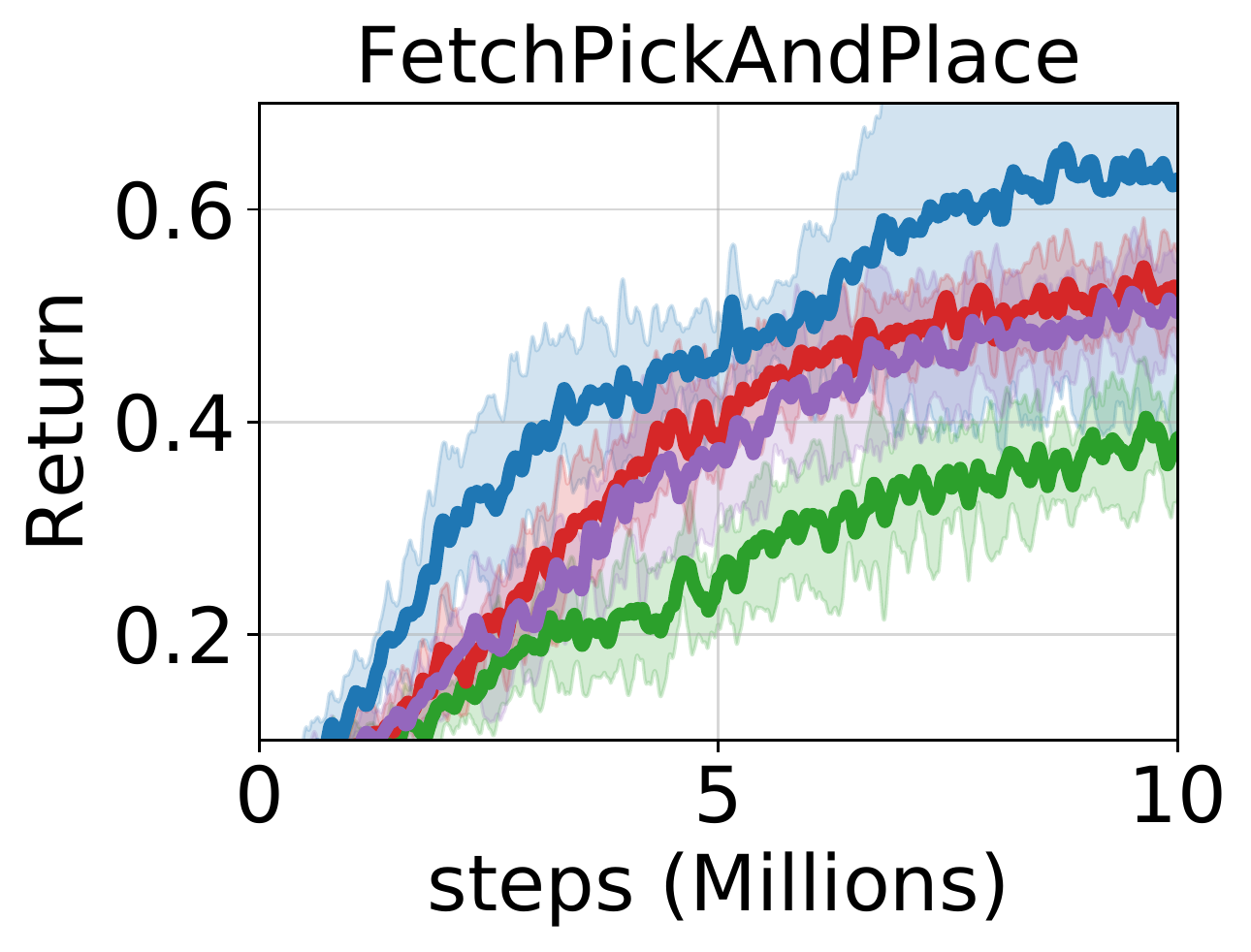}
    \caption{
        Progress of average episode reward on \fetch{} tasks. We report the mean (solid curve) and standard error (shadowed area) of the performance over four random seeds. 
    }
    \label{fig:fetch}
\vspace{-10pt}
\end{figure*}

\subsection{Settings}\label{sec:es}
\paragraph{Environments.}
We evaluate our \sprl{} on four challenging domains: \grid{}~\citep{gym_minigrid}, \dmlab{}~\citep{beattie2016deepmind}, \atari~\citep{bellemare2013arcade}, and \fetch~\citep{plappert2018multi}. 
\grid{} is a 2D grid world environment with challenging features such as pictorial observation,
random initialization of the agent and the goal, complex state and action space where coordinates, directions, and other object statuses (\eg, key-door) are considered.
We conducted experiments on four standard tasks: \fours{}, \fourl{}, \keys{}, and \keyl{}.
\dmlab{} is a 3D environment with a first-person view. Along with the nature of partially-observed MDP, at each episode,
the agent's initial and the goal location are reset randomly with a change of texture, maze structure, and colors.
We conducted experiments on three standard tasks: \dmgs{}, \dmgl{}\footnote{\dmgl{} task corresponds to the \textit{Sparse} task in~\citet{savinov2018episodic}, and our~\Cref{fig:dmlab} reproduces the result reported.} %
, and \dmom{}.
For \atari, among 52 games we chose two sparse-reward tasks (\montezuma, \freeway), one dense-reward task (\mspacman), and three non-navigational tasks (\gravitar, \seaquest, \hero) where the agent receives reward by hitting the enemy by firing a bullet or removing the obstacle by installing a bomb.
\fetch{} is a continuous control environment with a two-fingered gripper. Initial locations of the agent and the goal change every episode. We made two changes in the environment to make it a sparse-reward task.
The agent receives +1 reward if the agent reaches the goal and 0 rewards otherwise. Also, the episode terminates when the agent reaches the goal such that the agent can receive a non-zero reward at most once in an episode.
We conducted experiments on all four tasks: \fpush{}, \fslide{}, \freach{}, \fpnp{}. We refer the readers to \Cref{fig:dmlab_maze_layout} for examples of observations. See \Cref{appendix:minigrid-detail}, \Cref{appendix:dmlab-detail}, \Cref{appendix:atari-detail}, and  \Cref{appendix:fetch-detail} 
for more details of \grid{}, \dmlab{}, \atari, and \fetch{} respectively.

\cutparagraphdown

\cutparagraphup
\paragraph{Baselines.}
We compared our methods with four baselines: \ppo~\citep{schulman2017proximal}, \textit{episodic curiosity} (\eco)~\citep{savinov2018episodic}, \textit{intrinsic curiosity module} (\icm)~\citep{pathak2017curiosity}, and \oracle~\citep{savinov2018episodic}. 
The \ppo{} is used as a baseline RL algorithm for all other agents.
The \eco{} agent is rewarded when it visits a state that is not reachable from the states in episodic memory
within a certain number of actions; thus the novelty is only measured within an episode. Following~\citet{savinov2018episodic}, we trained RNet in an off-policy manner from the agent's experience and used it for our \sprl{} and \eco{} on  \grid{} ~(\Cref{sec:exp-grid}), \dmlab{}~(\Cref{sec:exp-dmlab}), \atari~(\Cref{sec:exp-atari}) and \fetch~(\Cref{sec:exp-fetch}). For the accuracy of the learned RNet on each task, please refer to \Cref{appendix:rnet}.
The \oracle{} agent has access to the agent's ($x$, $y$) coordinates. It uniformly divides the world into 2D grid cells, and the agent is rewarded for visiting a novel grid cell. 
The \icm{} agent learns a forward and inverse dynamics model and uses the prediction error of the forward model to measure the novelty.
We used the publicly available codebase~\citep{savinov2018episodic} %
to obtain the baseline results. 
We used the same hyperparameter for all the tasks for a given domain --- the details are described in the Appendix.
We used the standard domain and tasks for reproducibility.  
\subsection{Results on \grid{}}\label{sec:exp-grid}
\Cref{fig:minigrid} shows the performance of all the methods on the \grid{} domain.
\sprl{} consistently outperforms all the baseline methods over all tasks.
We observe that exploration-based methods (\ie, \eco{}, \icm{}, and \oracle{}) perform similarly to the \ppo{} in the tasks with small state space (\eg, \fours{} and \keys{}).
However, \sprl{} demonstrates a significant performance gain since it improves the exploitation by avoiding sub-optimality caused by taking a non-shortest-path.
\subsection{Results on \dmlab{}}\label{sec:exp-dmlab}
\Cref{fig:dmlab} shows the performance of all the methods on \dmlab{} tasks.
Overall, \sprl{} achieves superior results compared to other methods. 
By the task design, the difficulty of exploration increases in the order of \dmgs{}, \dmom{}, and \dmgl{} tasks, and we observe a coherent trend in the result. 
For harder exploration tasks, the exploration-based methods (\oracle{}, \icm{} and \eco{}) achieve a larger improvement over \ppo{}: \eg, 20\%, 50\%, and 100\% improvement in \dmgs{}, \dmom{}, and \dmgl{}, respectively. 
As shown in~\Cref{lem:ksp-reduce}, \sprl{} is expected to have larger improvement for larger trajectory space and sparser reward settings. %
We can verify this from the result: \sprl{} has the largest improvement in \dmgl{} task, where both the map is largest and the reward is most sparse. Interestingly, \sprl{} even outperforms \oracle{} which simulates the upper-bound performance of novelty-seeking exploration method. This is possible since \sprl{} improves the exploration by suppressing unnecessary explorations, which is different from novelty-seeking methods and also improves the exploitation by reducing the policy search space.
\begin{figure}{r}%
\vspace{0pt}
\centering
  \hspace{-15pt}
  \includegraphics[draft=false, height=0.28\linewidth, valign=c]{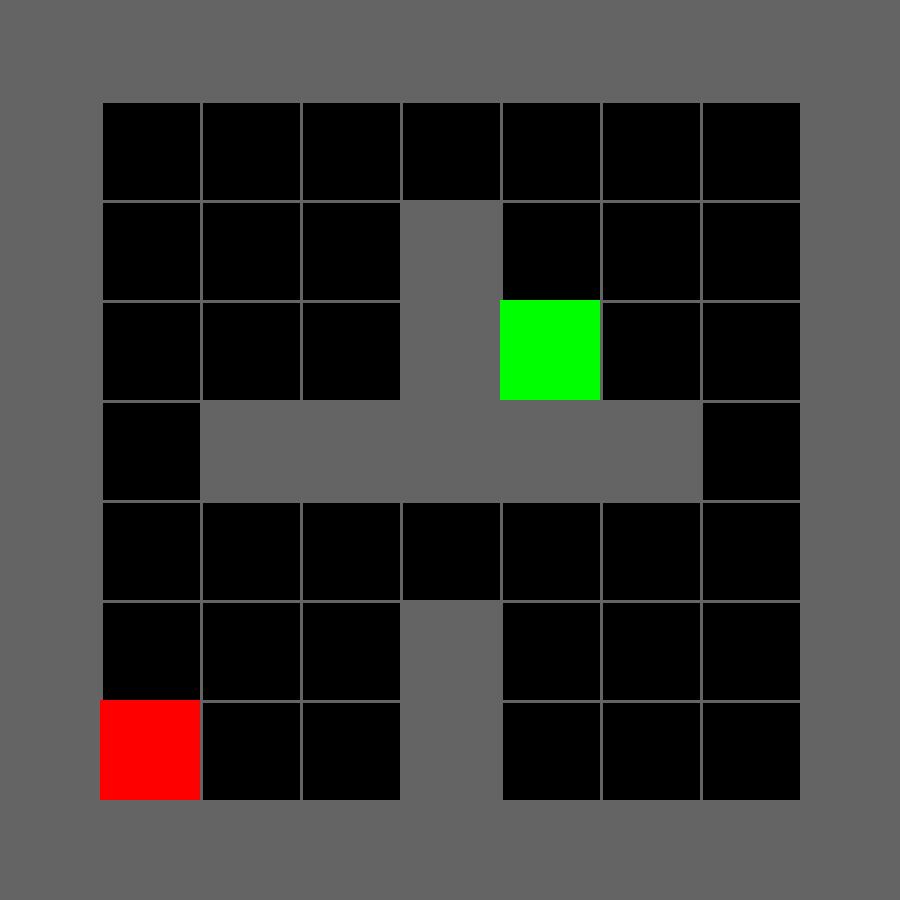}
  \hspace{-2pt}
  \includegraphics[draft=false, height=0.32\linewidth, valign=c]{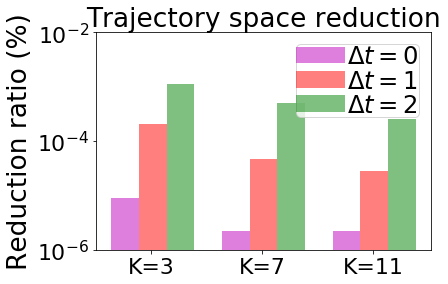}
  \hspace{-18pt}
  \vspace*{-8pt}
  \caption{
        (Left) 7$\times$7 Tabular four-rooms domain with initial agent location ({\color{red}red}) and the goal location ({\color{green}green}). (Right) The trajectory space reduction ratio (\%) before and after constraining the trajectory space for various $k$ and $\Delta t$ with \ksp{} constraint. Even a small $k$ can greatly reduce the trajectory space with a reasonable tolerance $\Delta t$.}
    \label{fig:policy-reduction}
    \vspace*{-17pt}
\end{figure}
\begin{figure*}[!htp]
    \centering
    \vspace*{0pt}
    \includegraphics[draft=false, width=0.16\linewidth, valign=t]
    {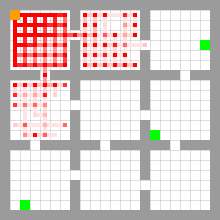}
    \hspace{2pt}
    \includegraphics[draft=false, width=0.16\linewidth, valign=t]
    {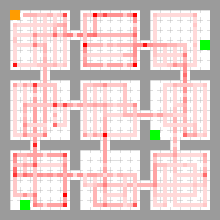}
    \hspace{2pt}
    \includegraphics[draft=false, width=0.16\linewidth, valign=t]
    {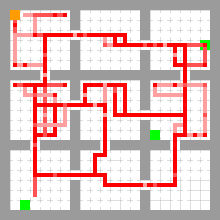}
    \hspace{2pt}
    \includegraphics[draft=false, width=0.16\linewidth, valign=t]
    {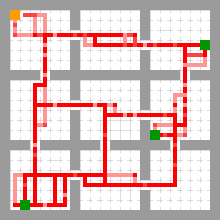}\\[2pt]
    {
    \parbox{0.16\linewidth}{\small\centering (a) \rand{} } \hspace{2pt}
    \parbox{0.16\linewidth}{\small\centering (b) \ucb{} } \hspace{2pt}
    \parbox{0.16\linewidth}{\small\centering (c) \sprl{} } \hspace{2pt}
    \parbox{0.16\linewidth}{\small\centering (d) \sprl{}+Reward }
    }
    \vspace*{-7pt}
    \caption{
        Transition count maps for baselines and \sprl{}:
        (a), (b), and (c) are in a \emph{reward-free} 
        while (d) is in a \emph{reward-aware} setting.
        In reward-free settings (a-c), we show rewarding states in {\color{green}{light green}} only for the visualization, but the agent does not receive rewards from the environment.
        The location of the agent's initial state ({\color{orange}{orange}}) and rewarding states ({\color
        {green!50!black}{dark green}}) are fixed. 
        The episode length is limited to 500 steps.
    }
    \label{fig:qualitative-result}
\vspace*{-10pt}
\end{figure*}

\subsection{Results on \atari}\label{sec:exp-atari}

\label{sec:challenges-atari}
One of the main challenges in \atari is the distribution shift in the state space within a task.
Unlike \grid{} and \dmlab{}, many \atari tasks involve the transition between different rooms in each of which the agent observes a significantly different set of states. 
This induces instability in the RNet training; RNet often overfits to the initial room and performs poorly when the agent navigates to the different rooms. 
To mitigate this problem, we added the weight decay for the RNet training. 
For other technical details, please refer to~\Cref{appendix:atari-rnet-detail}.

\Cref{fig:atari} summarizes the performance of all the methods on \atari tasks. 
\sprl{} outperforms all the baseline methods on five out of six tasks except for \mspacman, which is a dense reward task. We note that other exploration methods, \icm \, and \eco, also perform poorly on this task.
For the sparse reward task, especially in \montezuma, \sprl{} achieves the performance comparable to the SOTA exploration methods such as RND~\citep{burda2018exploration} (1000 score at 50M steps with 32 parallel environments. \sprl \ used 12 parallel environments.) and SOTA exploitation methods such as  SIL~\citep{oh2018self} (2500 score at 50M steps). Lastly, \sprl{} achieves the largest improvement to the \ppo{} in non-navigational tasks (\gravitar, \seaquest, \hero). This verifies that our \ksp{} constraint is not limited to just the geometric path but can be applied to any general trajectory, or a sequence of state transitions, in MDP.

\subsection{Results on \fetch{}}\label{sec:exp-fetch}

\label{sec:modification-fetch}

\Cref{fig:fetch} summarizes the performance of all the methods on \fetch tasks. 
\sprl{} outperforms all the baseline methods on every task except for \freach{}, in which all the compared methods perform similarly well. 
We also found that the performance improvement of the exploration methods such as \sprl{}, \eco{}, and \icm{} over \ppo{} is larger for the tasks with sparser reward; the reward is the densest in \freach{} while the most sparse in \fslide{}.
As suggested by our theory, this result shows that \sprl{} is not restricted to the domains with discrete action space but performs well on the continuous control domain. 
We present the additional results in~\Cref{appendix:rnet} showing that the reachability network, which is the key component of \sprl{}, can be efficiently trained and accurately compute the state proximity in the continuous control domain.

\subsection{Quantitative Analysis on $k$-SP Constraint}
\label{sec:analysis}
\cutsectiondown
In this section, we numerically evaluate the effect of our $k$-shortest path constraint in a tabular-RL setting. Specifically, we study the following questions:
(1) Does the \ksp{} constraint with larger $k$ results in more reduction in trajectory space? (\ie, validation of \Cref{lem:ksp-reduce})
(2) How much reduction in trajectory space does \ksp{} constraint provide with different $k$ and tolerance $\Delta t$?

We implemented a simple tabular 7x7 four-rooms domain where each state maps to a unique $(x, y)$ location of the agent. The agent can take \textit{up, down, left, right} primitive actions to move to the neighboring state, and the episode horizon is set to 14 steps. The goal of the agent is reaching the goal state, which gives +1 reward and terminates the episode. We computed the ground-truth distance between a pair of states to implement the $k$-shortest path constraint. We used the ground-truth distance function instead of the learned RNet to implement the exact SPRL agent.

\Cref{fig:policy-reduction} summarizes the reduction in the trajectory space size. We searched over all possible trajectories of length 14 using breadth-first-search (BFS). Then we counted the number of trajectories satisfying our \ksp{} constraint with varying parameters $k$ and tolerance $\Delta t$ and divided by the total number of trajectories (\ie, $4^{14}=268$M).
The result shows that our \ksp{} constraint drastically reduces the trajectory space even in a simple 2D grid domain; with a very small $k=3$ and no tolerance $\Delta t=2$, we get only $24/268$M size of the original search space. 
As we increase $k$, we can see more reduction in the trajectory space, which is consistent with~\Cref{lem:ksp-reduce}.
Also, increasing the tolerance $\Delta t$ slightly hurts the performance, but still achieves a large reduction (See~\Cref{appendix:ablation} for more analysis on the effect of $k$ and tolerance).

\vspace*{2pt}
\subsection{Qualitative Analysis on \grid{}} %
We qualitatively studied what type of policy is learned with the \ksp{} constraint with the ground-truth RNet in \textit{NineRooms} domain of \grid{}.
\Cref{fig:qualitative-result} (a-c) shows the converged behavior of \sprl{} ($k=15$), the ground-truth count-based exploration~\citep{lai1985asymptotically} agent (\ucb{}), and uniformly random policy (\rand{}) in a reward-free setting.
We counted all the state transitions ($s_t\rightarrow s_{t+1}$) of each agent's roll-out and averaged over 4 random seeds.
\rand{} cannot explore further than the initial few rooms. \ucb{} seeks novel states and visits all the states uniformly. \sprl{} learns to take the longest possible shortest path, which results in a ``straight'' path across the rooms.
Note that this only represents a partial behavior of \sprl{}, since our cost also considers the \emph{existence of non-zero reward} (see~\Cref{eq:cost}). Thus, in (d), we tested \sprl{} while providing only the \emph{existence} of non-zero reward (but not the reward magnitude). \sprl{} learns to take the shortest path between rewarding and initial states that is consistent with the shortest-path definition in~\Cref{def:k-shortest-path-constraint}.

%% file: 7_conclusion.tex
\section{Conclusion}\label{sec:c}
\cutsectiondown

We presented the $k$-shortest-path constraint, which can improve the sample efficiency of any model-free RL method by preventing the agent from taking sub-optimal transitions. We empirically showed that \sprl{} outperforms vanilla RL and strong novelty-seeking exploration baselines on four challenging domains. 
We believe that our framework develops a unique direction for improving the sample efficiency in reinforcement learning; hence, combining our work with other techniques for better sample efficiency will be an  interesting future work that could benefit many practical tasks.

\paragraph{Acknowledgements}
This research is supported in part by NSF IIS \#1453651 and Korea Foundation for Advanced Studies.

%% file: 8_appendix_ablation.tex
\cutsectionup
\section{More ablation study}\label{appendix:ablation}
\cutsectiondown
\begin{figure}[!h]
    \centering
    \hspace{-10pt}
    \includegraphics[draft=false, width=0.33\linewidth, valign=t]
    {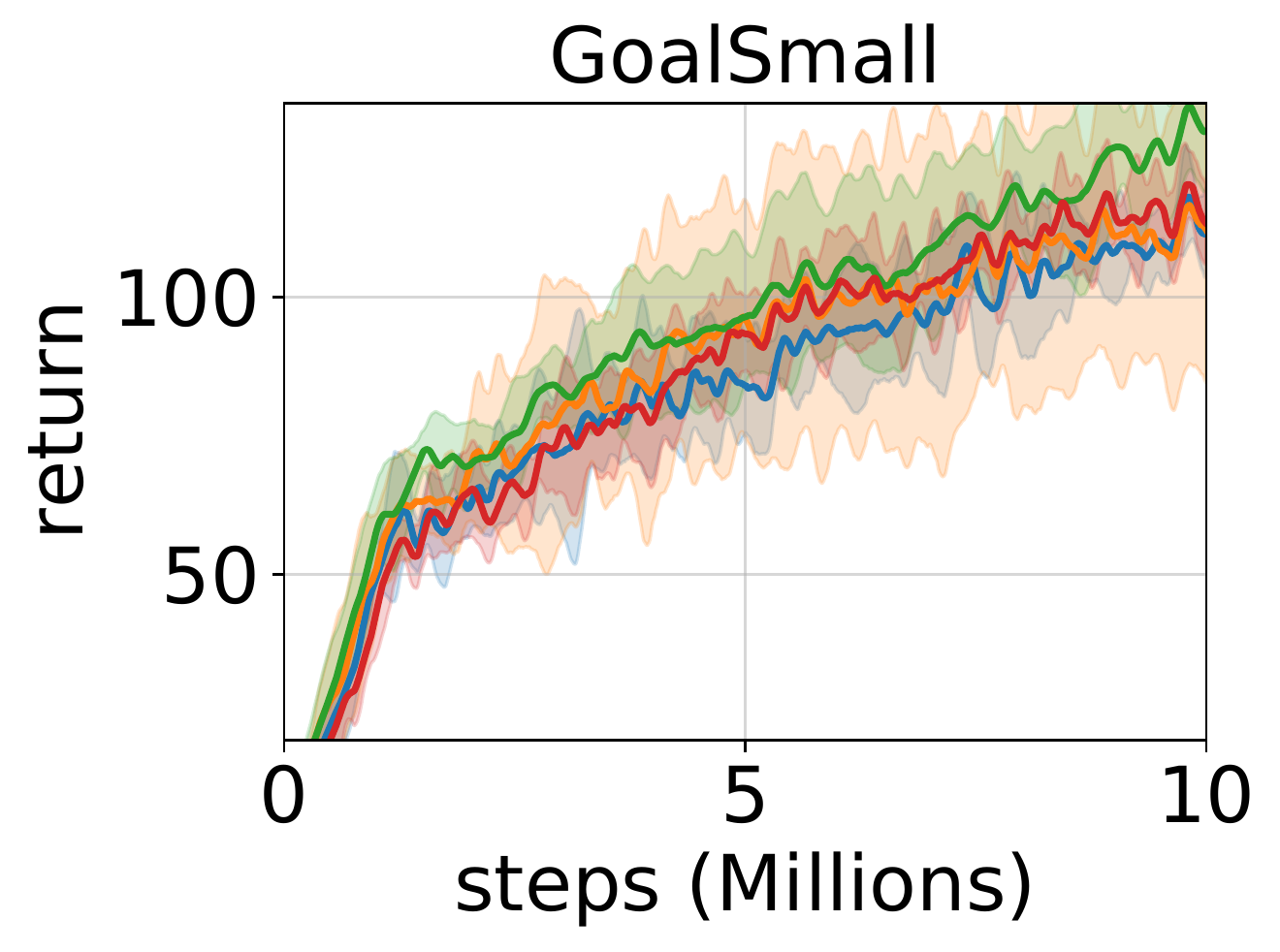}
    \hspace{-8pt}
    \includegraphics[draft=false, width=0.33\linewidth, valign=t]
    {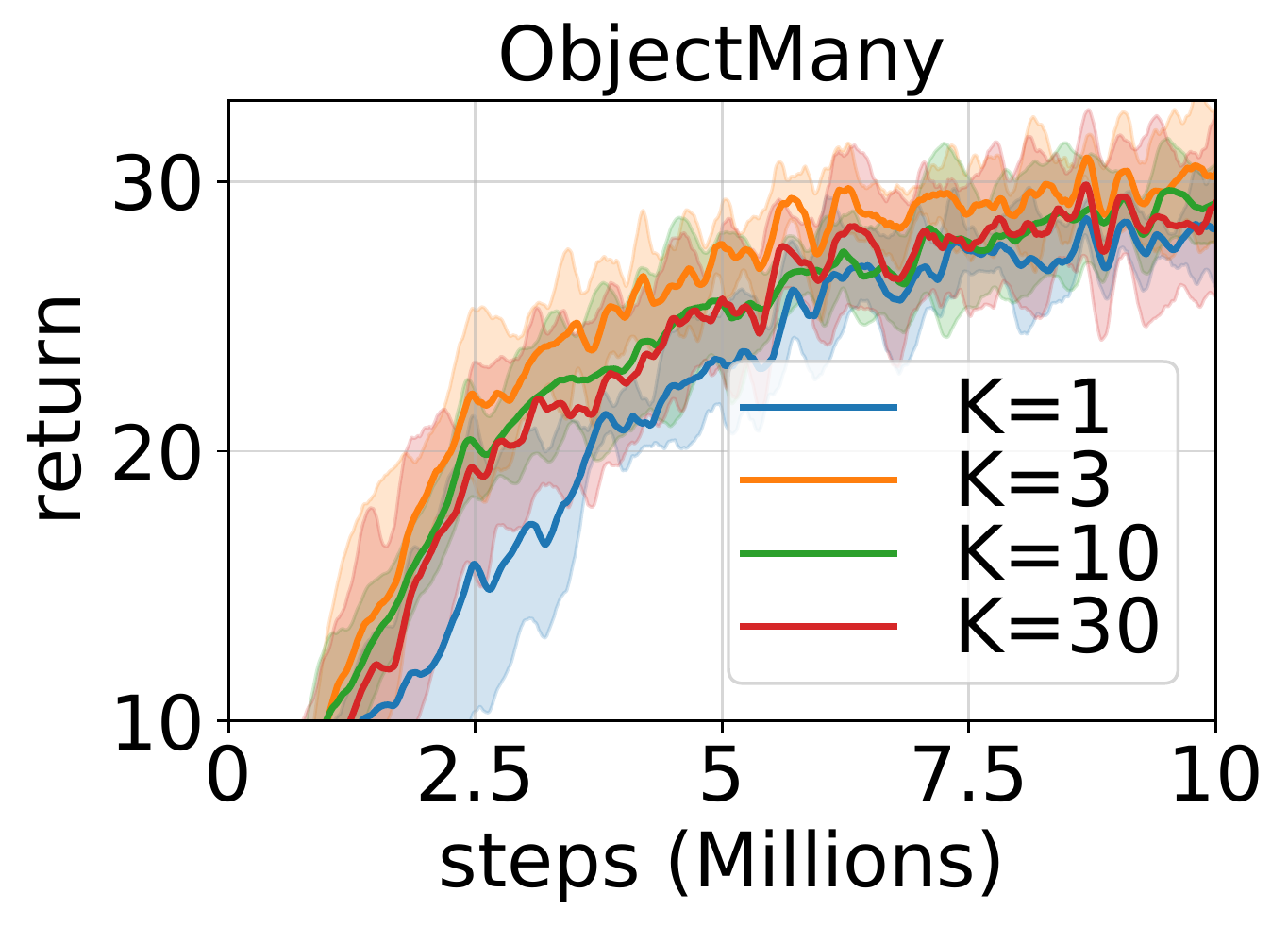}
    \hspace{-8pt}
    \includegraphics[draft=false, width=0.33\linewidth, valign=t]
    {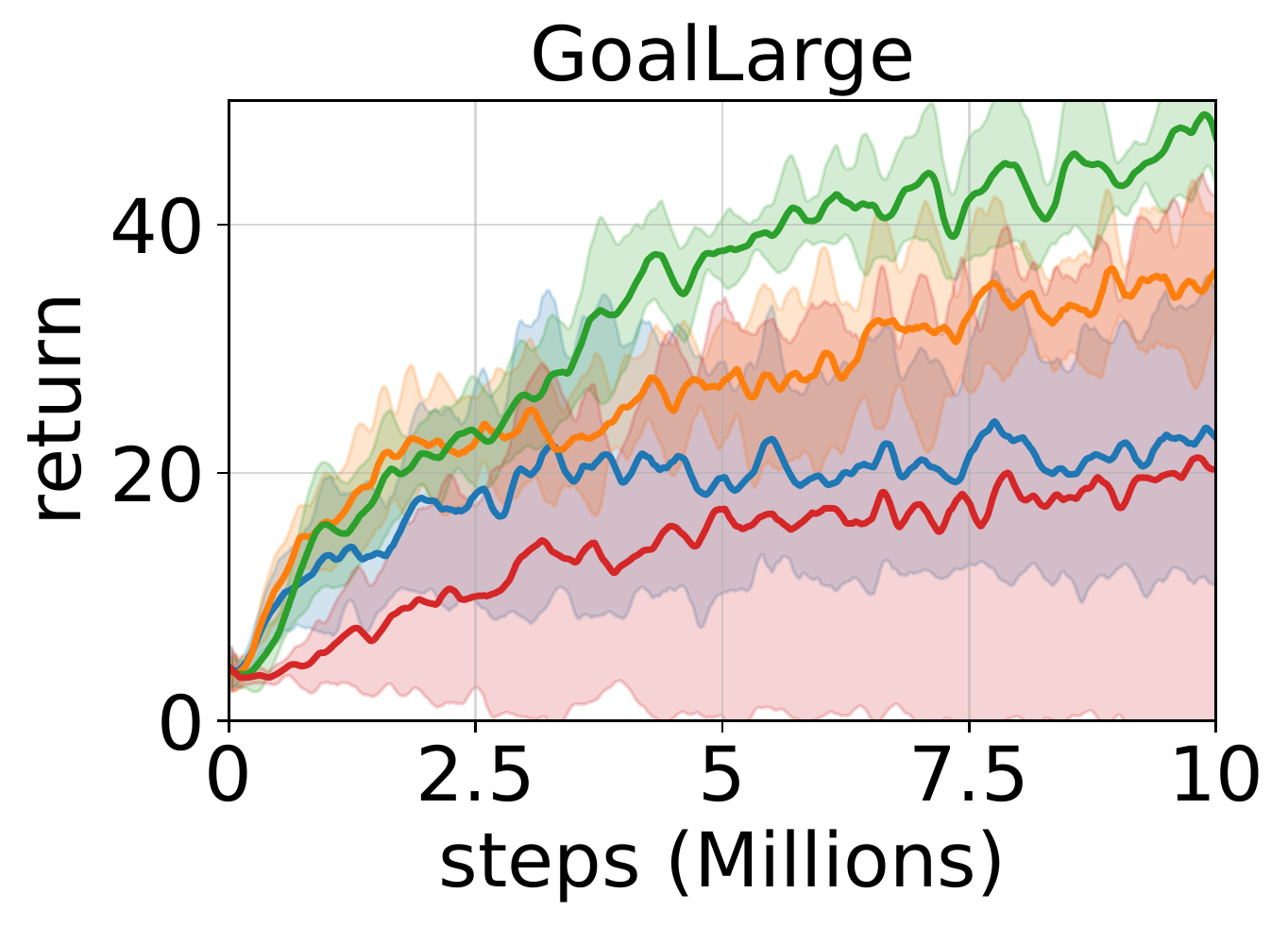}
    \hspace{-10pt}
    \caption{
        Average episode reward of \sprl{} with varying $k=$1, 3, 10, 30 as a function of environment steps for \dmlab{} tasks. Other hyper-parameters are kept same as the best hyper-parameter. The best performance is obtained with $k=10$.
    }
    \label{fig:k}
\end{figure}

\begin{figure}[!h]
    \centering
    \vspace*{-10pt}
    \hspace*{-5pt}
    \includegraphics[draft=false, height=6.7\baselineskip, valign=t]{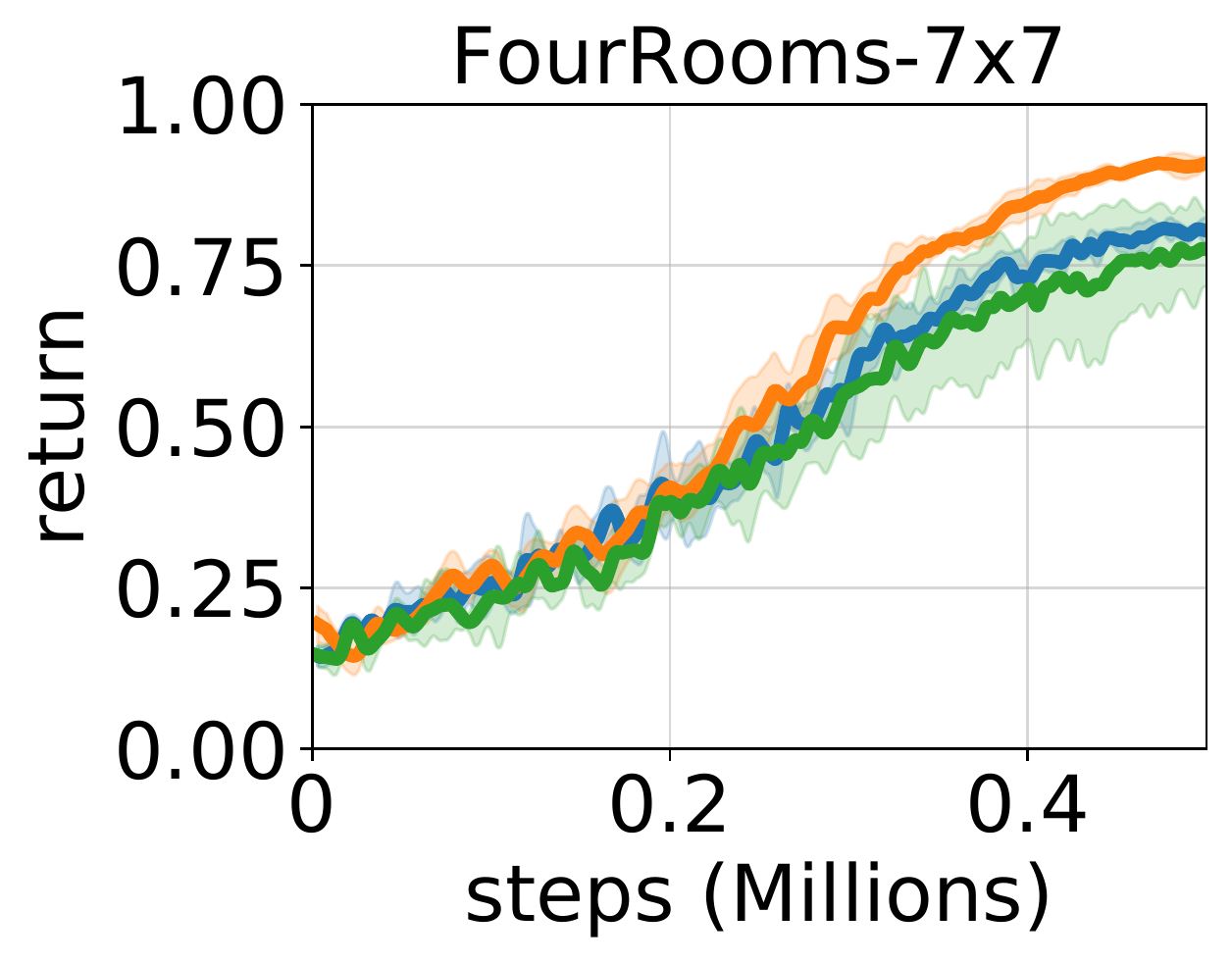}
    \hspace{-6pt}
    \includegraphics[draft=false, height=6.7\baselineskip, valign=t]{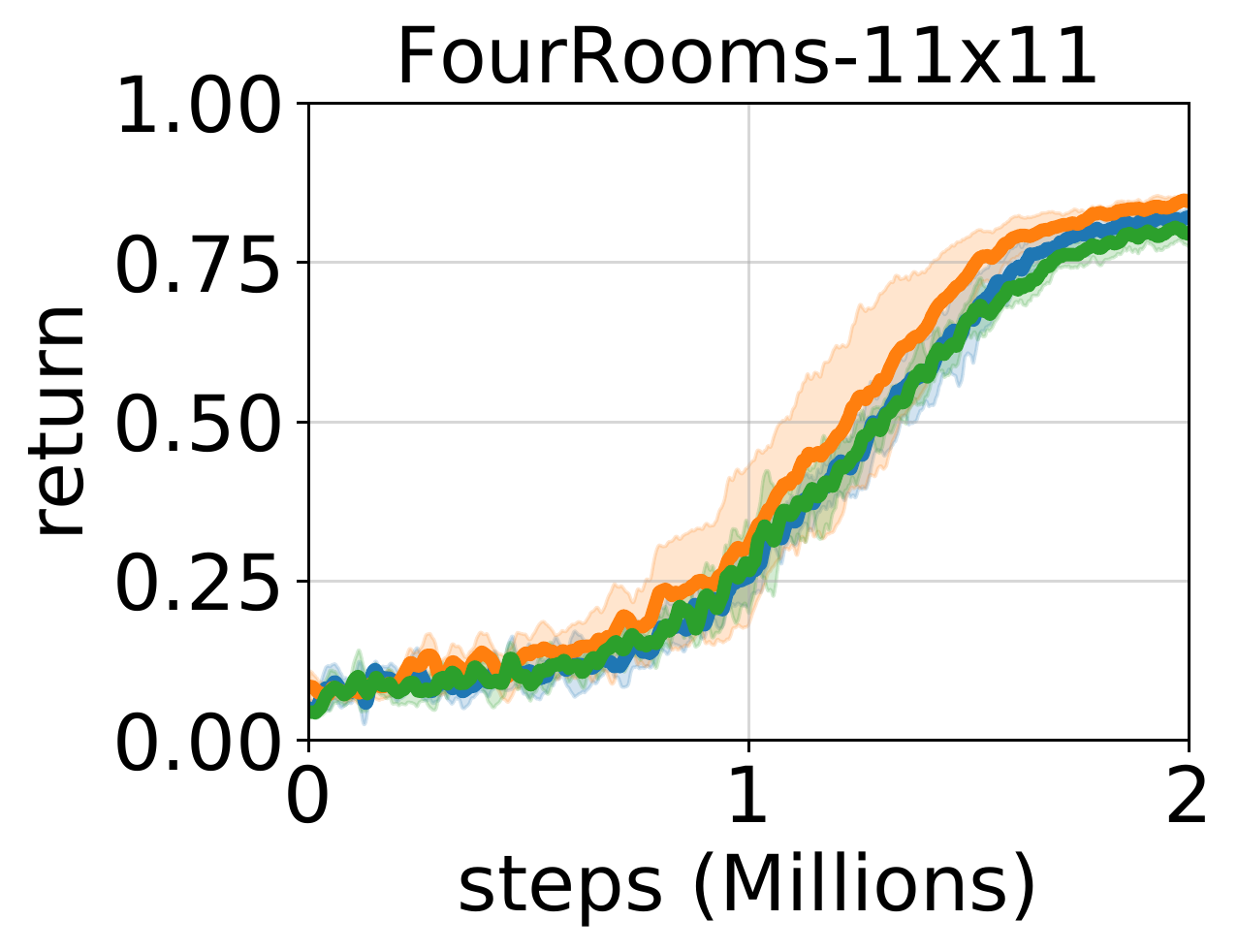}
    \hspace{-6pt}
    \includegraphics[draft=false, height=6.7\baselineskip, valign=t]{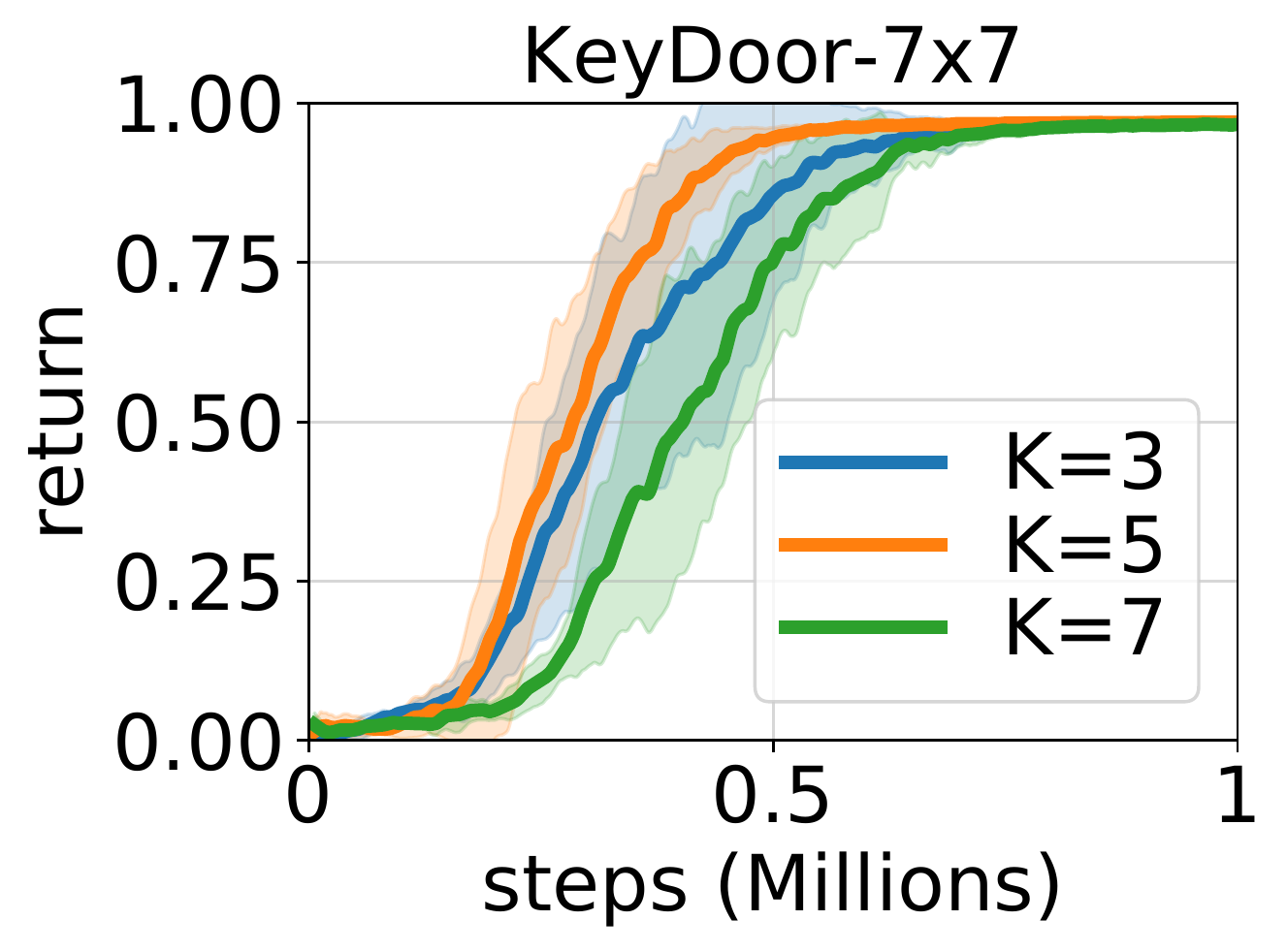}
    \hspace{-6pt}
    \includegraphics[draft=false, height=6.7\baselineskip, valign=t]{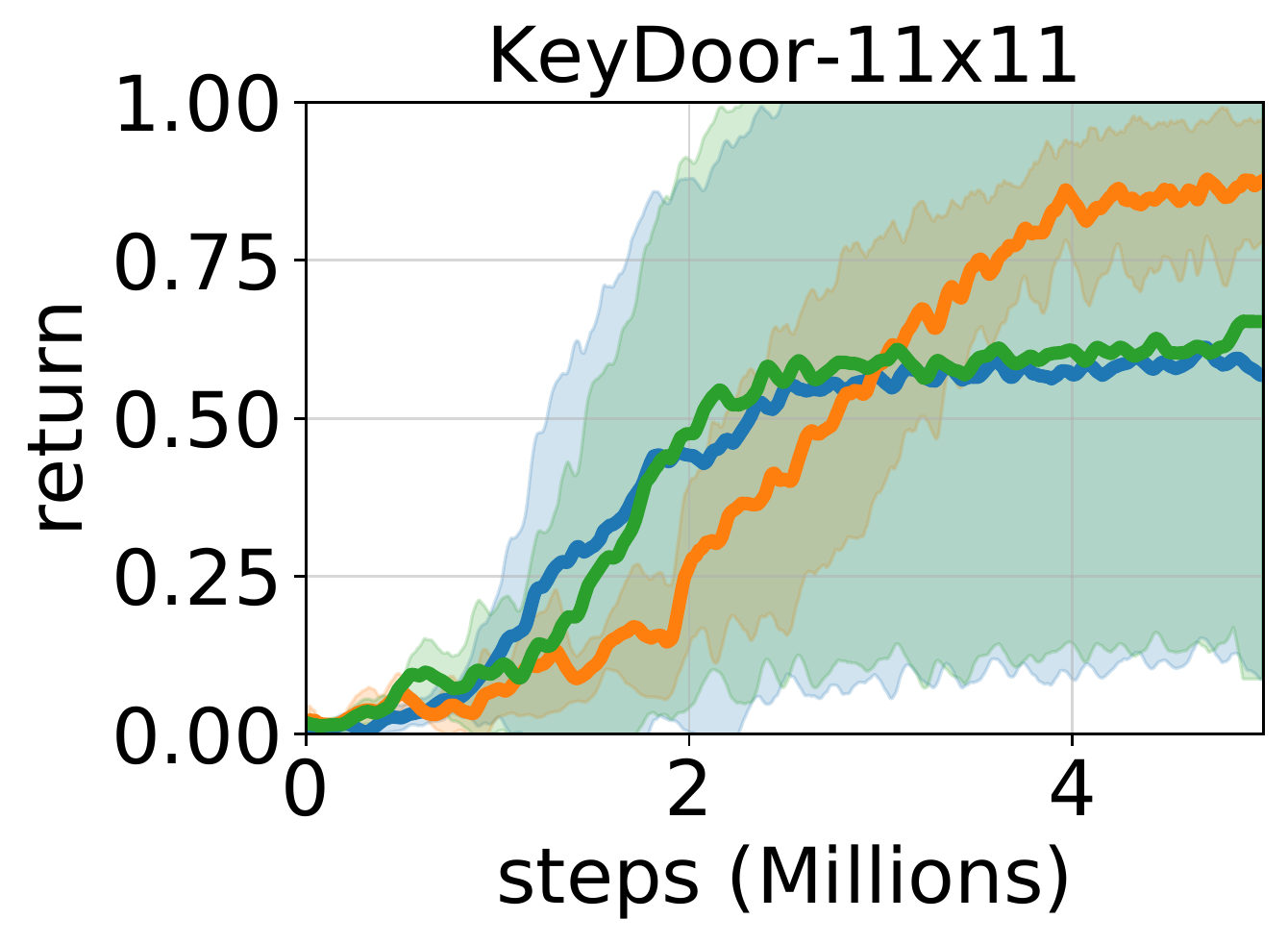}
    \hspace{-5pt}
    \vspace{-5pt}
    \caption{
        Average episode reward of \sprl{} with varying $k=$3, 5, 7 as a function of environment steps for \grid{} tasks. Other hyper-parameters are kept same as the best hyper-parameter. The best performance is obtained with $k=5$.
    }
    \label{fig:minigrid_k}
\end{figure}
\cutparagraphup
\subsection{Effect of $k$}
As proven in \tb{Lemma}~\ref{lem:ksp-reduce} and shown in \Cref{sec:analysis}, the larger $k$, the $k$-shortest constraint promises a larger reduction in policy space, which results in faster learning. However, with our practical implementation of \sprl{} with a learned (imperfect) reachability network, overly large $k$ has a drawback. 
Intuitively speaking, it is harder for policy to satisfy the $k$-shortest constraint, and the supervision signal given by our cost function becomes sparser (\ie, almost always penalized). Figure~\ref{fig:k} and \ref{fig:minigrid_k} shows the performance of \sprl{} on \dmlab{} and \grid{} domains with varying $k$. In both domains, we can see that there exists a ``sweet spot'' that balances between the reduction in policy space and sparsity of the supervision (\eg, $k=10$ for \dmlab{} and $k=5$ for \grid{}). In practice, we performed grid-search over the hyper-parameter $k$.

Training multiple reachability networks with different $k$'s may achieve the advantages of both low and high $k$. But practically, training multiple reachability networks requires high computational cost and more extensive hyperparameter search, which can be intractable. Thus, we instead tried the curriculum learning of $k$; \ie, starting from low $k$ and gradually increasing it up to the target $k$. However, we found that curriculum learning makes the training of the reachability network unstable when $k$ changes which results in lower performance.
\cutsectiondown
\begin{figure}[!h]
    \centering
    \hspace{-10pt}
    \includegraphics[draft=false, width=0.33\linewidth, valign=t]
    {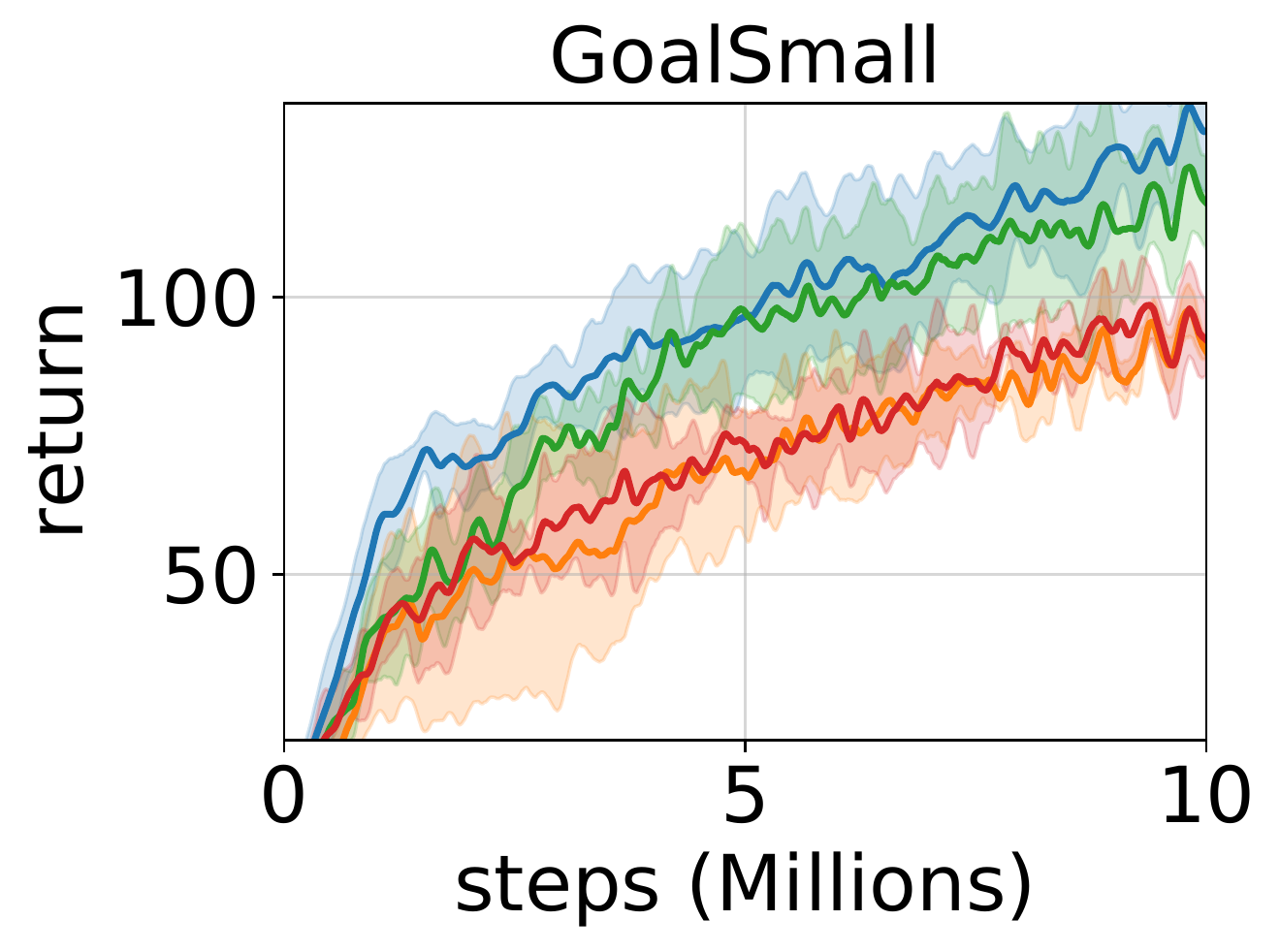}
    \hspace{-8pt}
    \includegraphics[draft=false, width=0.33\linewidth, valign=t]
    {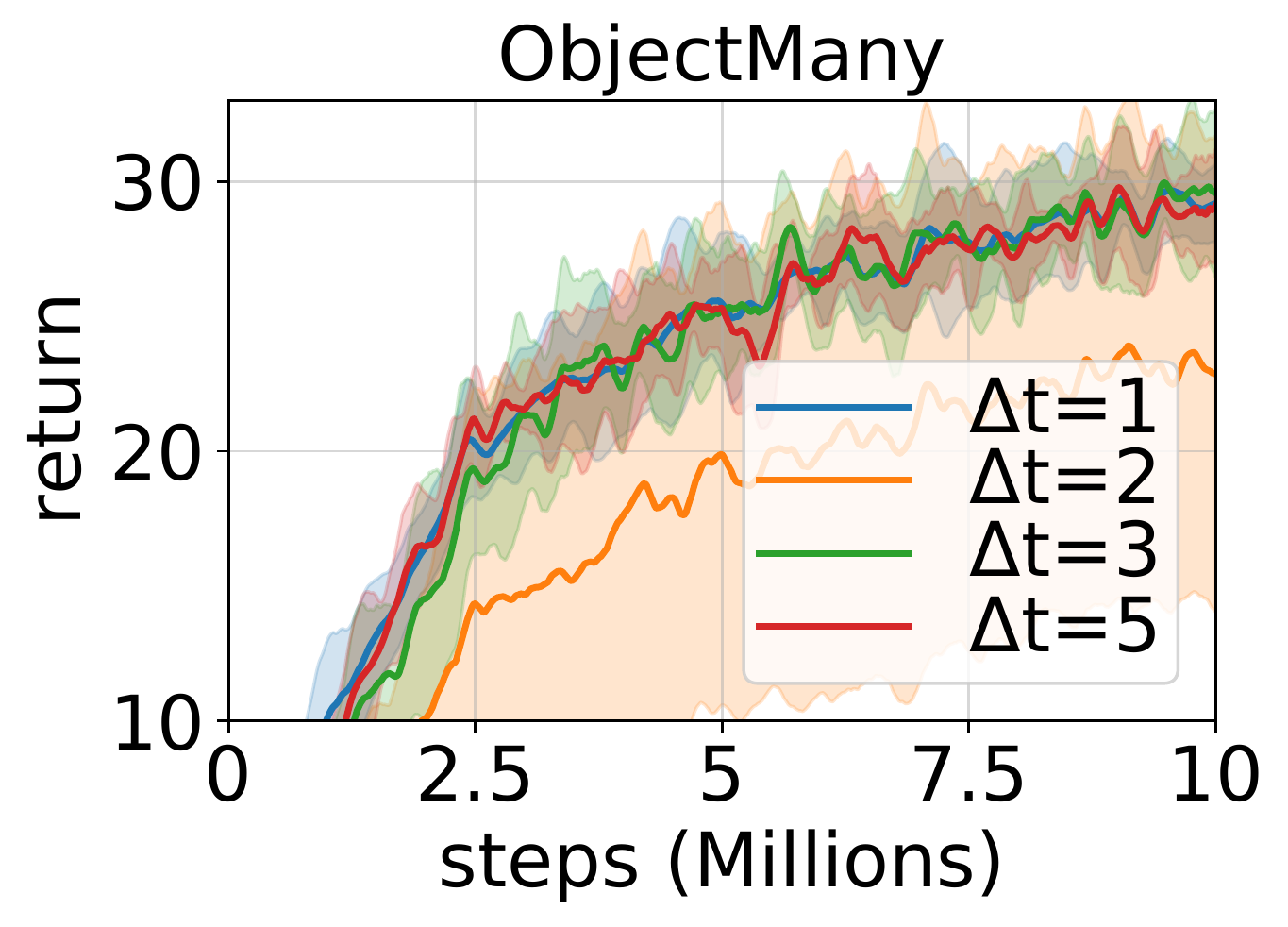}
    \hspace{-8pt}
    \includegraphics[draft=false, width=0.33\linewidth, valign=t]
    {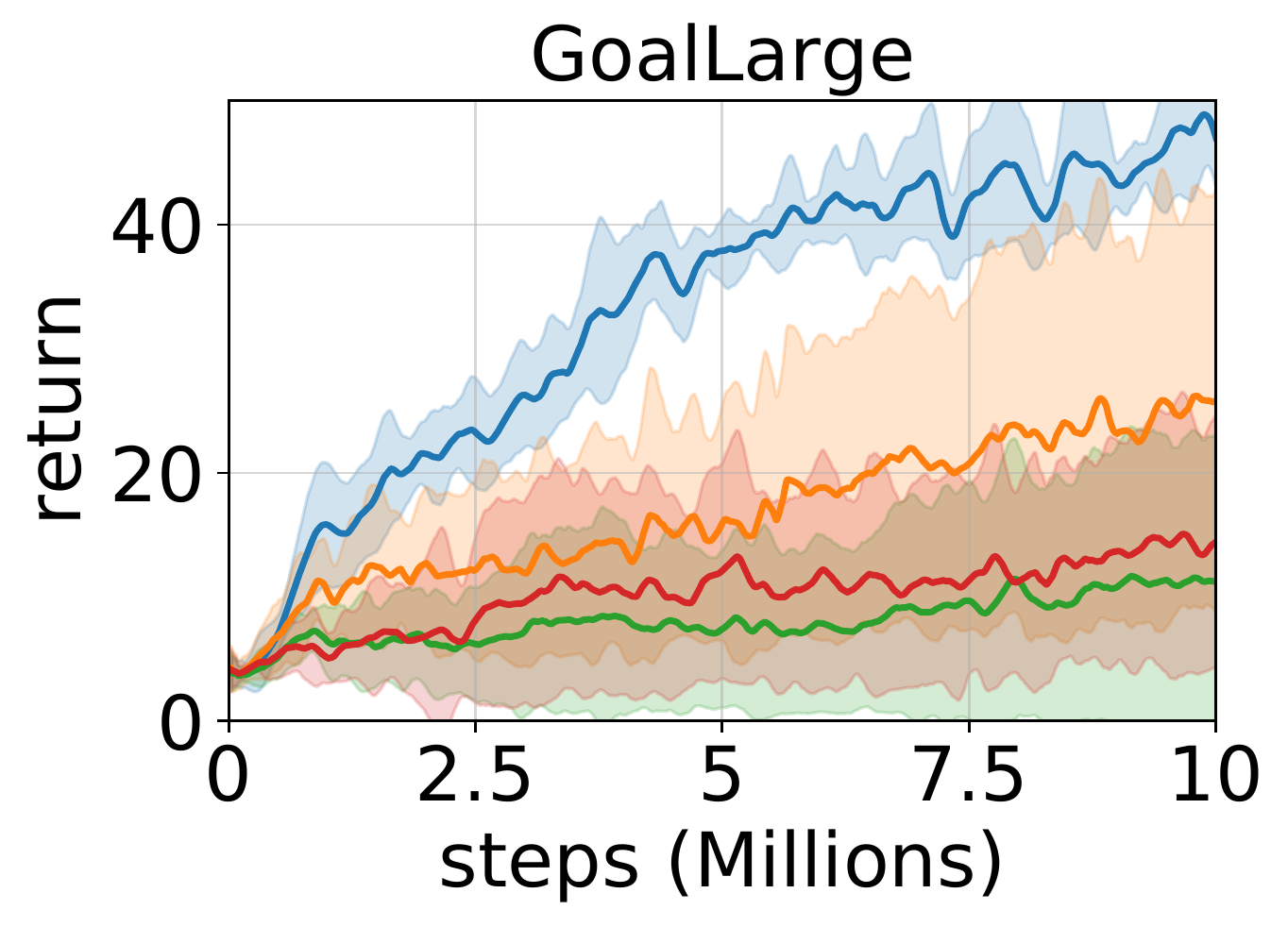}
    \hspace{-10pt}
    \caption{
        Average episode reward of \sprl{} with varying $\Delta t=$1, 2, 3, 5 as a function of environment steps for \dmlab{} tasks. Other hyper-parameters are kept same as the best hyper-parameter. The best performance is obtained with $\Delta t=1$.
    }
    \label{fig:tol}
\end{figure}
\begin{figure}[!h]
    \centering
    \vspace*{0pt}
    \hspace*{-5pt}
    \includegraphics[draft=false, height=6.7\baselineskip, valign=t]{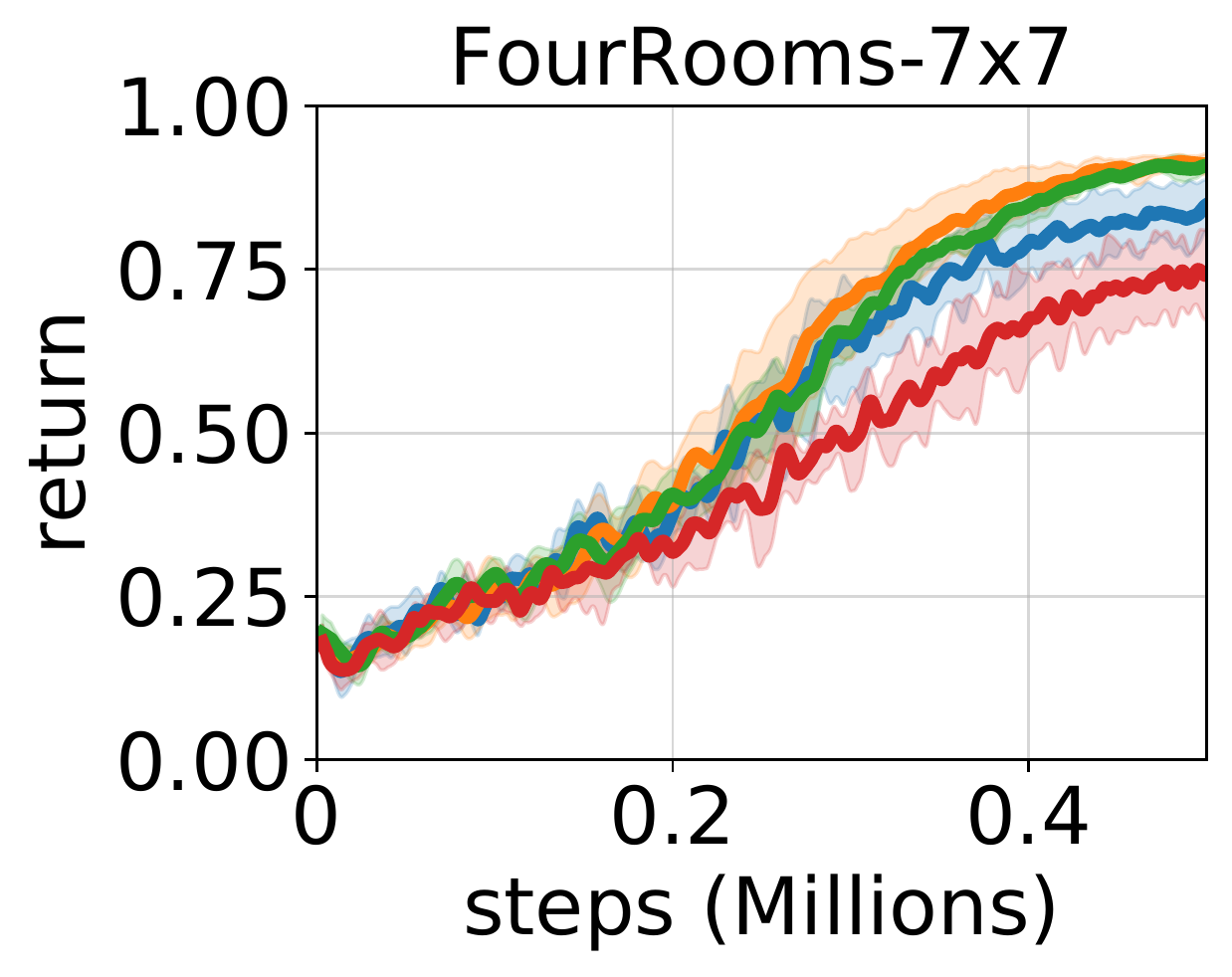}
    \hspace{-6pt}
    \includegraphics[draft=false, height=6.7\baselineskip, valign=t]{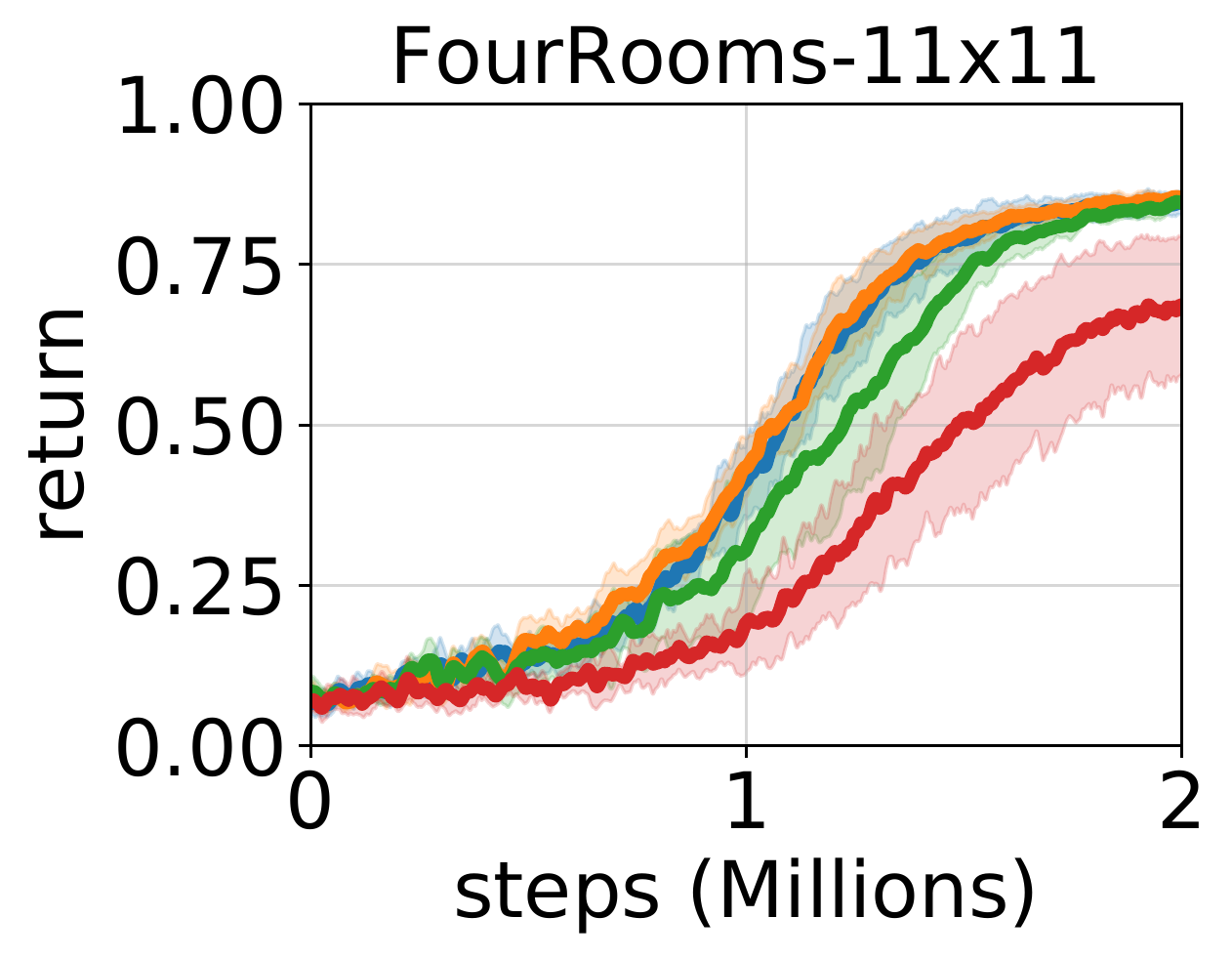}
    \hspace{-6pt}
    \includegraphics[draft=false, height=6.7\baselineskip, valign=t]{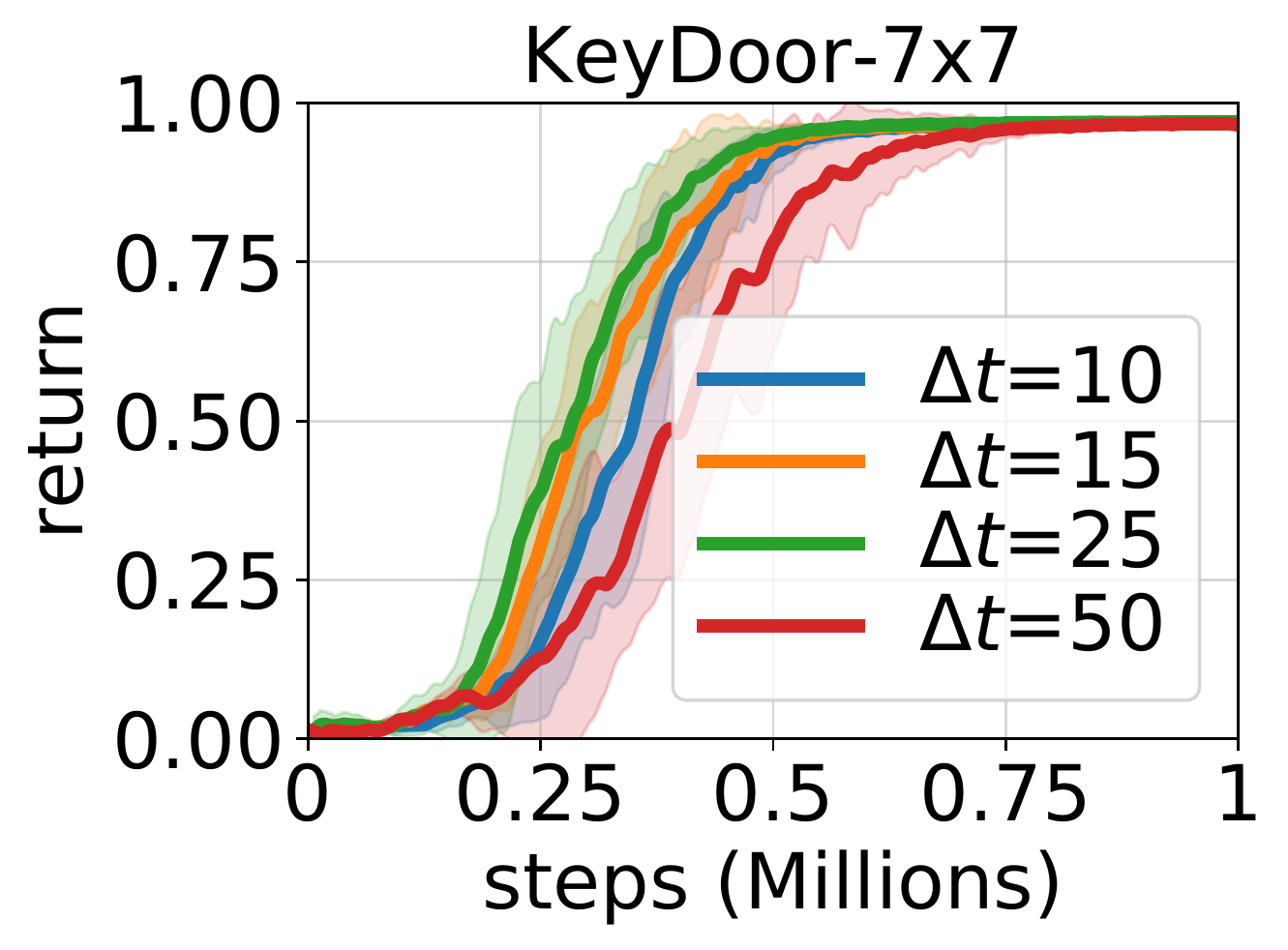}
    \hspace{-6pt}
    \includegraphics[draft=false, height=6.7\baselineskip, valign=t]{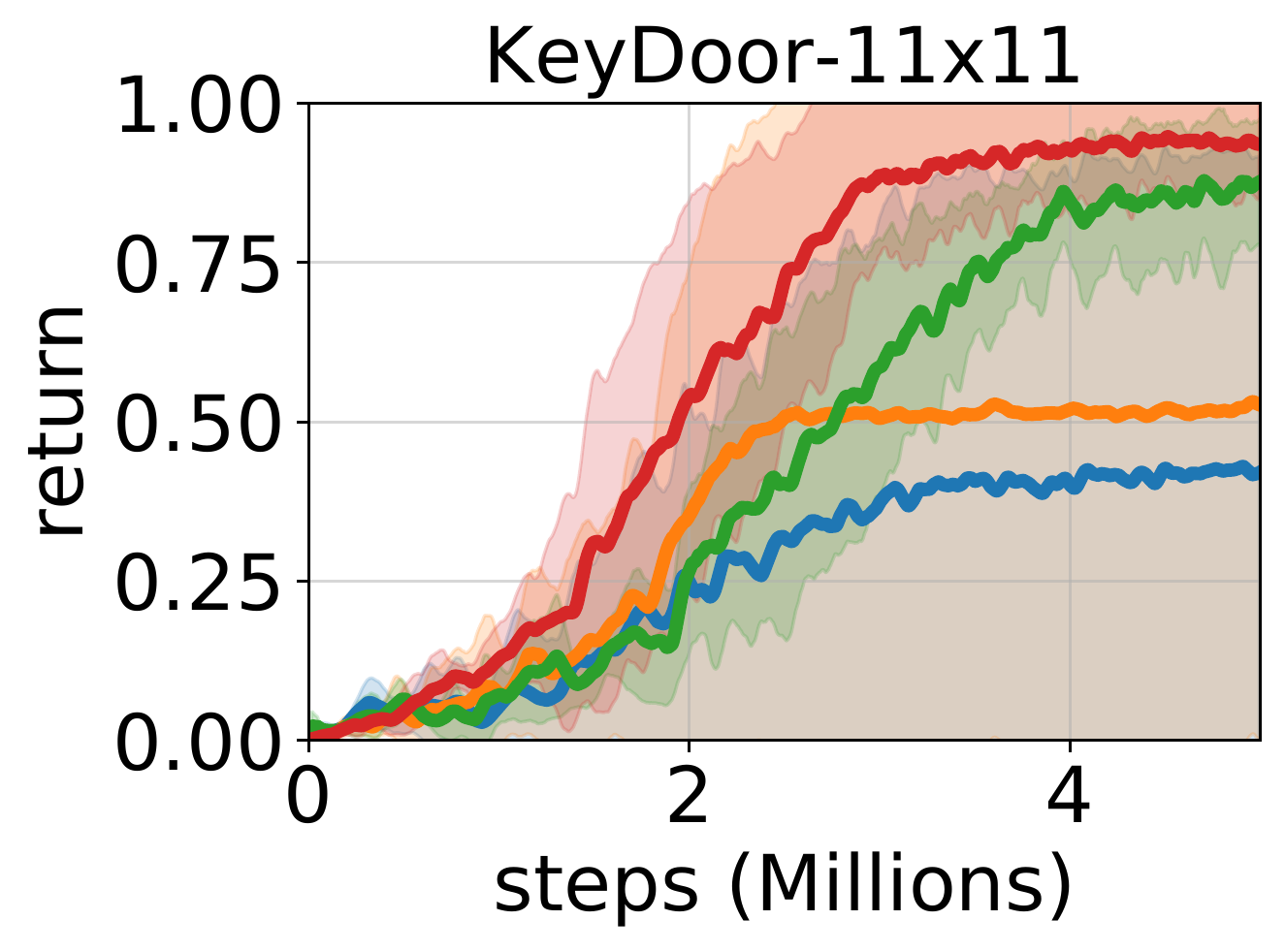}
    \hspace{-5pt}
    \vspace{-5pt}
    \caption{
        Average episode reward of \sprl{} with varying $\Delta t=$10, 15, 25, 50 as a function of environment steps for \grid{} tasks. Other hyper-parameters are kept same as the best hyper-parameter. The best performance is obtained with $\Delta t=25$.
    }
    \label{fig:minigrid_tol}
\end{figure}
\cutparagraphup
\subsection{Effect of tolerance $\Delta t$}
Adding the tolerance $\Delta t$ to our \ksp{} constraint makes it ``softer'' by allowing $\Delta t$-steps of redundancy in transition (See~\Cref{eq:cost-tol}). Intuitively, a small tolerance may improve the stability of RNet by incorporating a possible noise in RNet prediction, but a very large tolerance will make it less effective in removing sub-optimality in transition.
Figure~\ref{fig:tol} and \ref{fig:minigrid_tol} show the performance of \sprl{} on \dmlab{} and \grid{} domains with varying tolerance $\Delta t$. Similar to $k$, we can see that there exists a ``sweet spot'' that balances between the reduction in policy space and stabilization of noisy RNet output (\eg, $\Delta t=25$) in \grid{}.
Note that the best tolerance values for \dmlab{} and \grid{} are vastly different. This is mainly because we used \textit{multiple tolerance sampling} (See~\Cref{appendix:dmlab-detail}) for \dmlab{} but not for \grid{}. Since the multiple tolerance sampling also improves the stability of RNet, larger tolerance has less benefit compared to its disadvantage.

\begin{figure}[!h]
    \centering
    \hspace{-10pt}
    \includegraphics[draft=false, width=0.27\linewidth, valign=t]
    {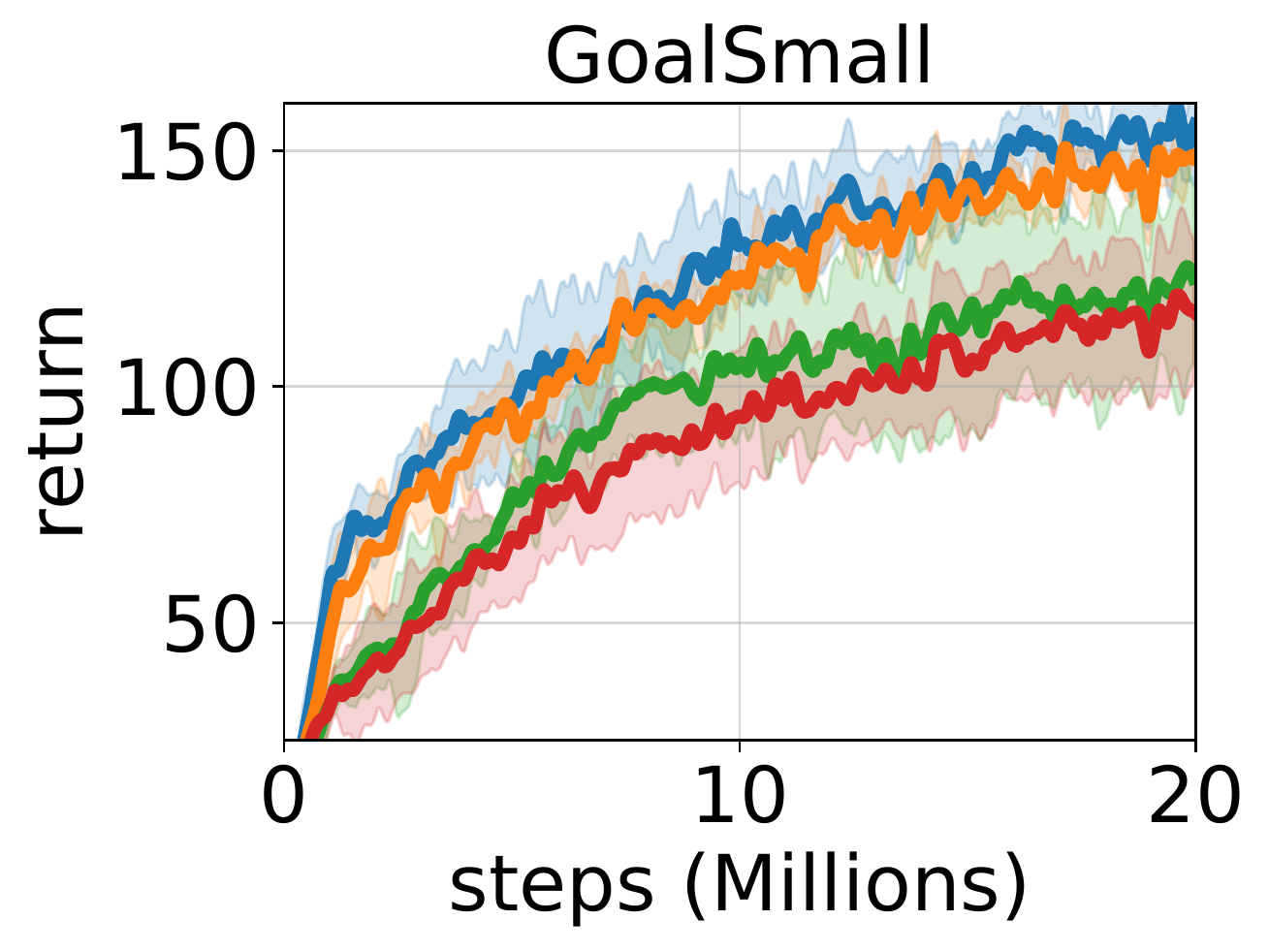}
    \hspace{-8pt}
    \includegraphics[draft=false, width=0.27\linewidth, valign=t]
    {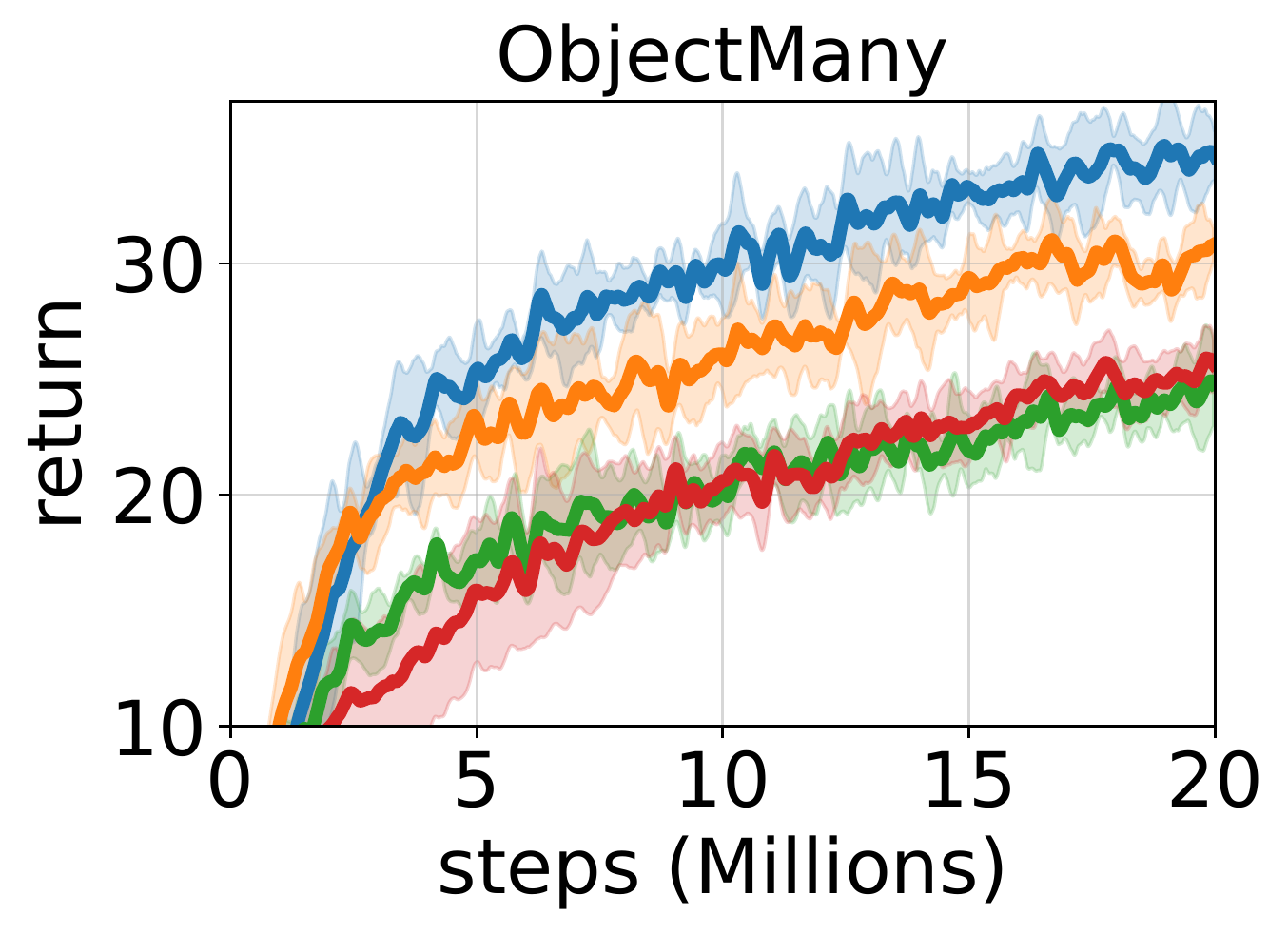}
    \hspace{-8pt}
    \includegraphics[draft=false, width=0.27\linewidth, valign=t]
    {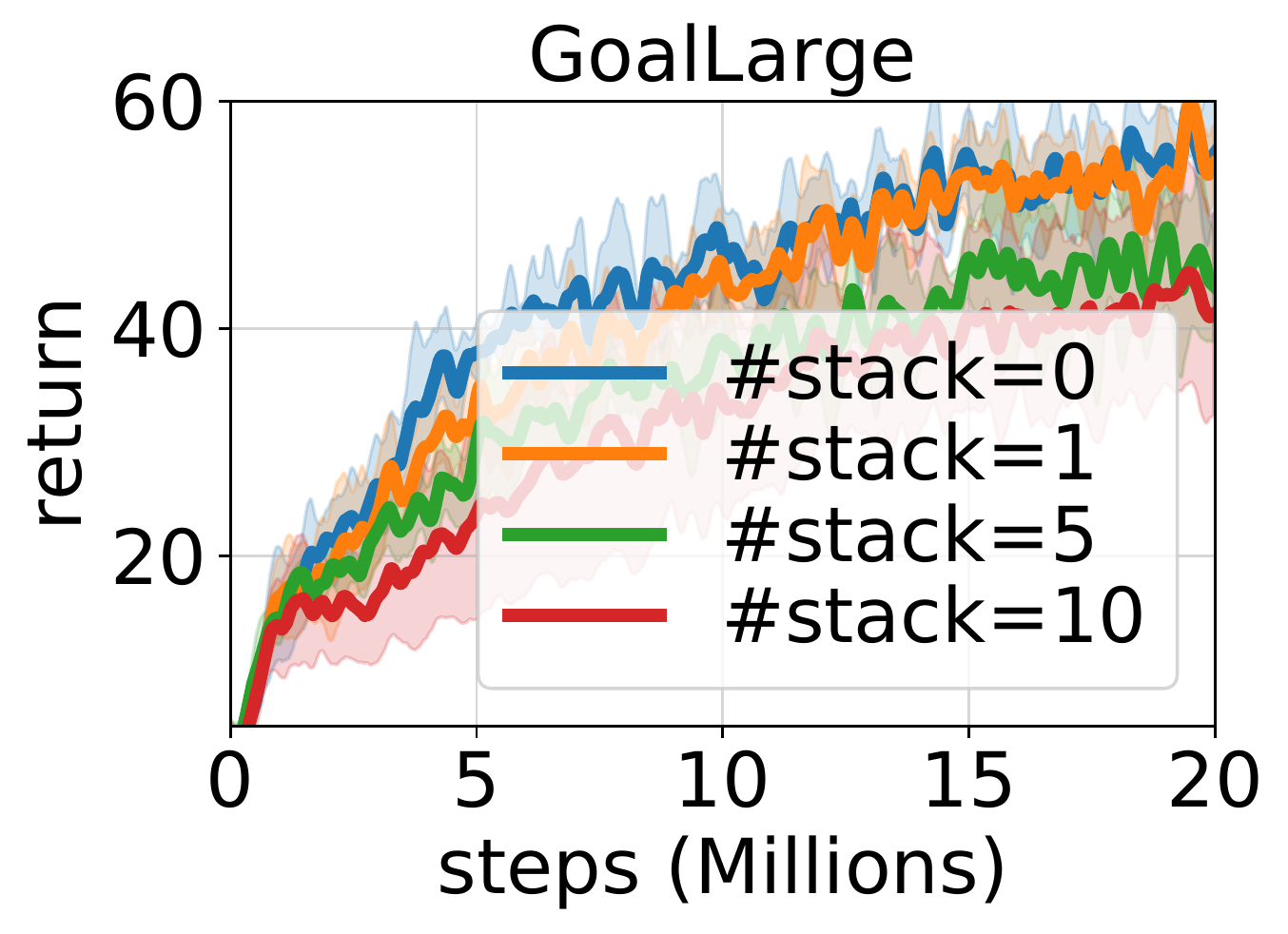}
    \hspace{-10pt}
    \caption{
        Average episode reward of \sprl{} with varying observation stacking dimension of 0, 1, 5, 10 as a function of environment steps for \dmlab{} tasks. Other hyper-parameters are kept same as the best hyper-parameter. The best performance is obtained without stacking (\ie, \#stack=0).
    }
    \label{fig:stack}
\end{figure}

\begin{figure}[!h]
    \centering
    \vspace*{-10pt}
    \hspace*{-5pt}
    \includegraphics[draft=false, height=6.7\baselineskip, valign=t]{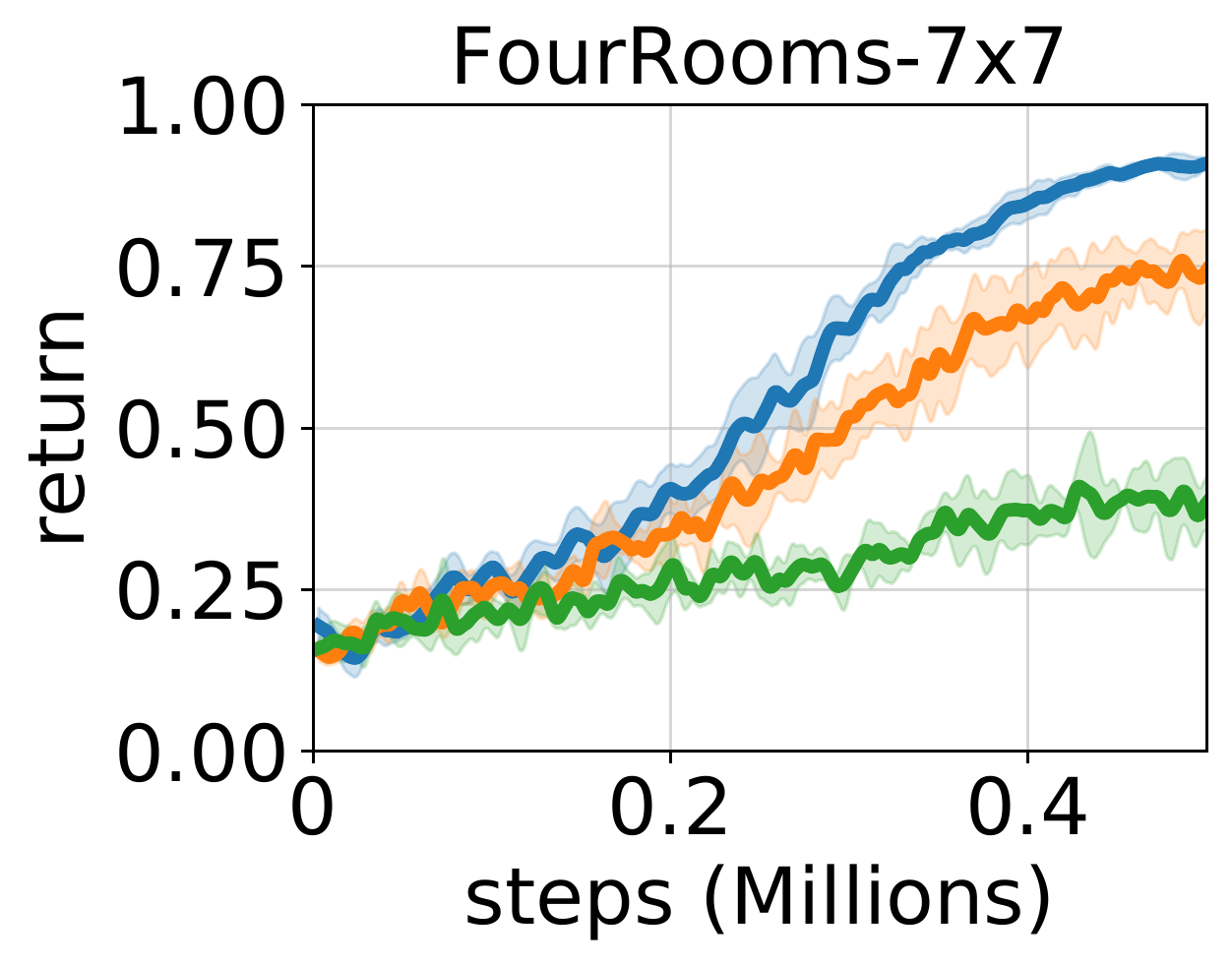}
    \hspace{-6pt}
    \includegraphics[draft=false, height=6.7\baselineskip, valign=t]{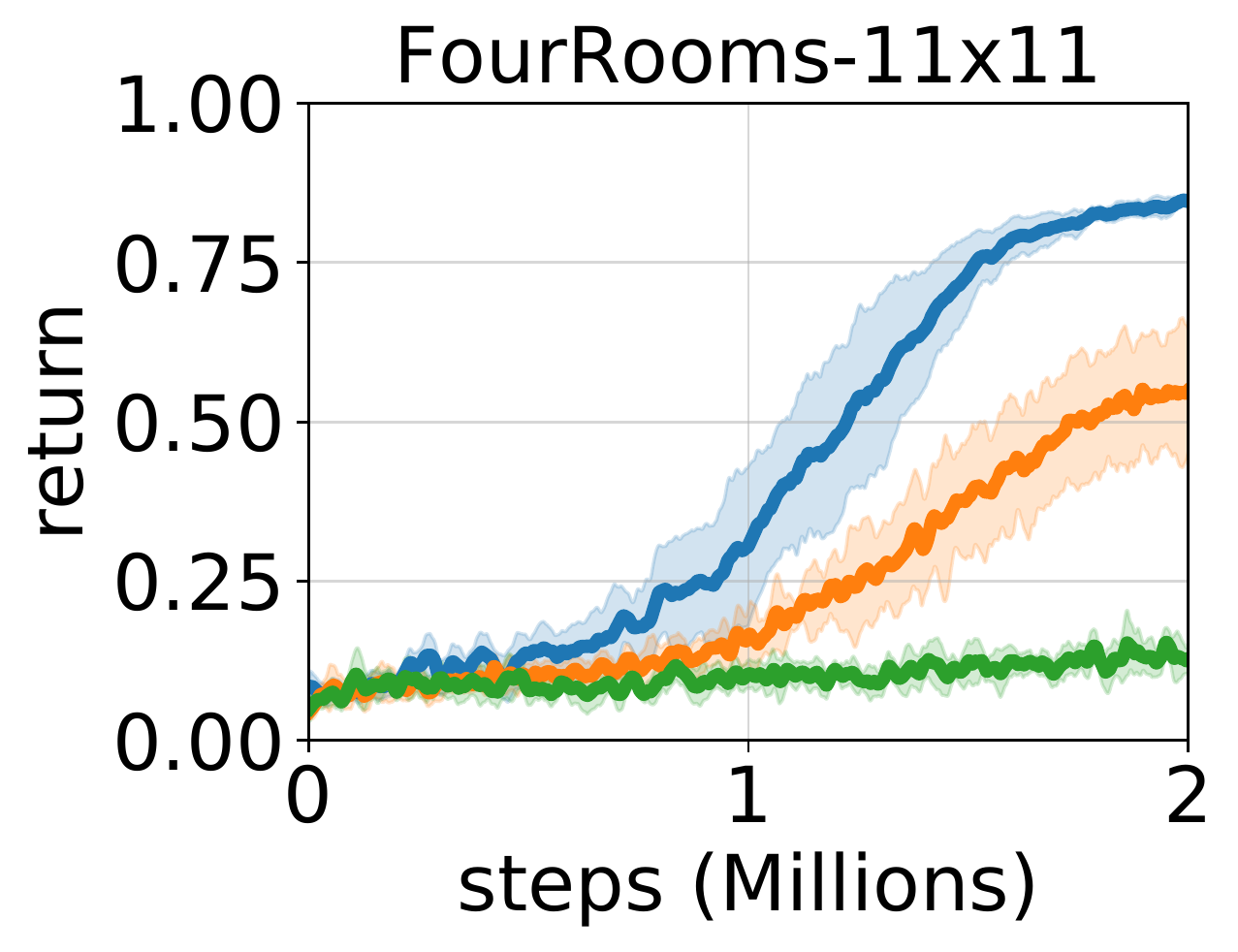}
    \hspace{-6pt}
    \includegraphics[draft=false, height=6.7\baselineskip, valign=t]{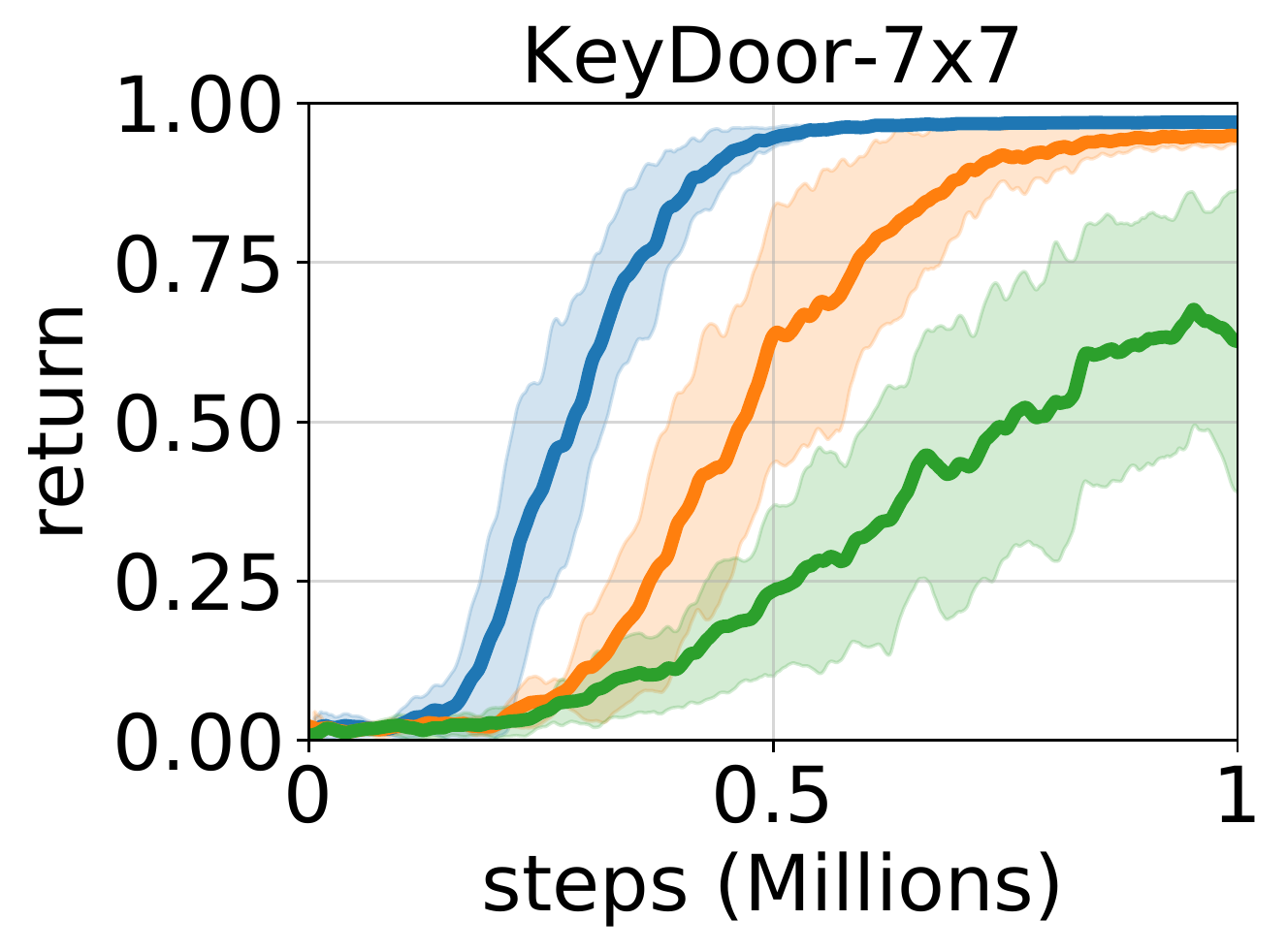}
    \hspace{-6pt}
    \includegraphics[draft=false, height=6.7\baselineskip, valign=t]{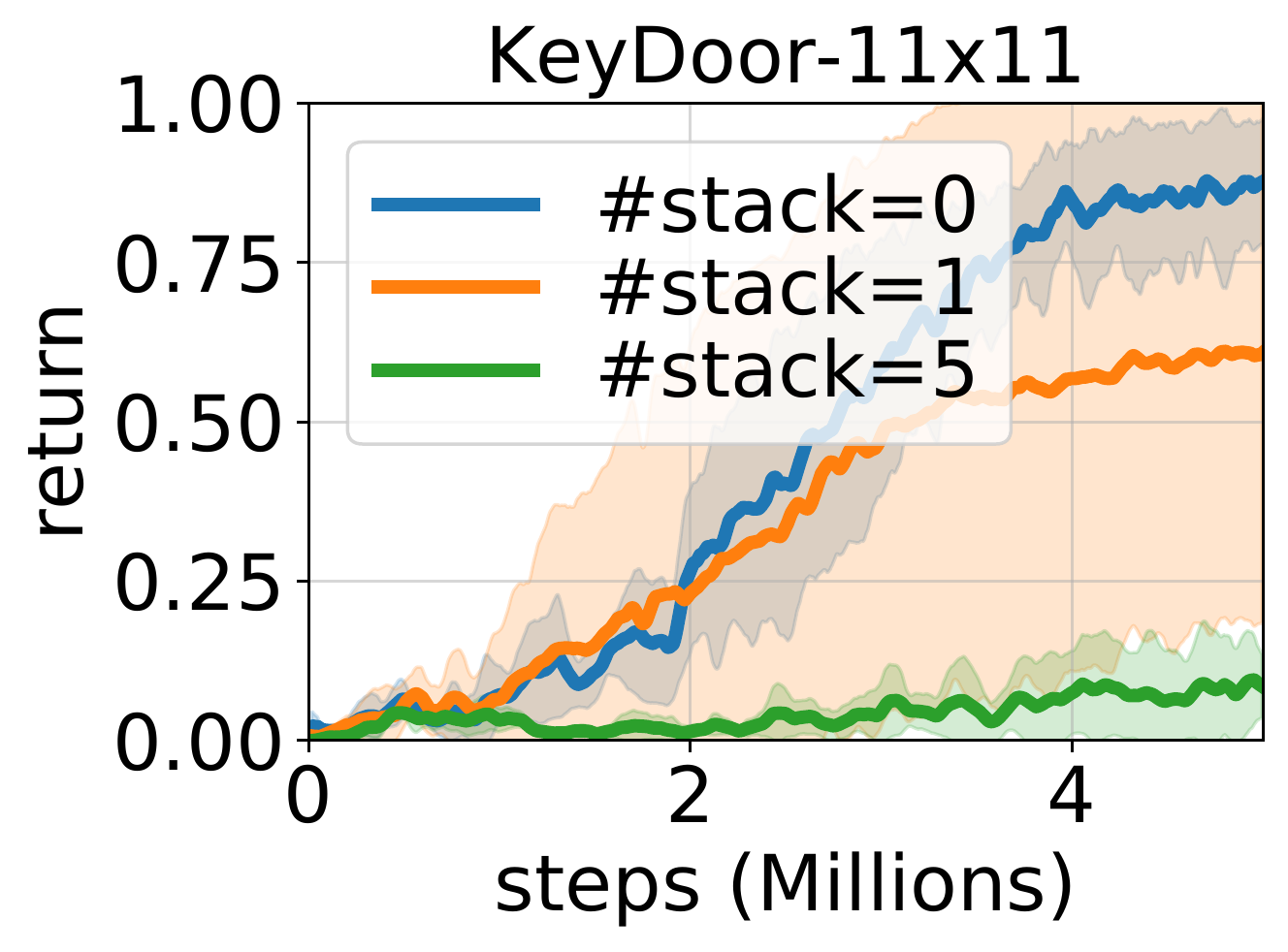}
    \hspace{-5pt}
    \vspace{-5pt}
    \caption{
        Average episode reward of \sprl{} with varying observation stacking dimension of 0, 1, 5 as a function of environment steps for \grid{} tasks. Other hyper-parameters are kept same as the best hyper-parameter. The best performance is obtained without stacking (\ie, \#stack=0)
    }
    \label{fig:minigrid_stack}
\end{figure}

\cutparagraphup
\subsection{Stacking observation}\label{appendix:stack}
The CMDP with \ksp{} constraint becomes the $(k+1)$-th order MDP as shown in~\Cref{eq:cost}. Thus, in theory, the policy should take current state $s_t$ augmented by stacking the $k$ previous states as input: $[s_{t-k}, s_{t-k+1}\ldots, s_t]$, where $[\cdot]$ is a stacking of the pixel observation along the channel (\ie, color) dimension. However, stacking the observation may not lead to the best empirical results in practice.
Figure~\ref{fig:stack} and~\ref{fig:minigrid_stack} show the performance of \sprl{} on \dmlab{} and \grid{} domains with varying stacking dimensions. For stack=$m$, we stacked the observation from $t-m$ to $t$: $[s_{t-m}, s_{t-m+1}, \ldots, s_t]$. We experimented up to $m=k$: up to $m=10$ for \dmlab{} and $m=5$ for \grid{}. The result shows that stacking the observation does not necessarily improve the performance for MDP order greater than 1, which is often observed when the function approximation is used (\eg, \citet{savinov2018episodic}). Thus, we did not augment the observation in all the experiments.

\begin{figure}[!h]
    \centering
    \vspace*{-10pt}
    \hspace*{-5pt}
    \includegraphics[draft=false, height=6.7\baselineskip, valign=t]{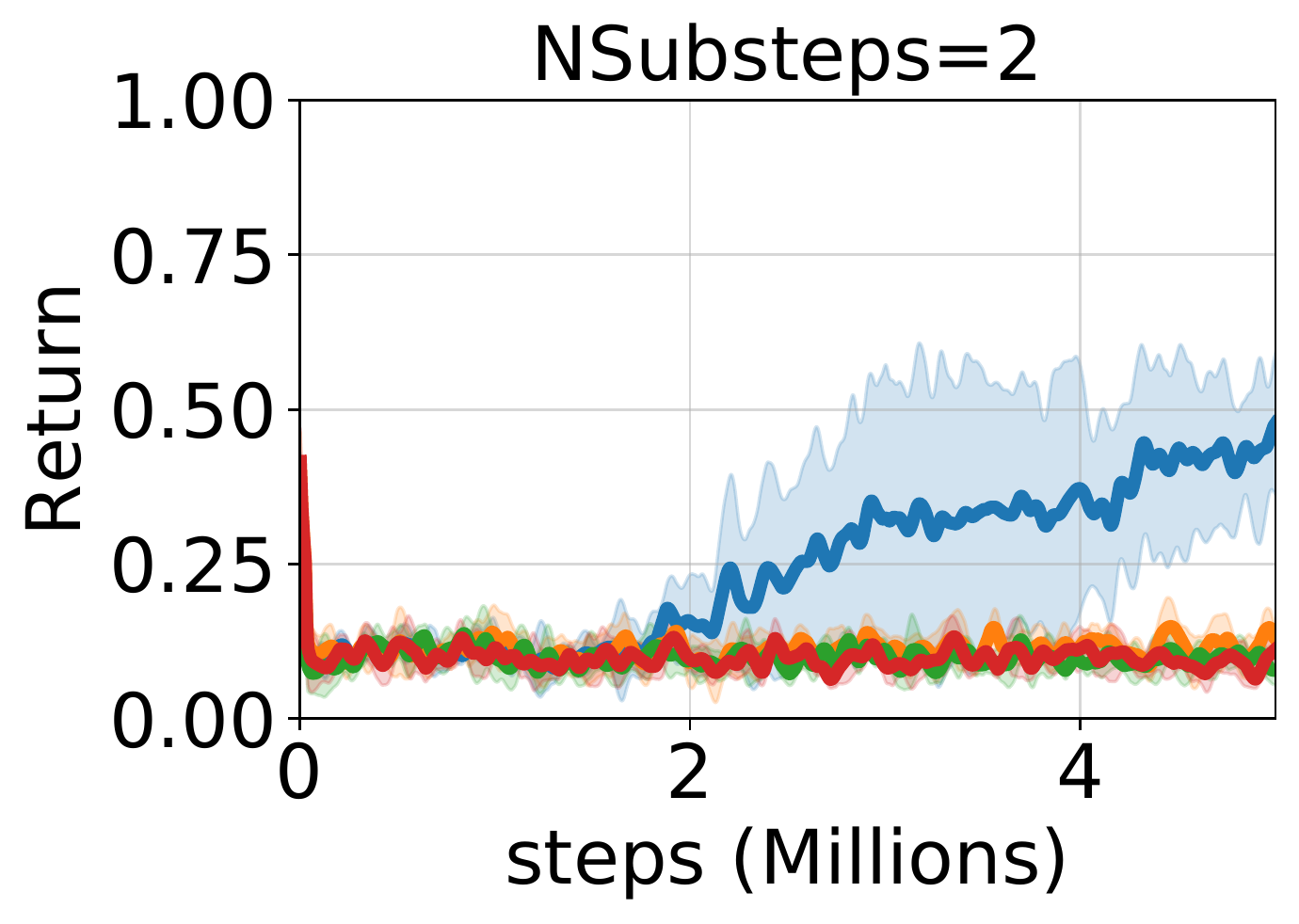}
    \hspace{-6pt}
    \includegraphics[draft=false, height=6.7\baselineskip, valign=t]{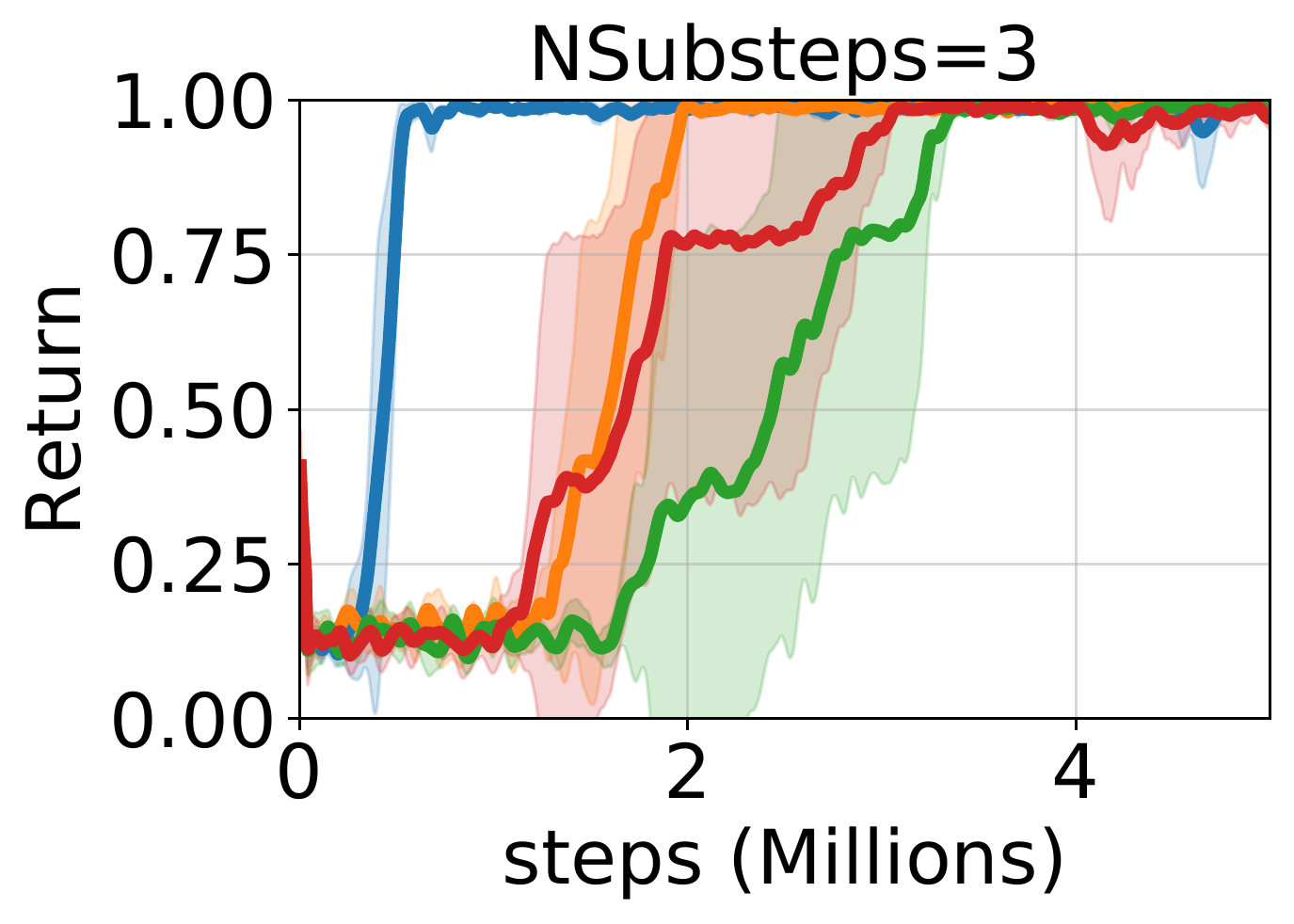}
    \hspace{-6pt}
    \includegraphics[draft=false, height=6.7\baselineskip, valign=t]{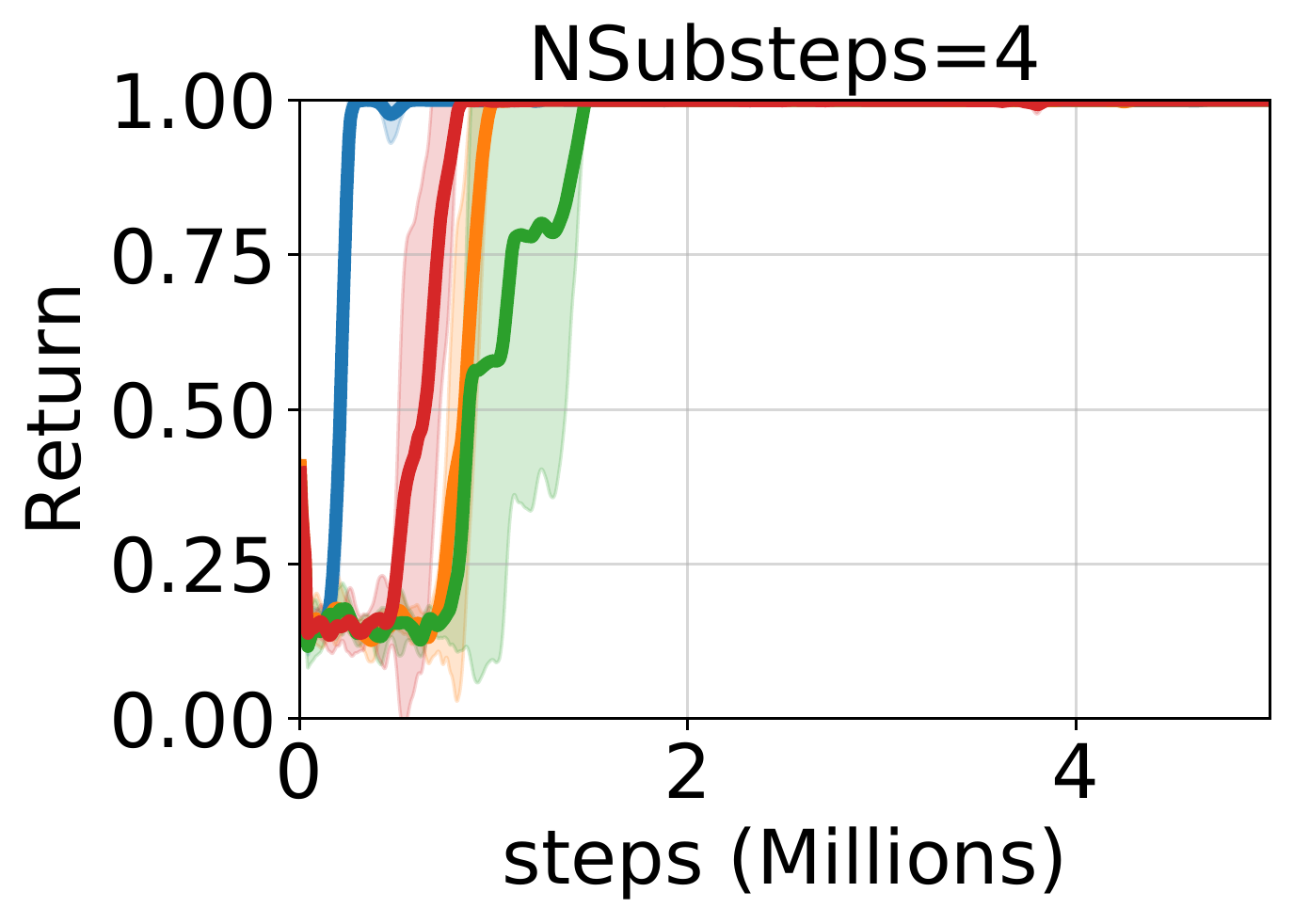}
    \hspace{-6pt}
    \includegraphics[draft=false, height=6.7\baselineskip, valign=t]{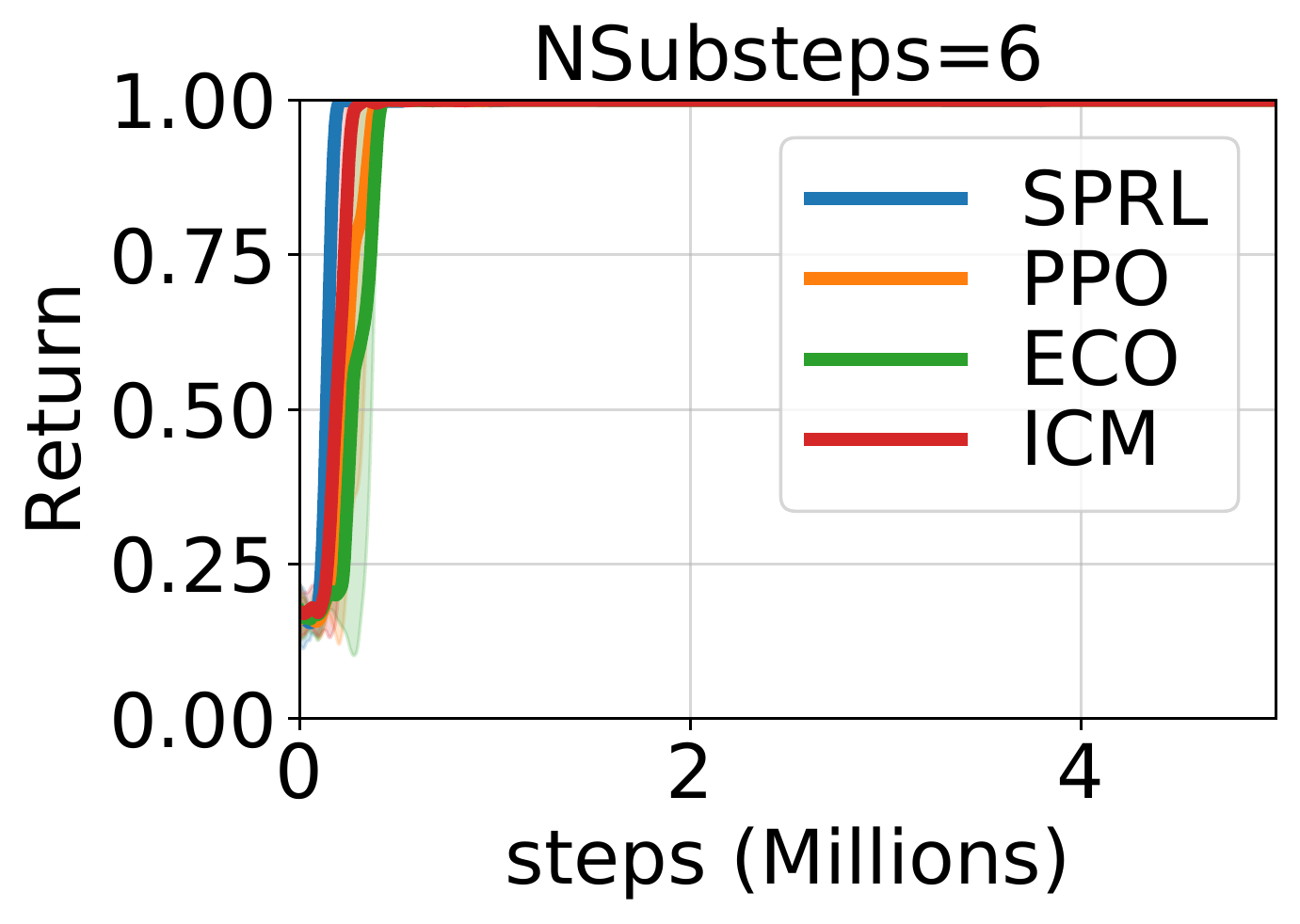}
    \hspace{-5pt}
    \vspace{-5pt}
    \caption{
        Average episode reward of \sprl{} with varying sparsity controlled by substep $N=2, 3, 4, 6$ for \fetch{} tasks. Other hyper-parameters are kept same as the best hyper-parameter. As sparsity grows, the difference between \sprl{} and the baselines increase accordingly.
    }
    \label{fig:fetch_sparsity}
\end{figure}

\cutparagraphup
\subsection{Effect of Sparsity of the reward }\label{appendix:sparsity}
\def\substep{N_\text{substep}\xspace}
We evaluated \sprl{} on \freach with a smaller ``substep'' (\ie, the number of action repetitions) $\substep$ to make the reward sparser, while keeping the episode length $\times$ substep the same to ensure the task is solvable within the episode length.
The default $\substep$ is 20, and we used $\substep=2,3,4,6$.~\Cref{fig:fetch_sparsity} summarizes the experiment result on varying sparsity of the reward. We can see that as the reward becomes sparser, the performance gap between \sprl{} and the baselines becomes larger, which is consistent with our theory.

\begin{figure}[!h]
    \centering
    \vspace*{-5pt}
    \hspace*{-5pt}
    \includegraphics[draft=false, height=6.7\baselineskip, valign=t]{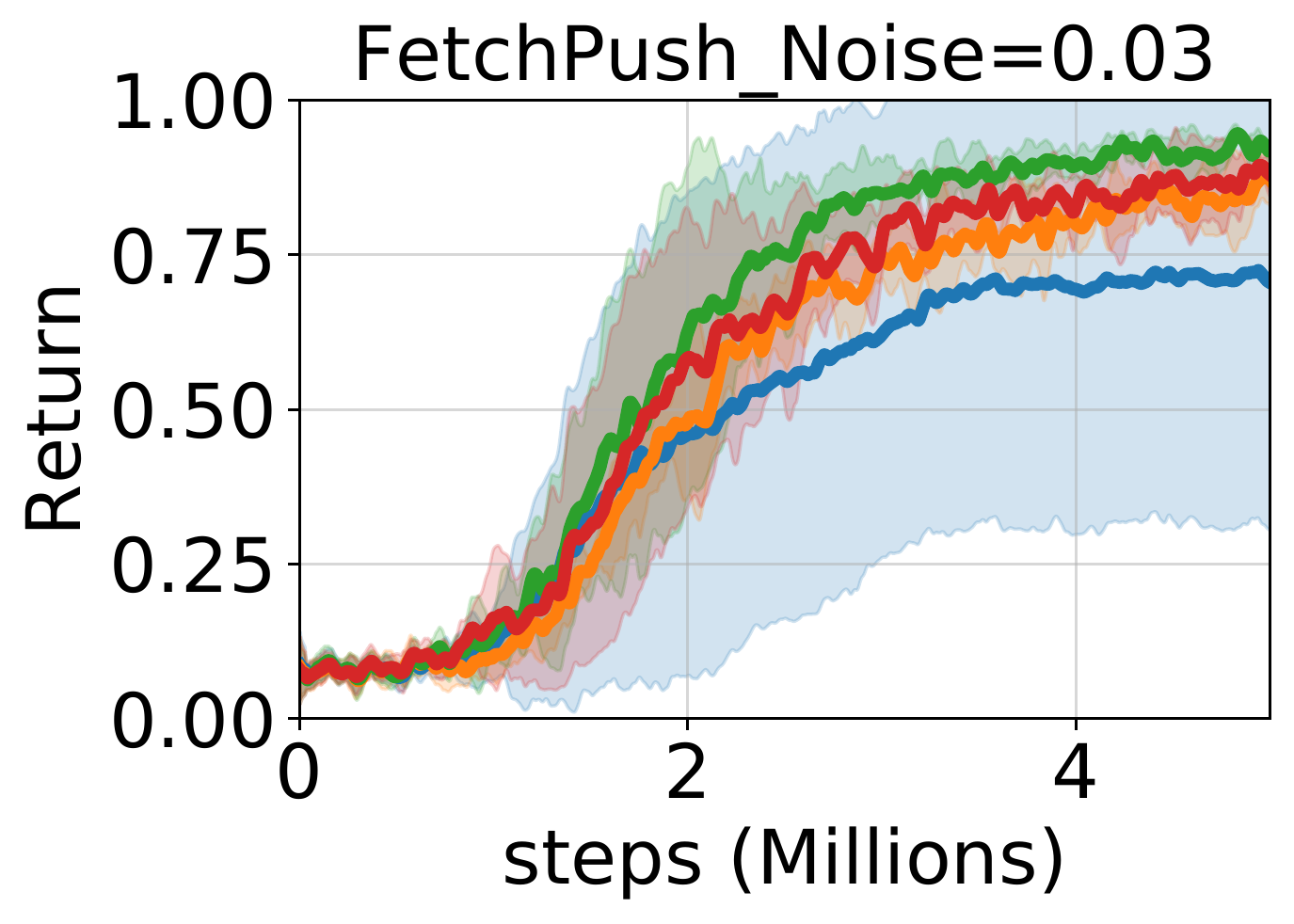}
    \hspace{-6pt}
    \includegraphics[draft=false, height=6.7\baselineskip, valign=t]{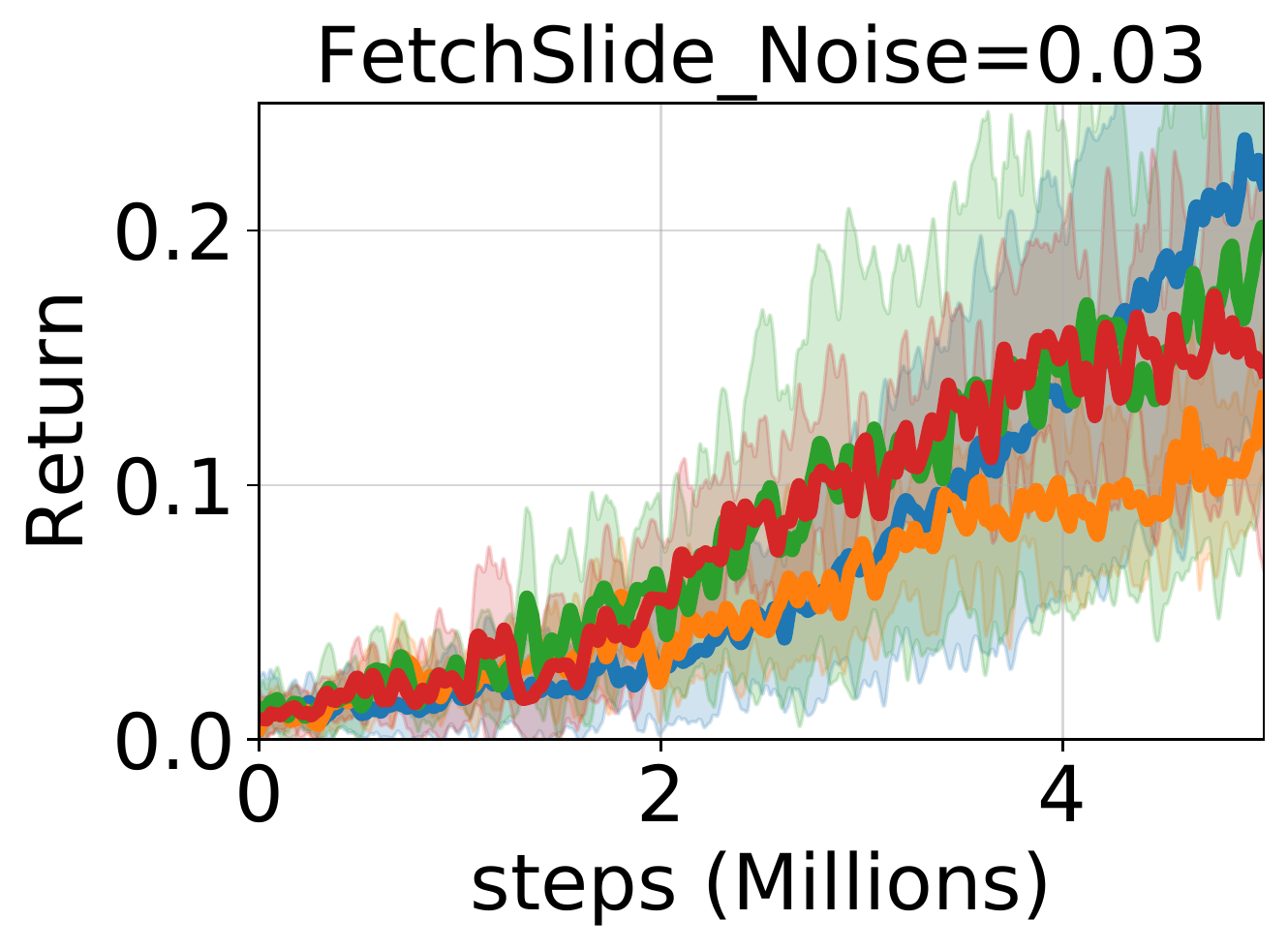}
    \hspace{-6pt}
    \includegraphics[draft=false, height=6.7\baselineskip, valign=t]{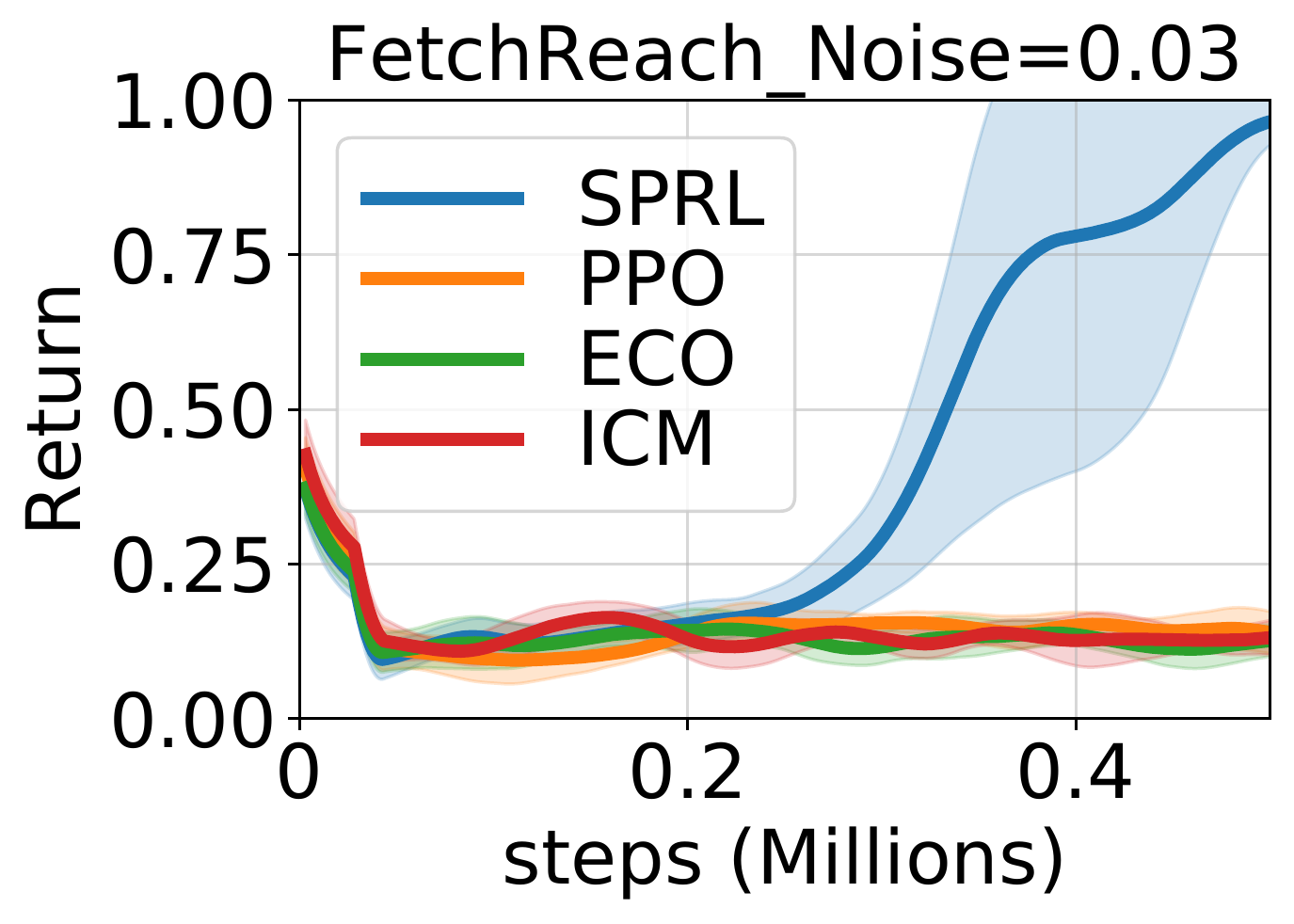}
    \hspace{-6pt}
    \includegraphics[draft=false, height=6.7\baselineskip, valign=t]{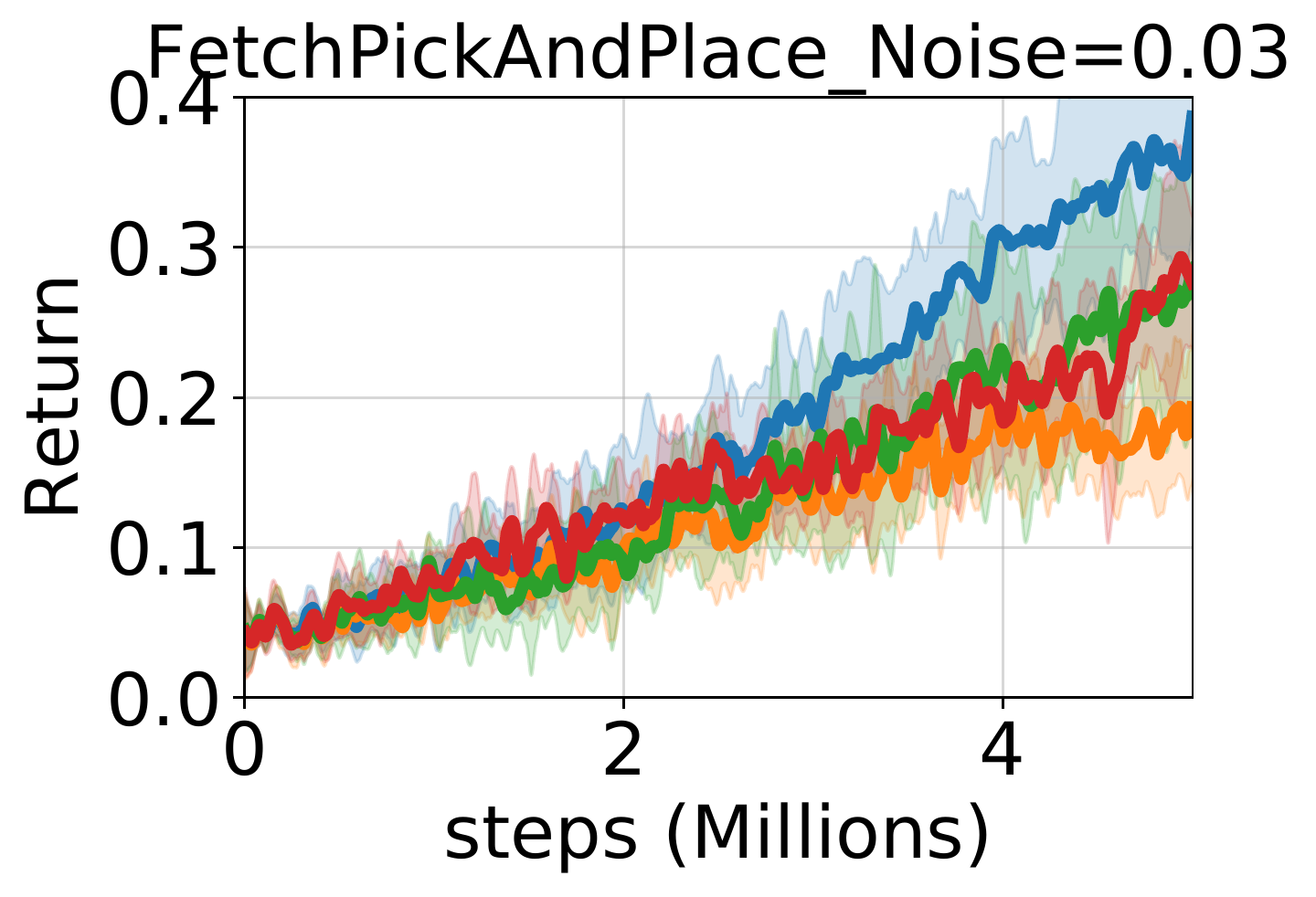}
    \hspace{-5pt}
    \\
    \hspace*{-5pt}
    \includegraphics[draft=false, height=6.7\baselineskip, valign=t]{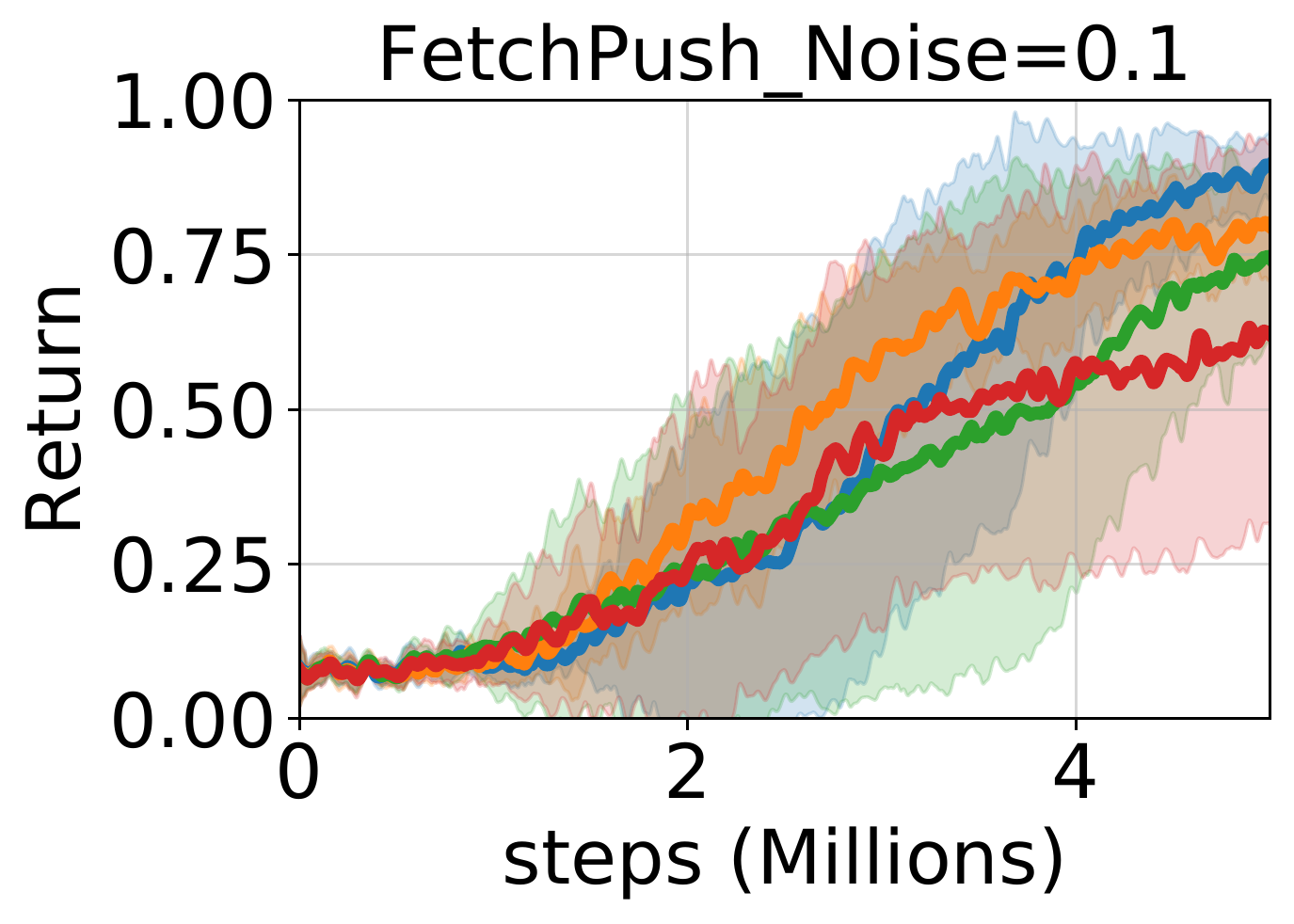}
    \hspace{-6pt}
    \includegraphics[draft=false, height=6.7\baselineskip, valign=t]{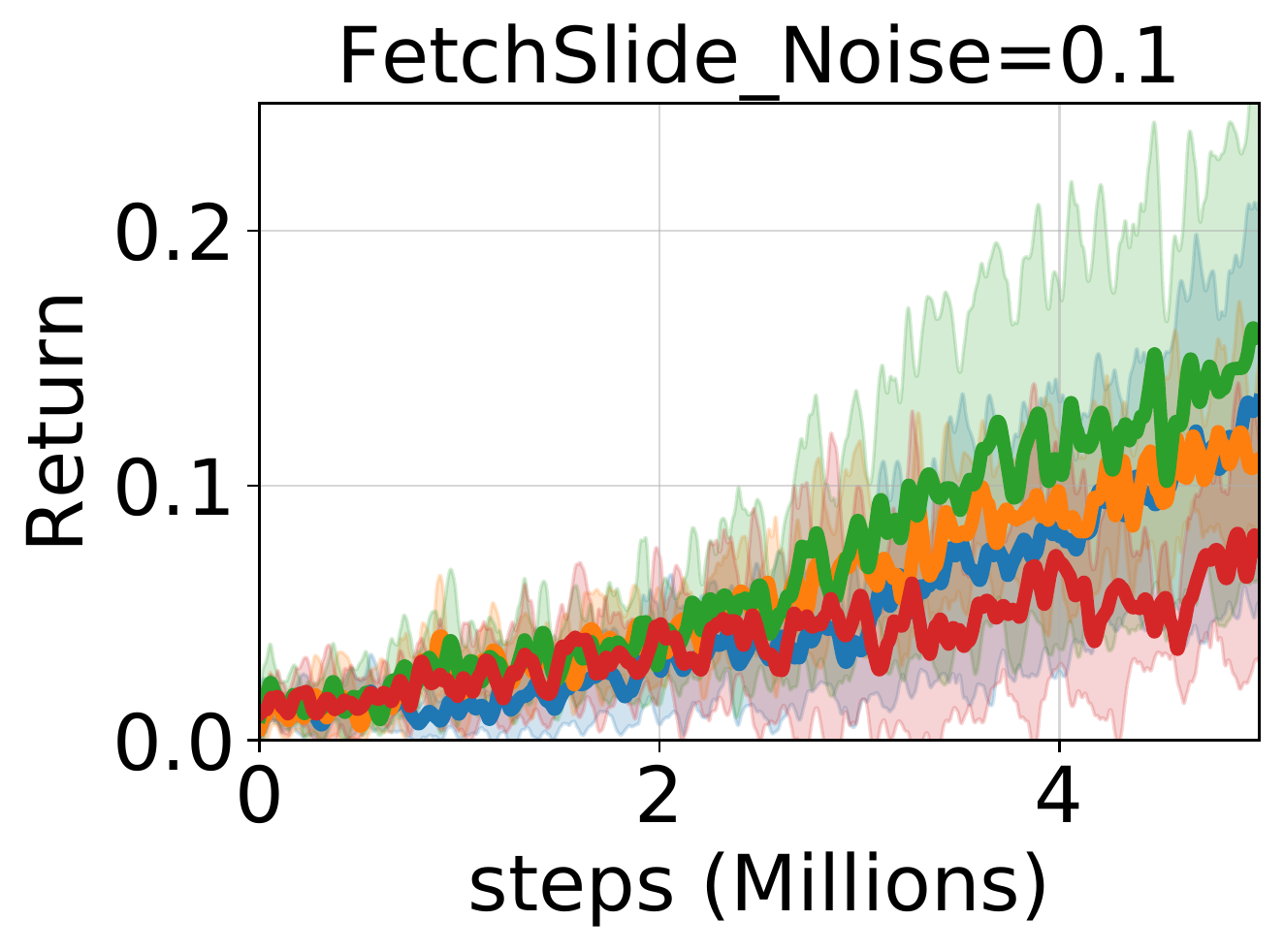}
    \hspace{-6pt}
    \includegraphics[draft=false, height=6.7\baselineskip, valign=t]{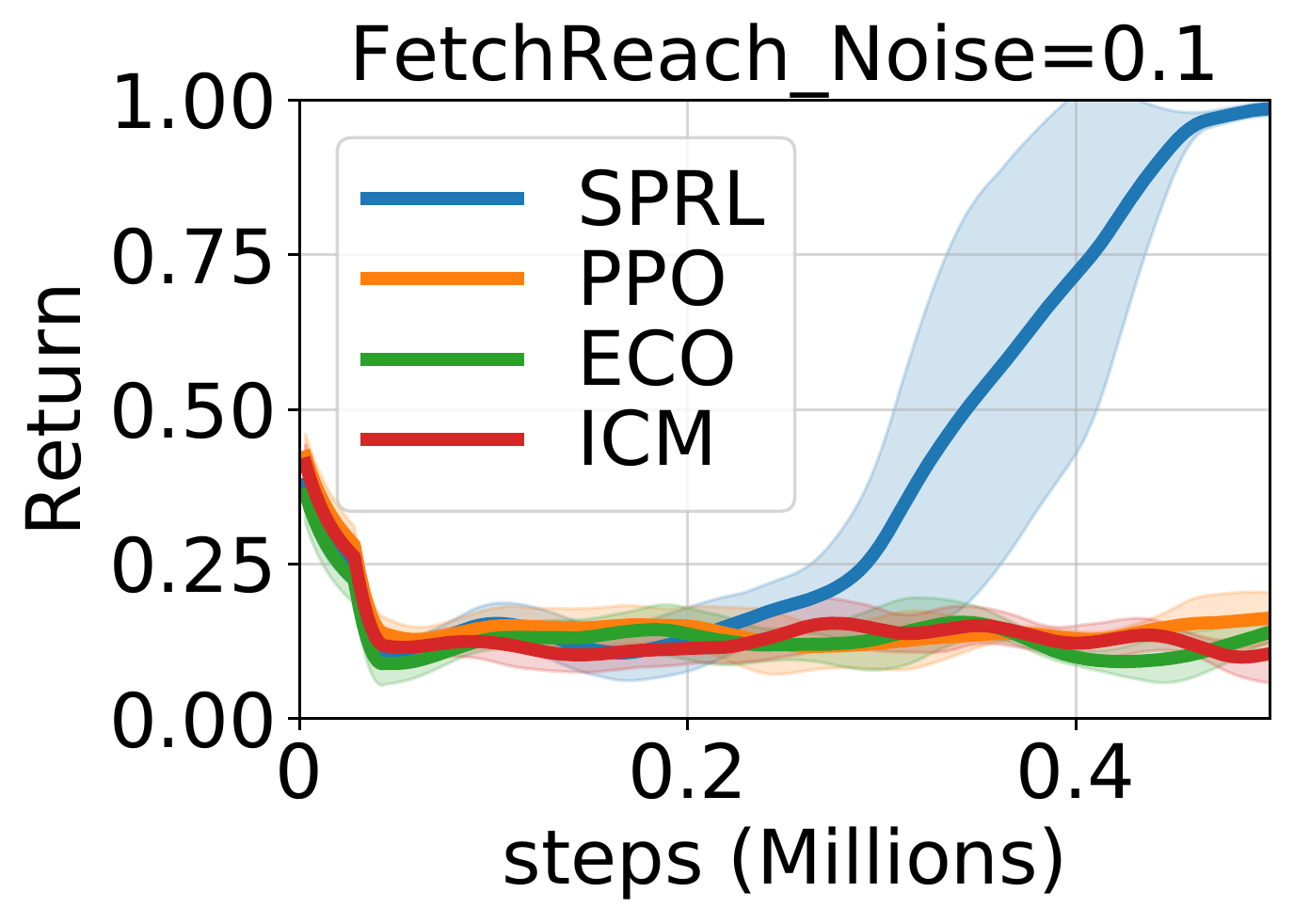}
    \hspace{-6pt}
    \includegraphics[draft=false, height=6.7\baselineskip, valign=t]{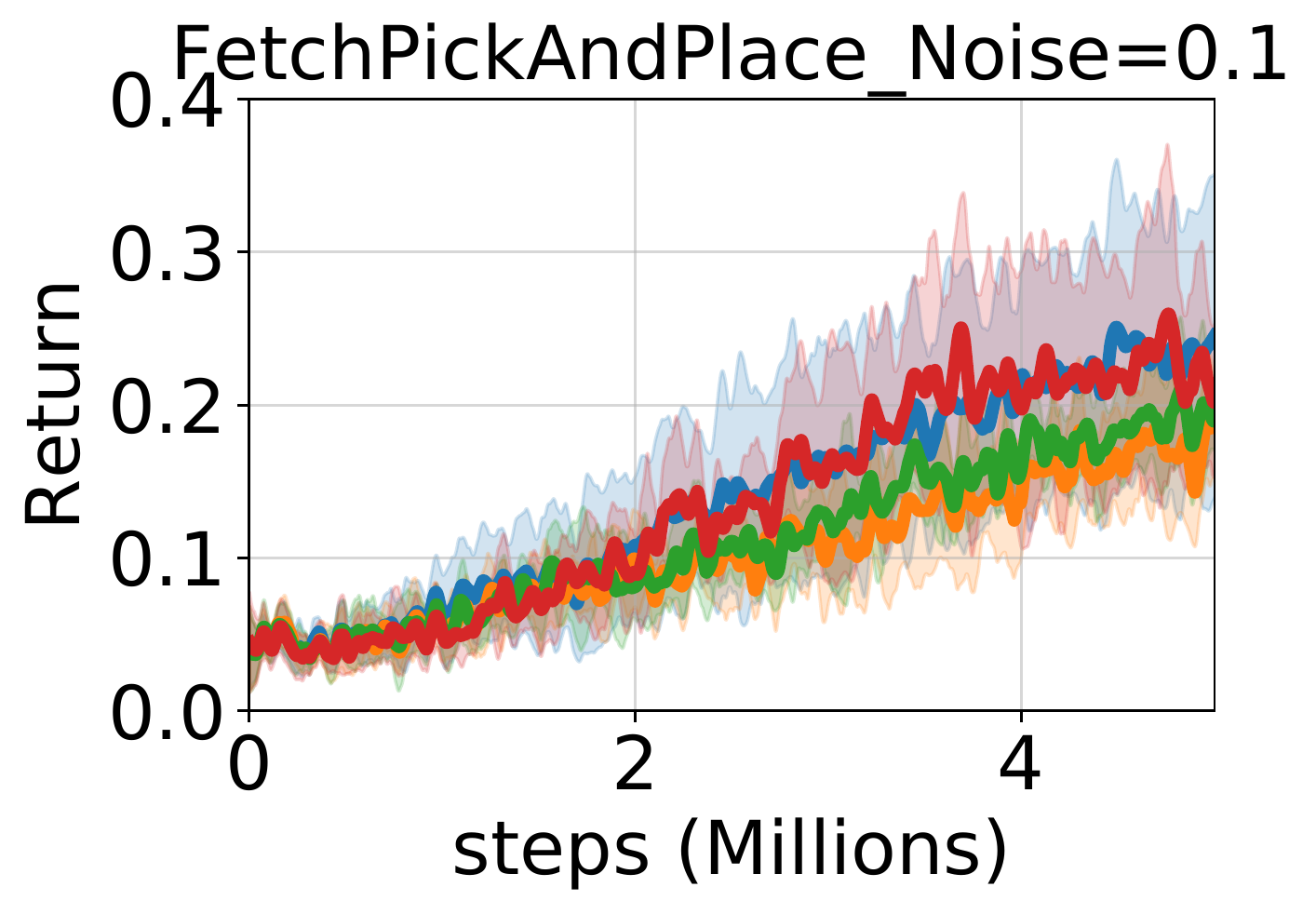}
    \hspace{-5pt}
    \\
    \hspace*{-5pt}
    \includegraphics[draft=false, height=6.7\baselineskip, valign=t]{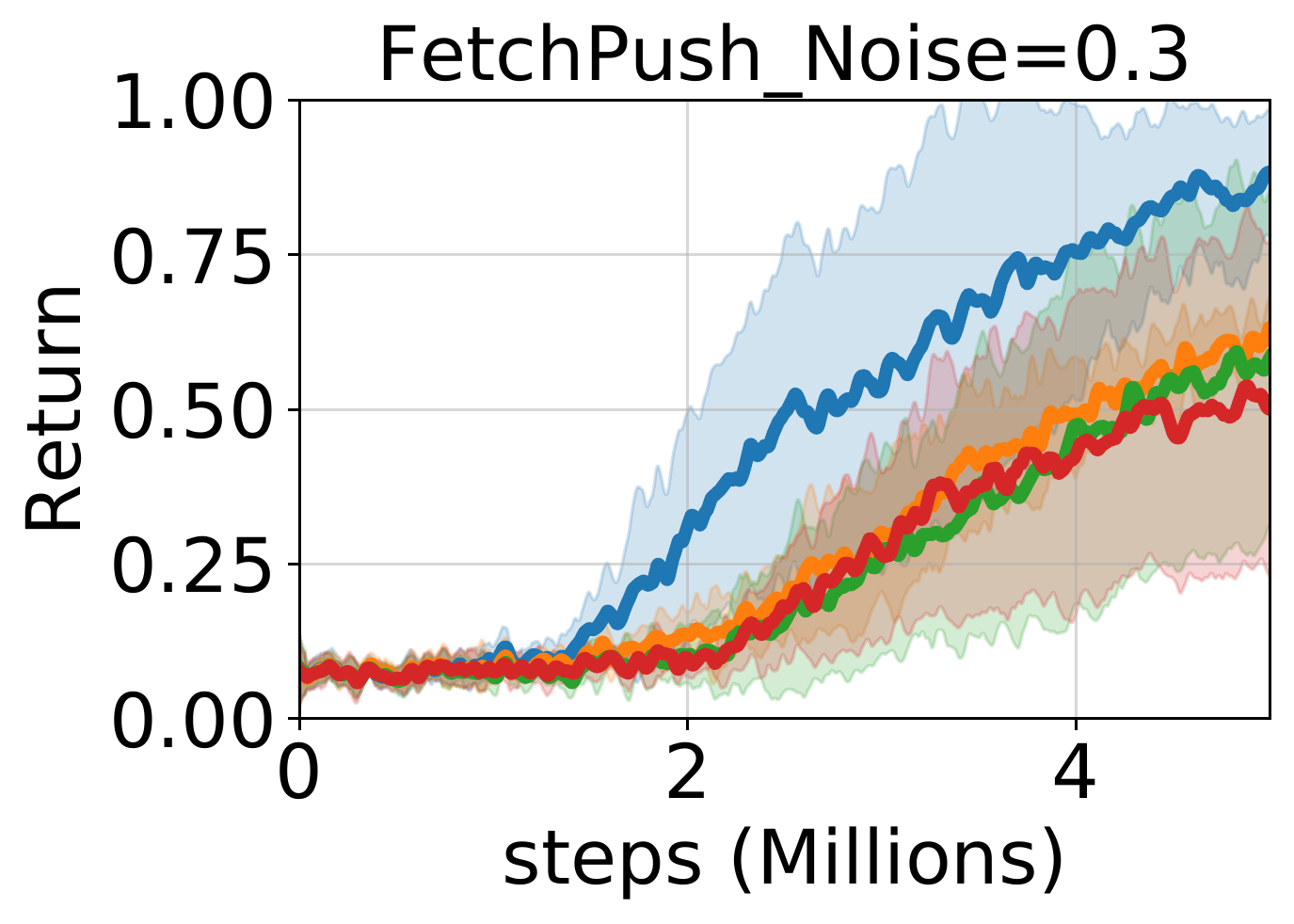}
    \hspace{-6pt}
    \includegraphics[draft=false, height=6.7\baselineskip, valign=t]{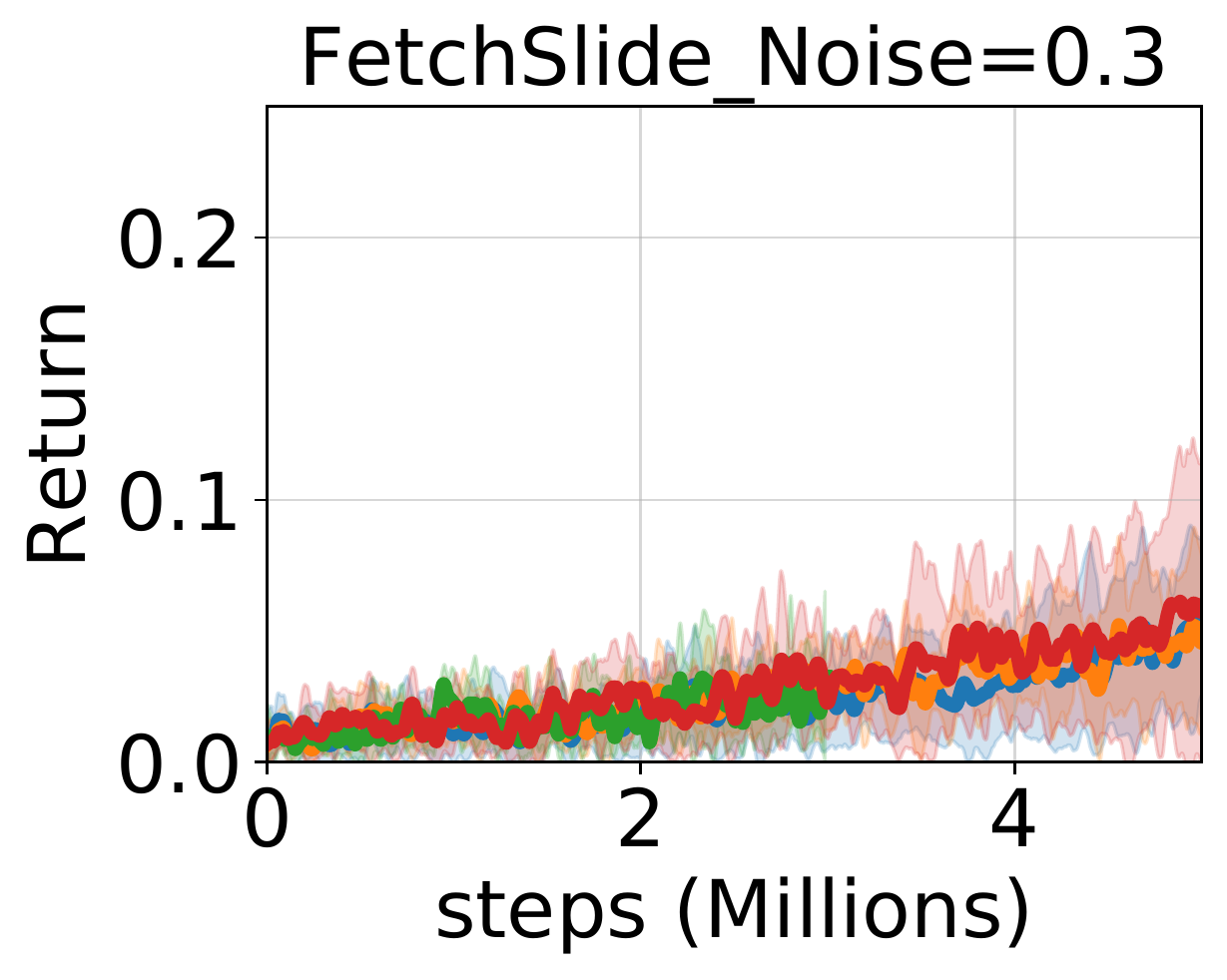}
    \hspace{-6pt}
    \includegraphics[draft=false, height=6.7\baselineskip, valign=t]{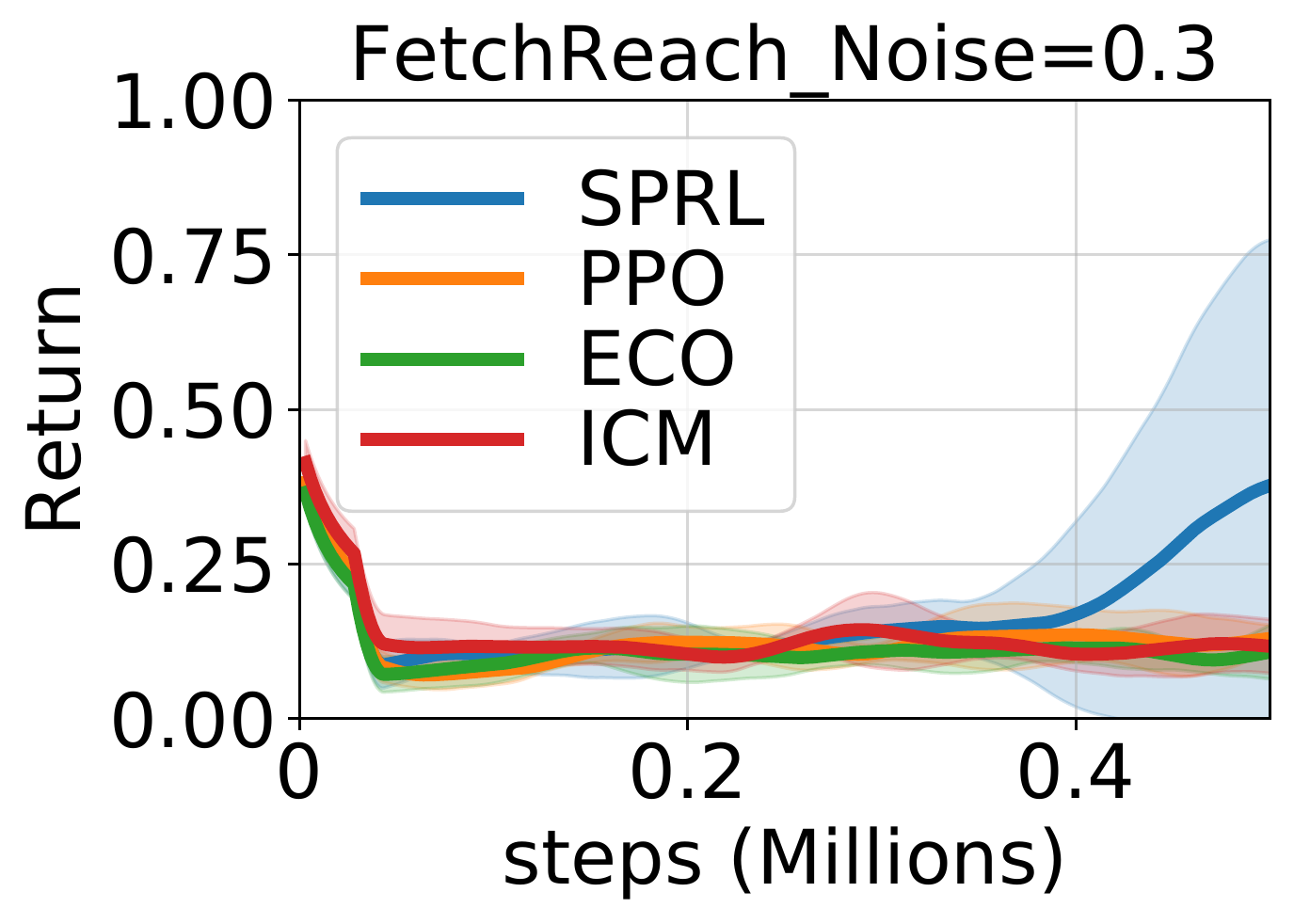}
    \hspace{-6pt}
    \includegraphics[draft=false, height=6.7\baselineskip, valign=t]{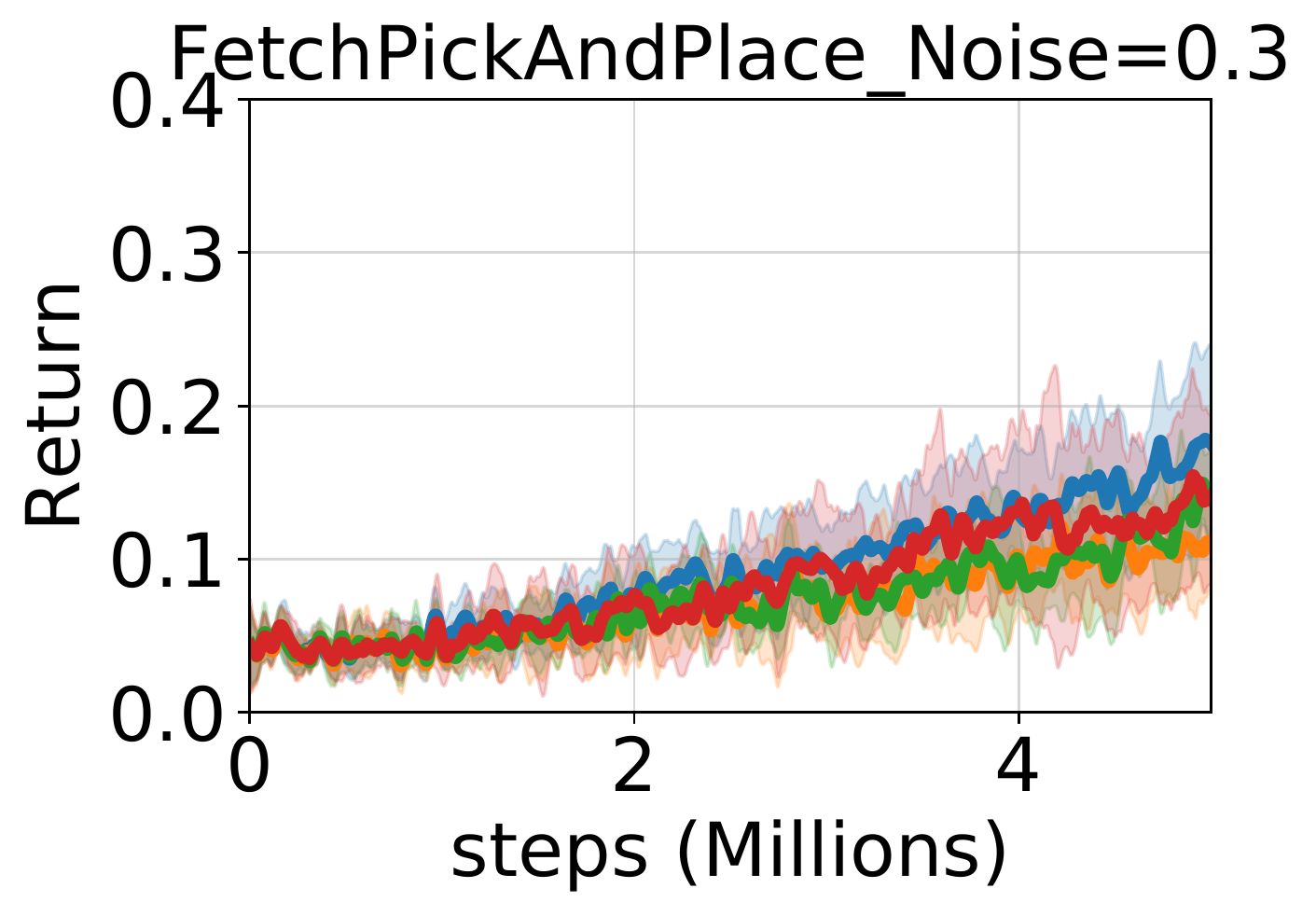}
    \hspace{-5pt}
    \vspace{-5pt}
    \caption{
        Average episode reward of \sprl{} with varying stochasticity using action noise of $\mathcal{N}(0, \sigma^2)$ where $\sigma=0.03, 0.1, 0.3$ for \fetch{} tasks. Other hyper-parameters are kept same as the best hyper-parameter. We observe that the performance is preserved the most for \sprl{} when stochasticity changes.
    }
    \label{fig:fetch_stochasticity}
\end{figure}

\cutparagraphup
\subsection{Effect of stochasticity of the environment}\label{appendix:stochastic}

\Cref{fig:fetch_stochasticity} summarizes the experiment result on varying stochasticity with a sparser version of \fetch{} tasks. We used an additive action noise following $\mathcal{N}(0, \sigma^2)$ with $\sigma=0.03, 0.1, 0.3$ to vary the stochasticity of the environment. Also, we used substep $\substep=3,8,10,10$ for \textit{Reach}, \textit{Push}, \textit{Slide}, \textit{Pick} respectively to make the reward sparser. This change was made to clearly observe the performance difference when stochasticity changes. The performance of \sprl{} does not degrade much as the stochasticity increases compared to the other baselines indicating that SPRL robustly performs well in stochastic environments.

%% file: 8_appendix_rnet.tex
\clearpage
\cutsectionup
\section{Analysis on the reachability network (RNet)}
\label{appendix:rnet}

\begin{figure}[!h]
    \centering
    \includegraphics[draft=false, height=6.5\baselineskip, valign=b]{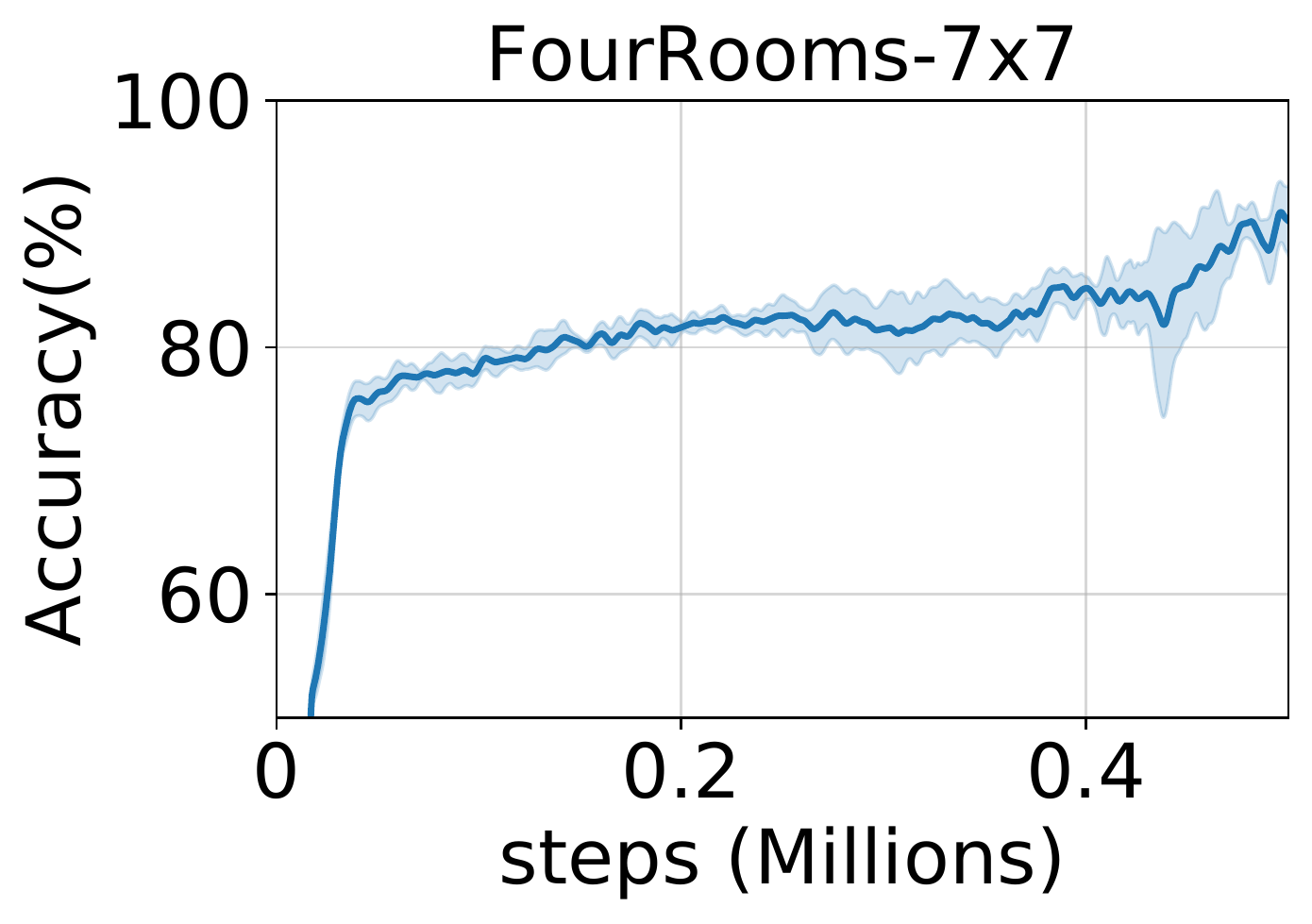}
    \includegraphics[draft=false, height=6.5\baselineskip, valign=b]{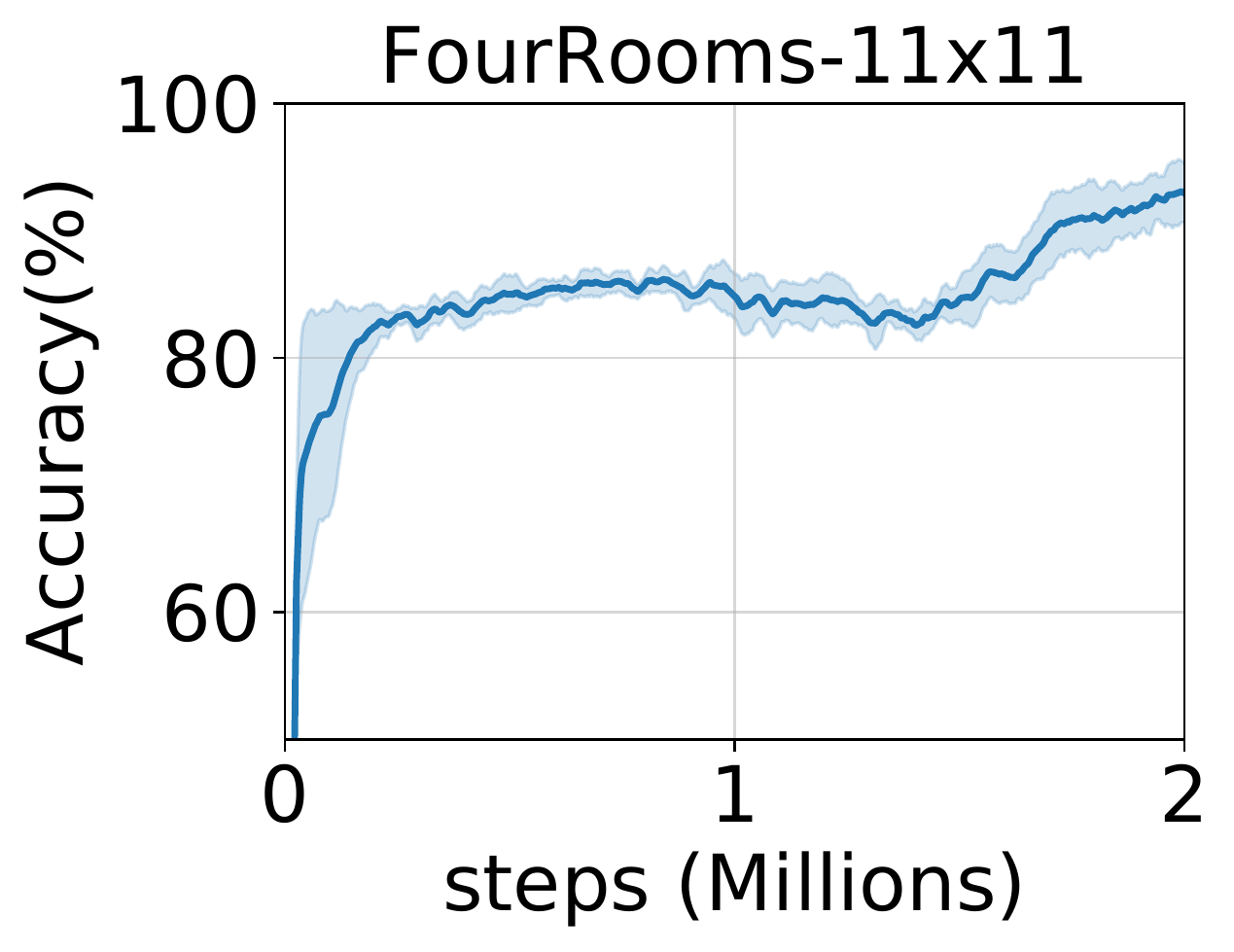}
    \includegraphics[draft=false, height=6.5\baselineskip, valign=b]{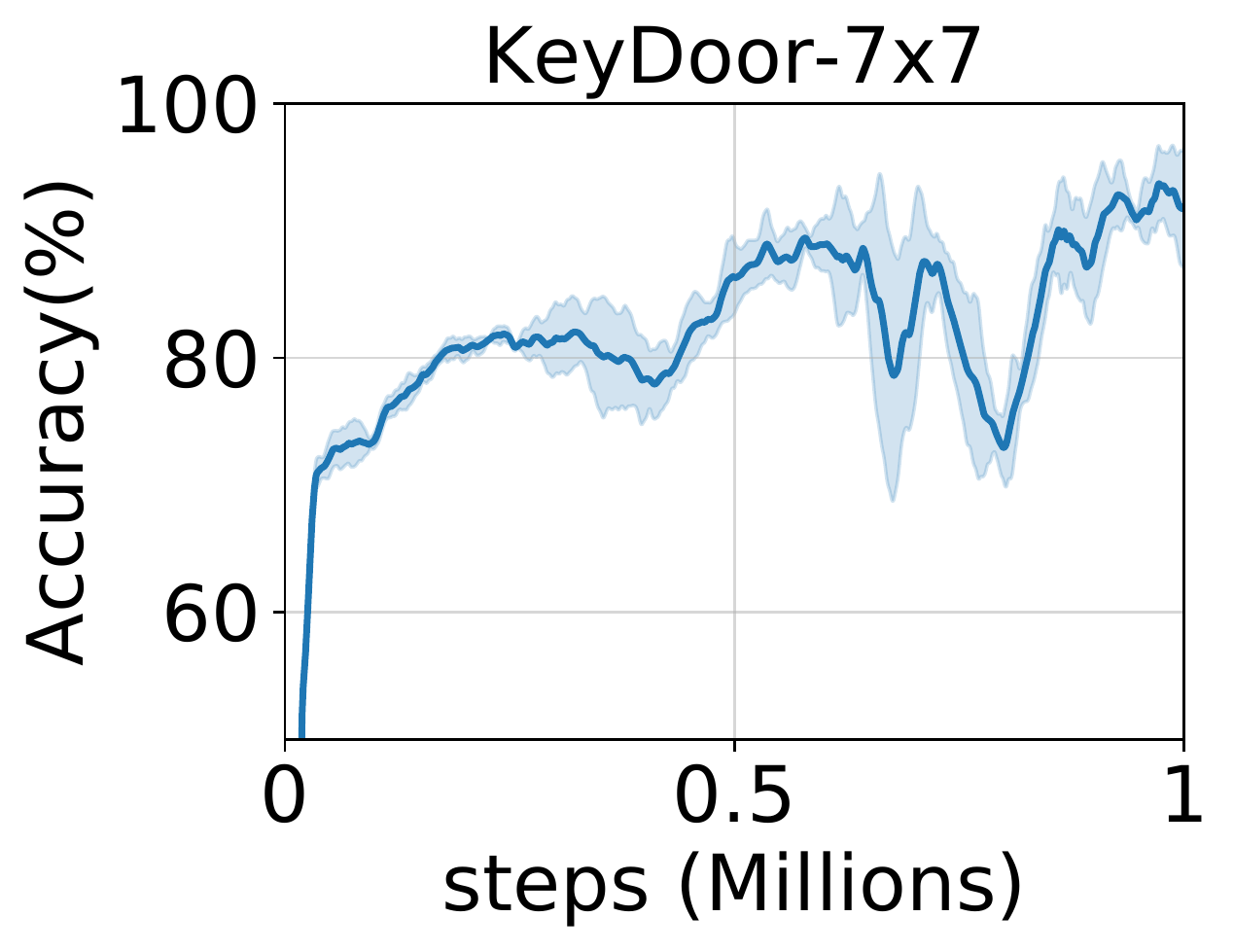}
    \includegraphics[draft=false, height=6.5\baselineskip, valign=b]{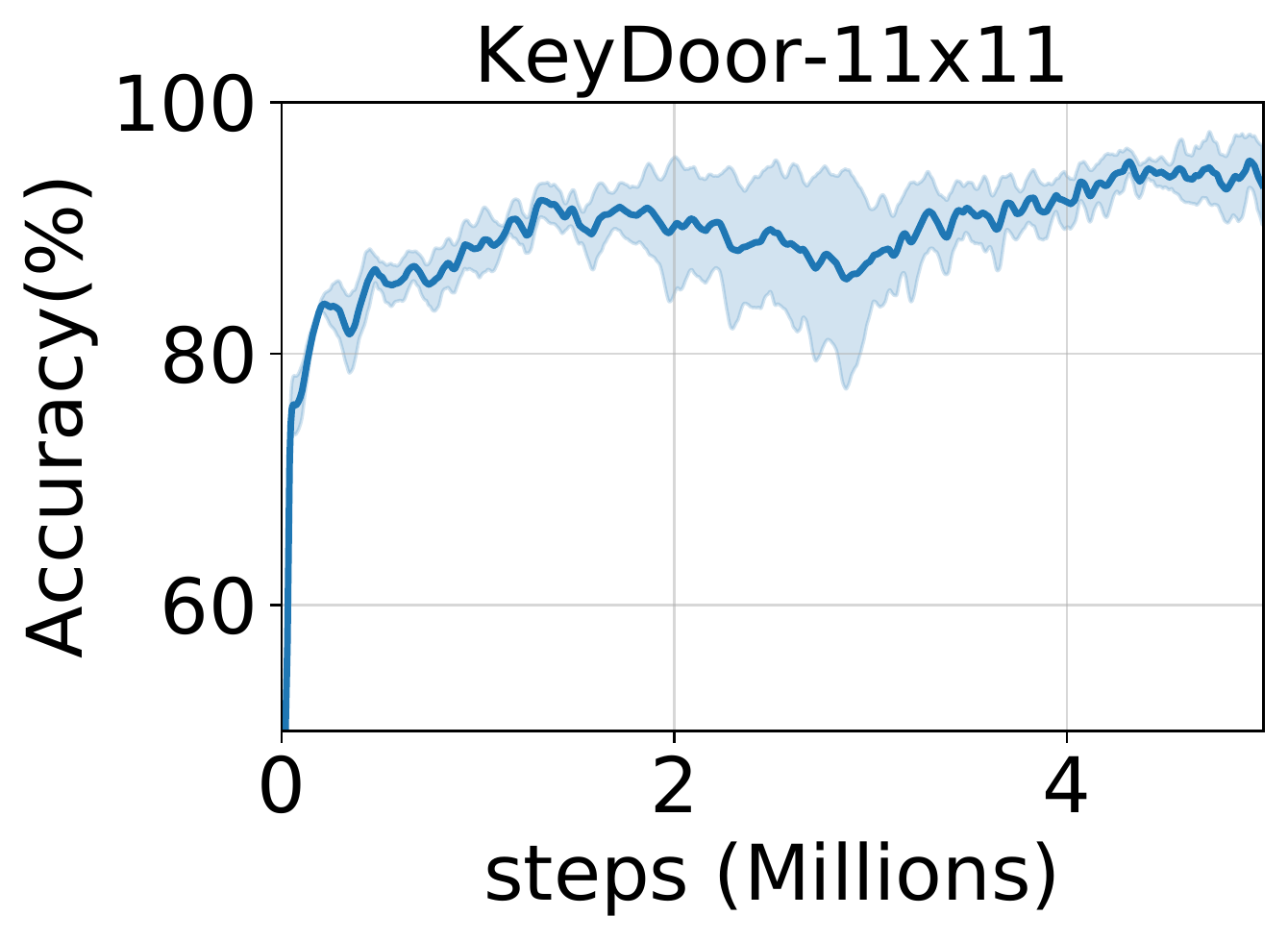}
    {
    \parbox{0.20\linewidth}{\hspace{+50pt} \small\centering (a)}
    \parbox{0.25\linewidth}{\hspace{+50pt} \small\centering (b)}
    \parbox{0.20\linewidth}{\hspace{+25pt} \small\centering (c)}
    \parbox{0.25\linewidth}{\hspace{+30pt} \small\centering (d)}
    }
    \caption{
        The accuracy of the learned reachability network on (a) \fours{} (b) \fourl{}, (c) \keys{} and (d) \keyl{} in \grid{} in terms of environment steps.
    }
    \label{fig:grid_rnet}
\end{figure}
\begin{figure}[!h]
    \centering
    \includegraphics[draft=false, width=0.30\linewidth, valign=t]
    {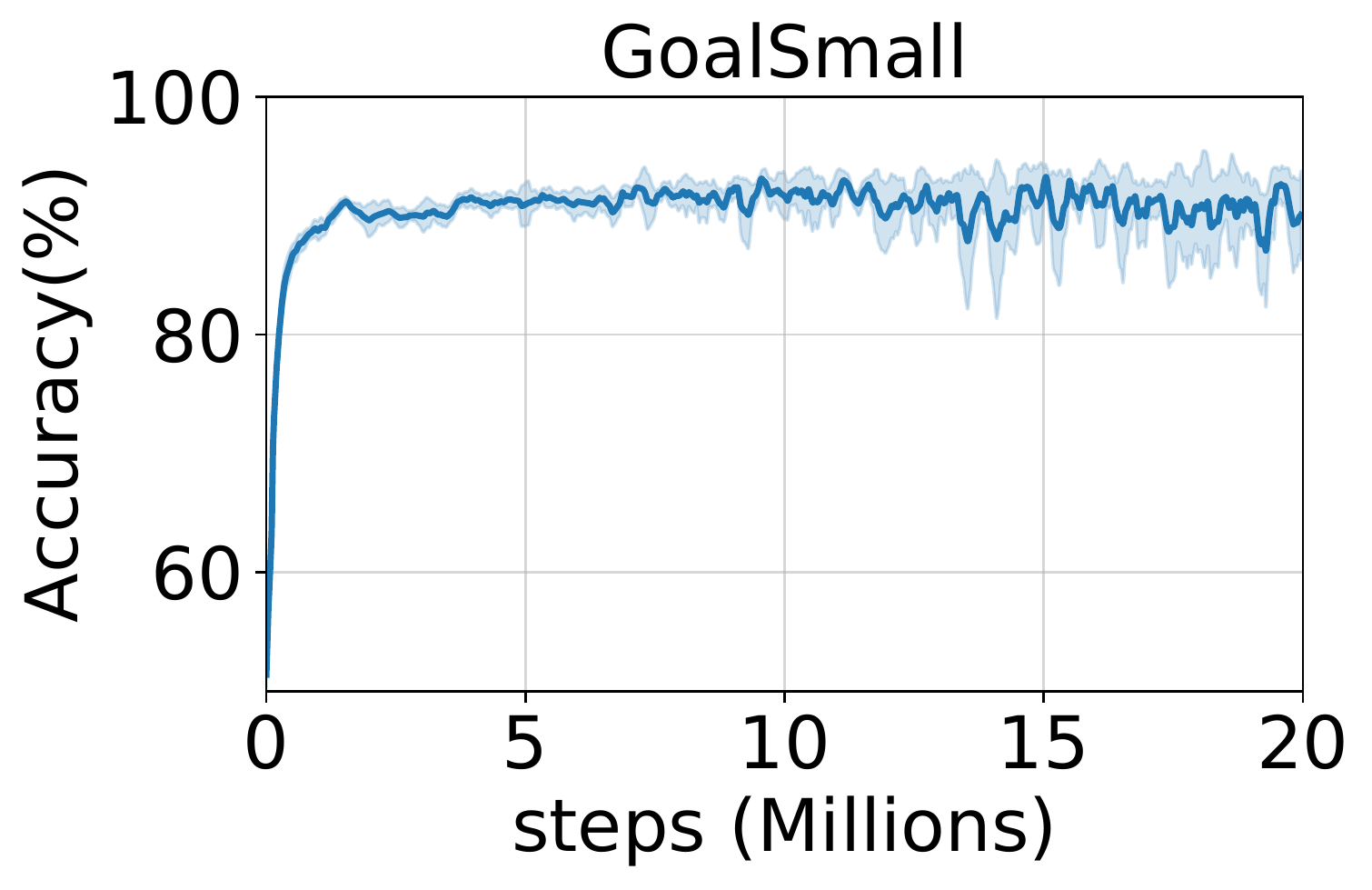}
    \includegraphics[draft=false, width=0.30\linewidth, valign=t]
    {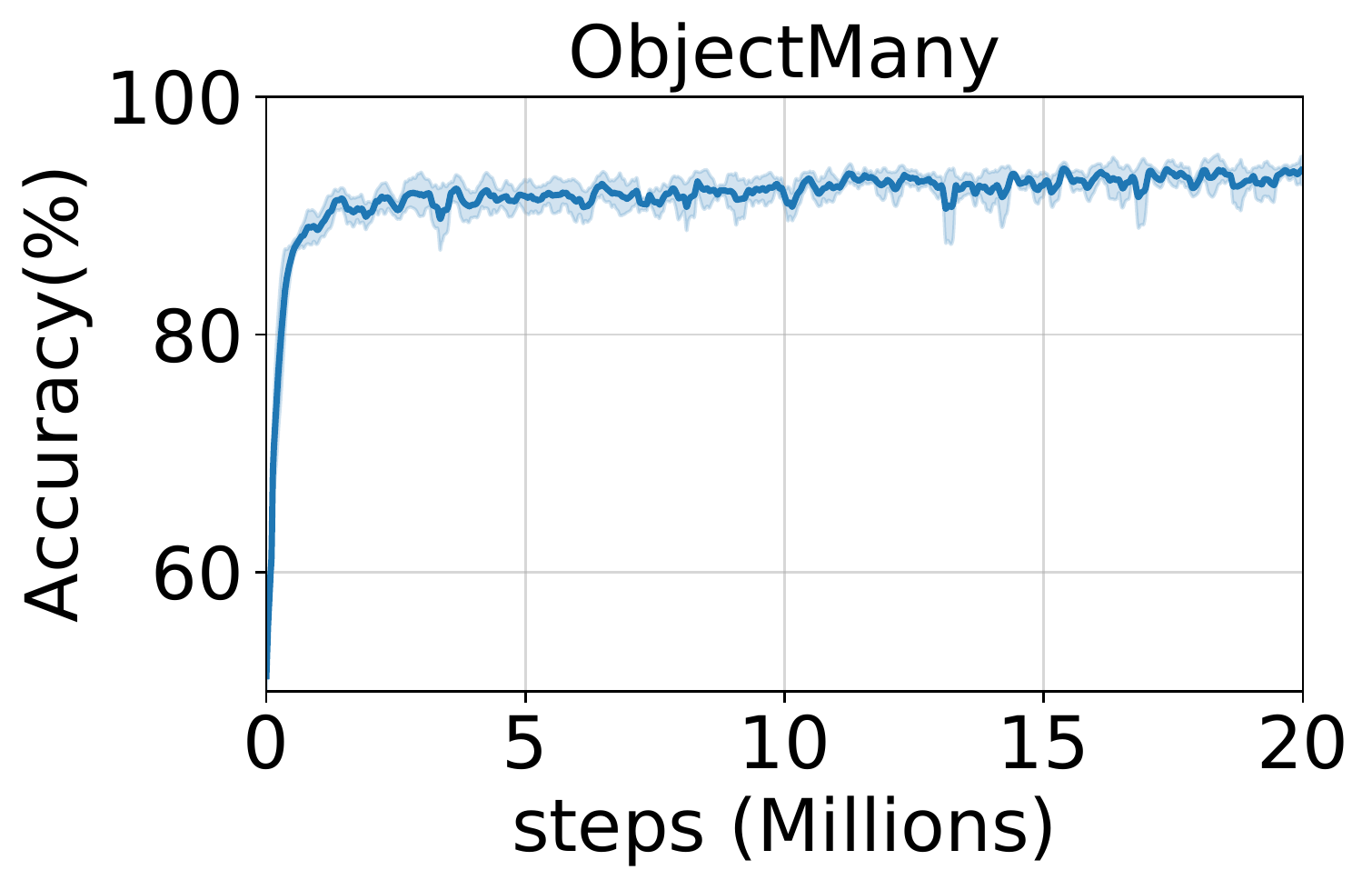}
    \includegraphics[draft=false, width=0.30\linewidth, valign=t]
    {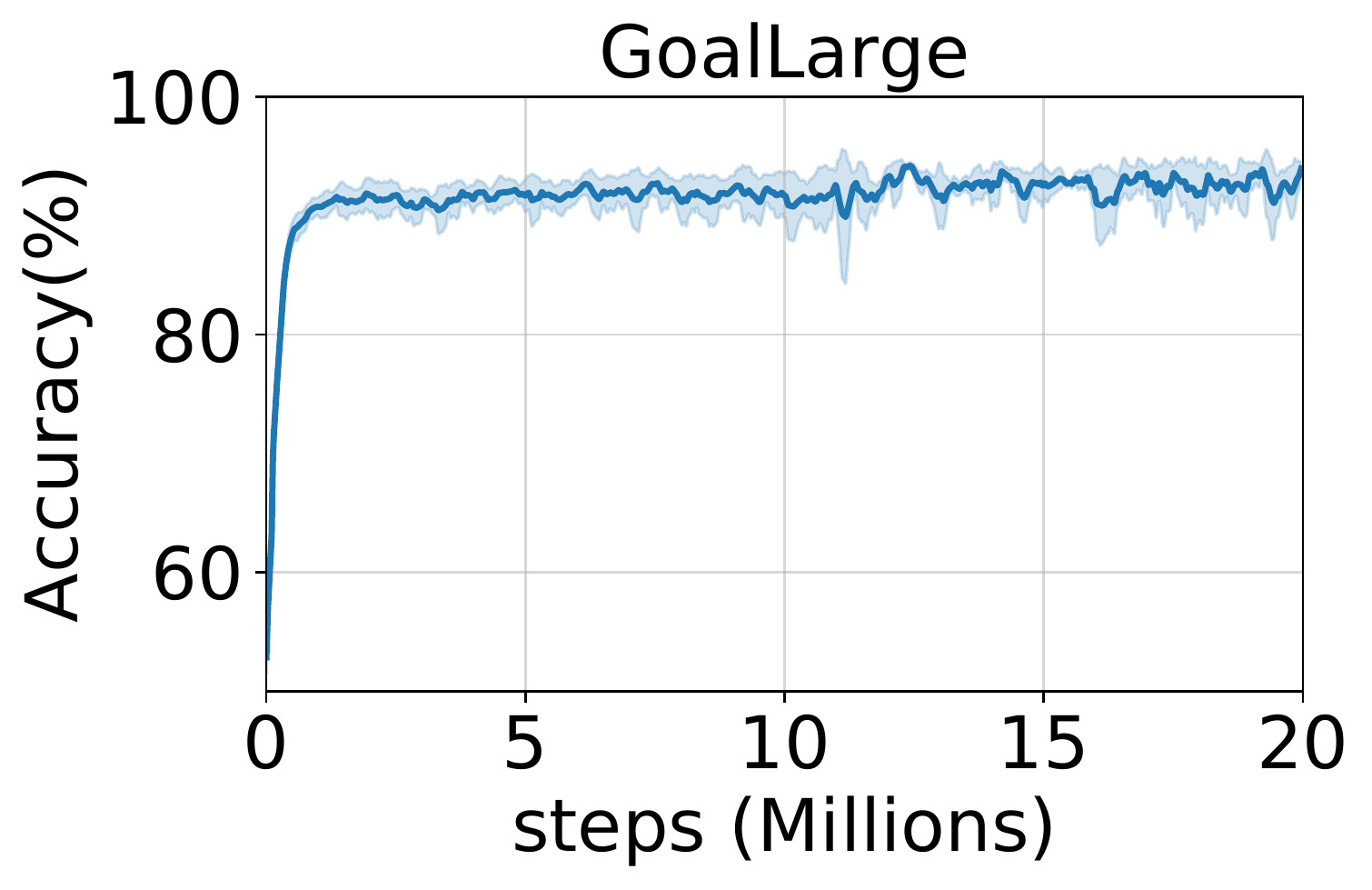}
    {
    \parbox{0.30\linewidth}{\hspace{+20pt} \small\centering (a)}
    \parbox{0.30\linewidth}{\hspace{+20pt}
    \small\centering (b)}
    \parbox{0.30\linewidth}{\hspace{+20pt}
    \small\centering (c)}
    }
    \vspace*{-9pt}
    \caption{
    The accuracy of the learned reachability network on (a) \dmgs{} (b) \dmom{}, and (c) \dmgl{} in \dmlab{} in terms of environment steps.
    }
    \label{fig:dmlab_rnet}
\end{figure}
\begin{figure}[!h]
    \centering
    \includegraphics[draft=false, width=0.25\linewidth, valign=t]    {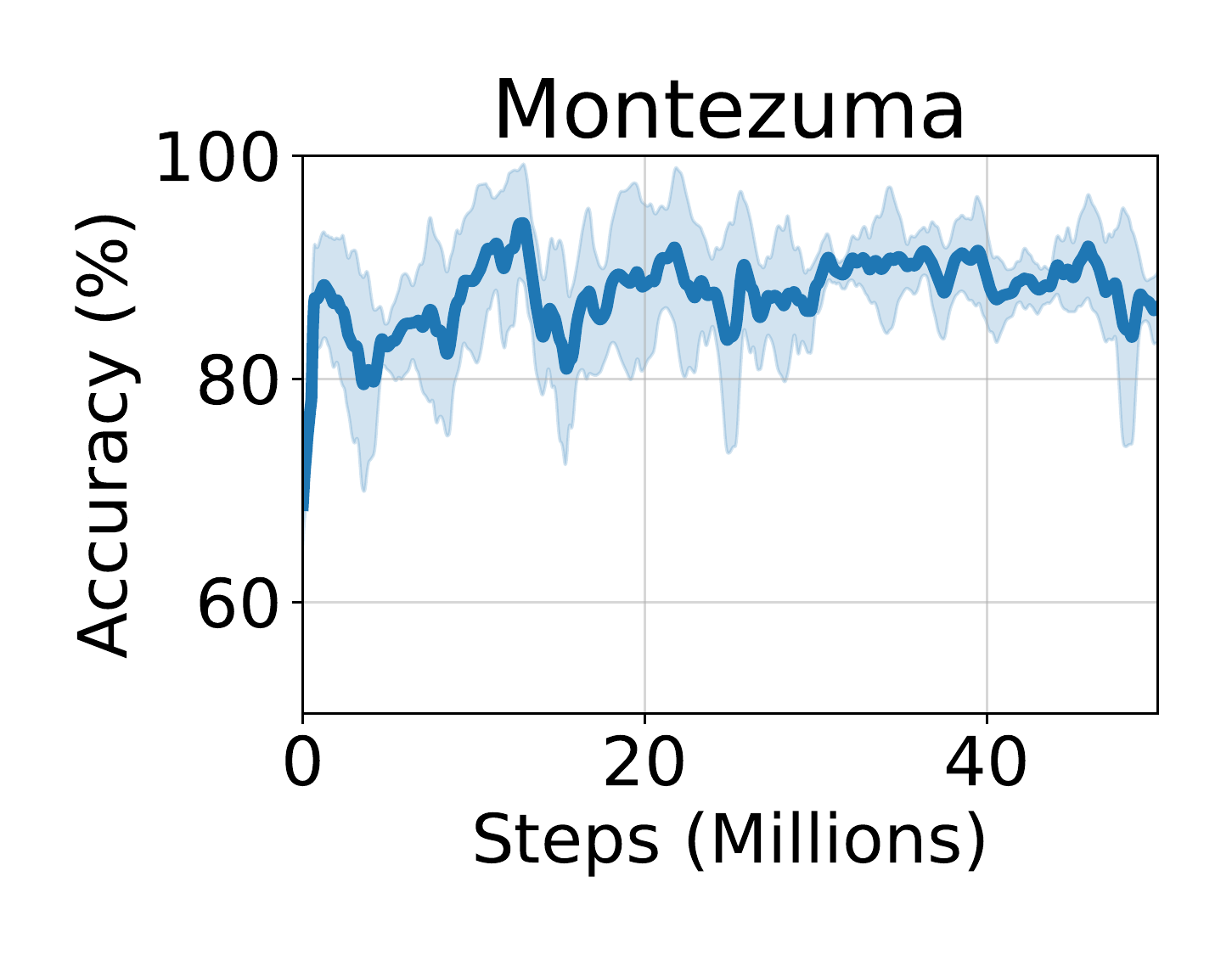}
    \hspace{-2pt}
    \includegraphics[draft=false, width=0.25\linewidth, valign=t]
    {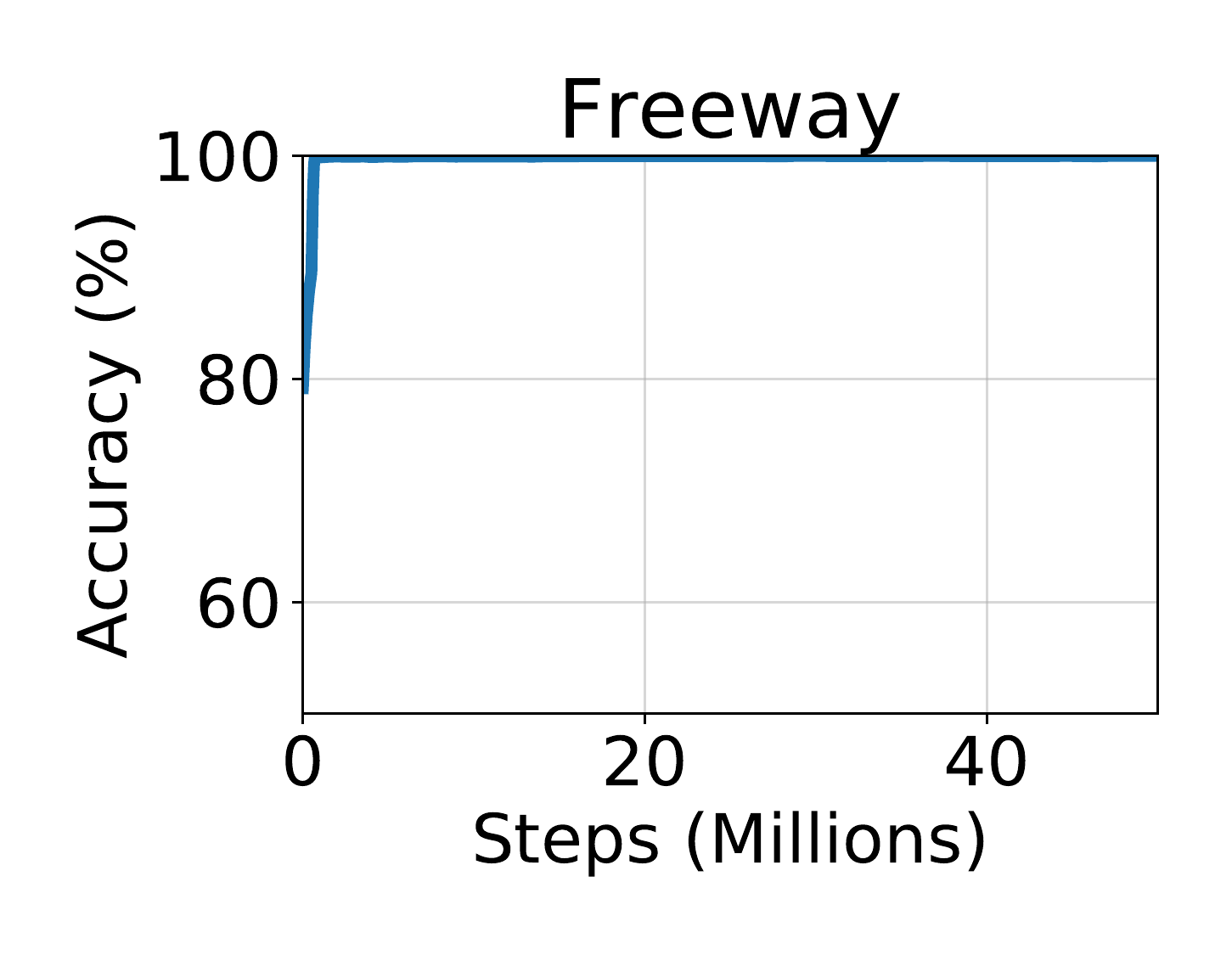}
    \hspace{-2pt}
    \includegraphics[draft=false, width=0.25\linewidth, valign=t]
    {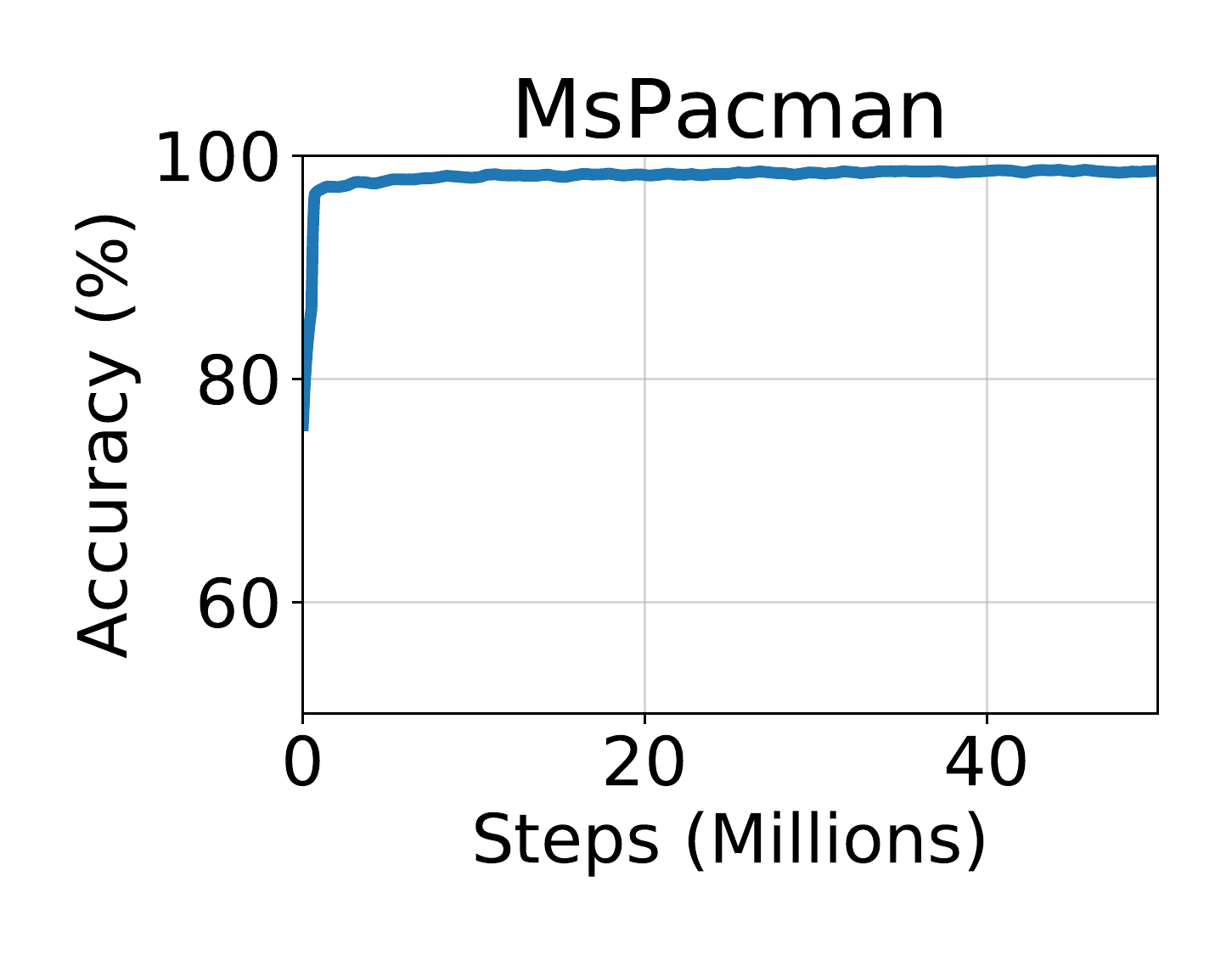}
    \hspace{-2pt}
    {
    \parbox{0.25\linewidth}{\hspace{+20pt} \small\centering (a) \hspace{-2pt}}
    \parbox{0.25\linewidth}{\hspace{+20pt} \small\centering (b) \hspace{-2pt}}
    \parbox{0.25\linewidth}{\hspace{+20pt} \small\centering (c) \hspace{-2pt}}
    }
    \\
    \includegraphics[draft=false, width=0.25\linewidth, valign=t]
    {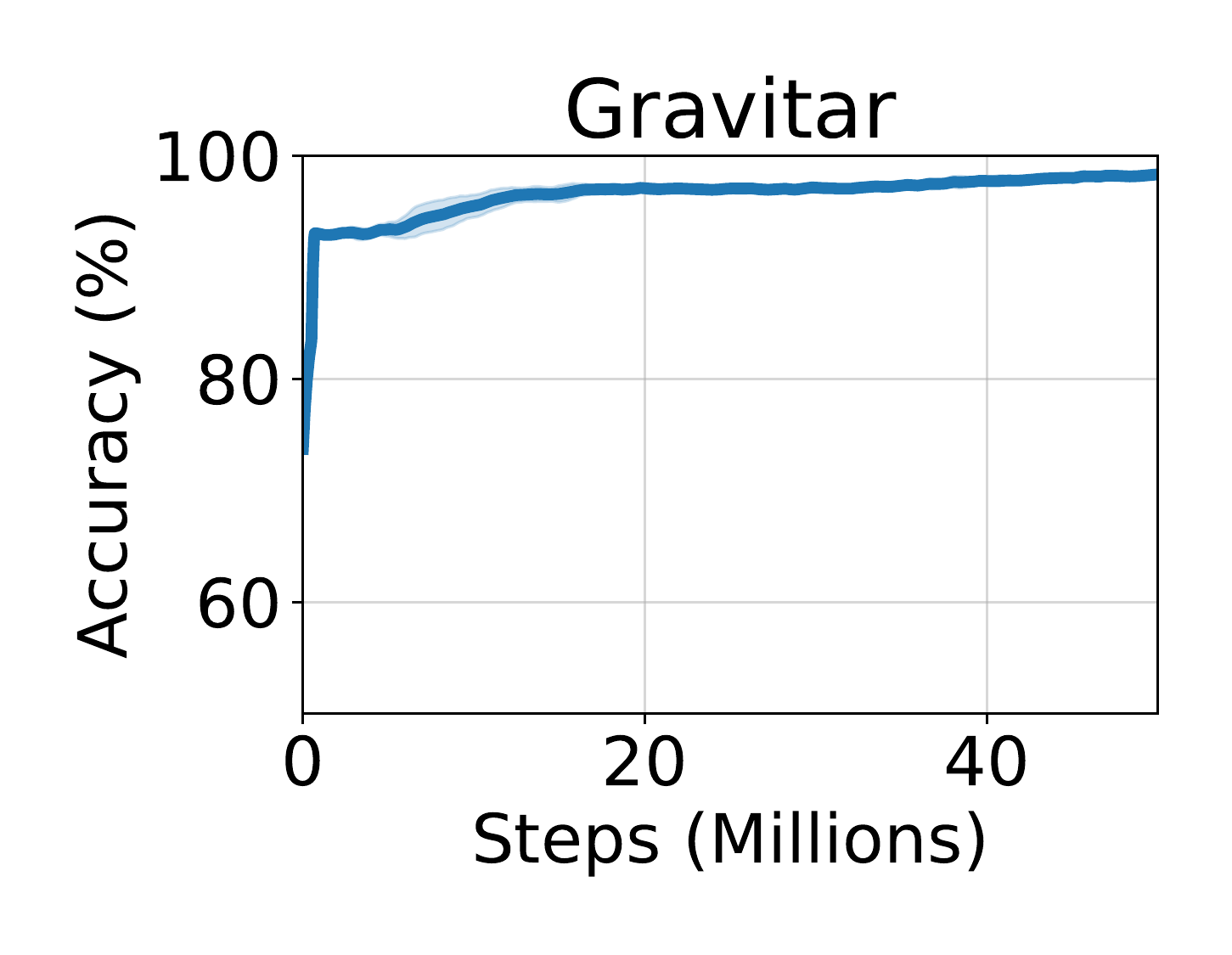}
    \hspace{-2pt}
    \includegraphics[draft=false, width=0.25\linewidth, valign=t]
    {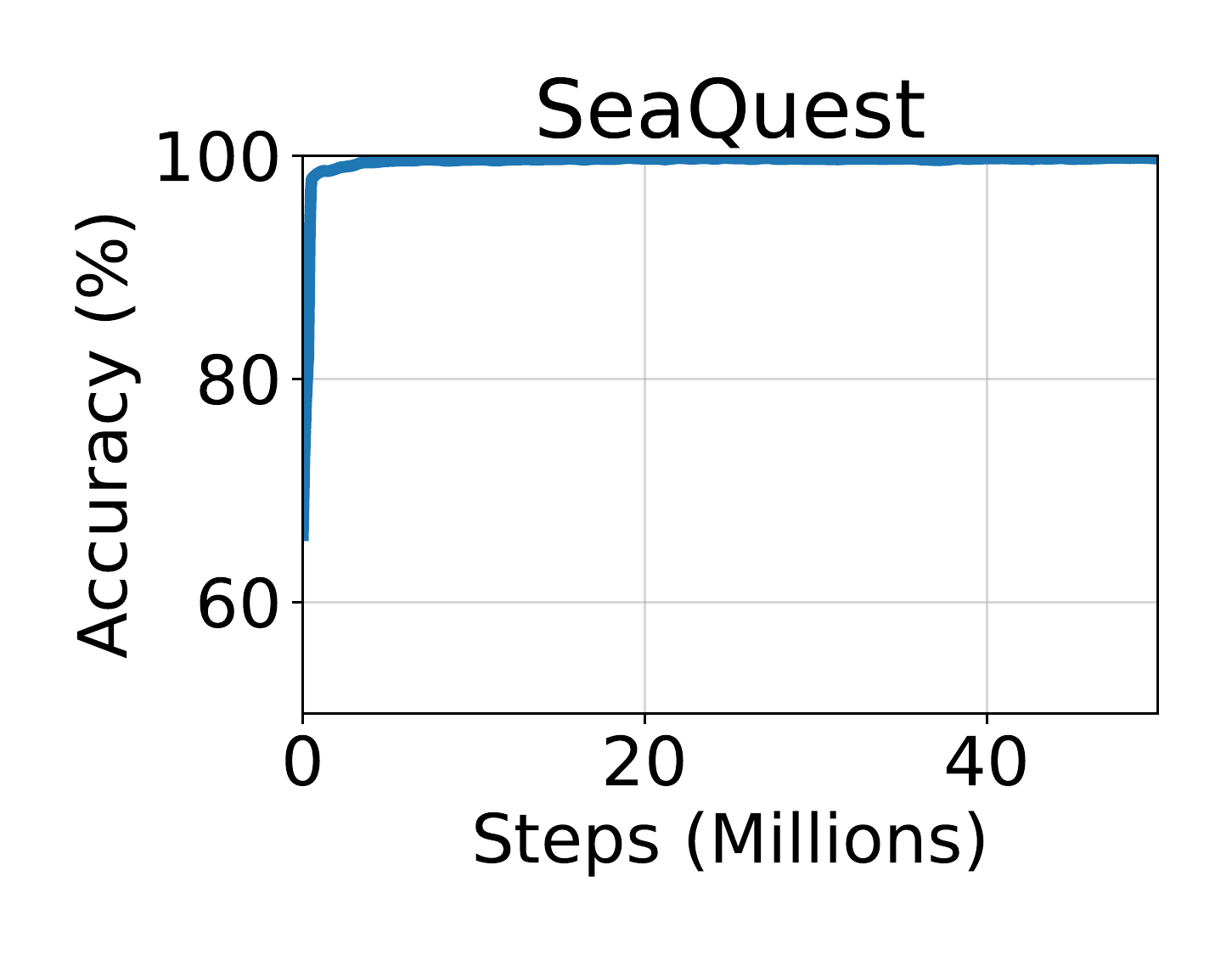}
    \hspace{-2pt}
    \includegraphics[draft=false, width=0.25\linewidth, valign=t]
    {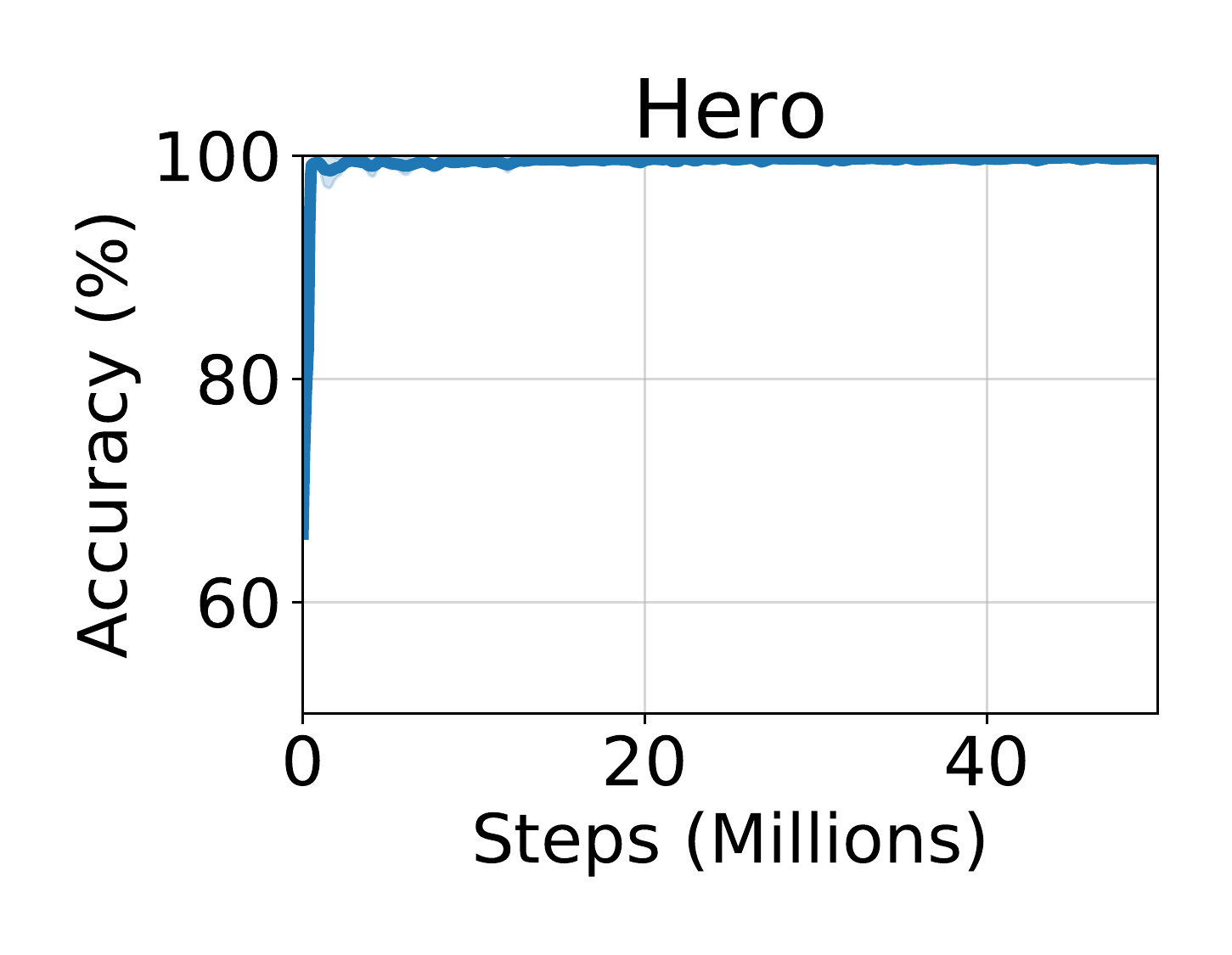}
    \hspace{-2pt}
    {
    \parbox{0.25\linewidth}{\hspace{+30pt} \small\centering (d) \hspace{+4pt}}
    \parbox{0.25\linewidth}{\hspace{+30pt} \small\centering (e) \hspace{+4pt}}
    \parbox{0.25\linewidth}{\hspace{+30pt} \small\centering (f) \hspace{+4pt}}
    }
    \vspace*{-9pt}
    \caption{
    The accuracy of the learned reachability network on (a) \montezuma (b) \freeway, (c) \mspacman, (d) \gravitar, (e) \seaquest, and (f) \hero in \atari in terms of environment steps.
    }
    \label{fig:atari_rnet}
\end{figure}
\begin{figure}[!h]
    \centering
    \includegraphics[draft=false, height=6.5\baselineskip, valign=b]{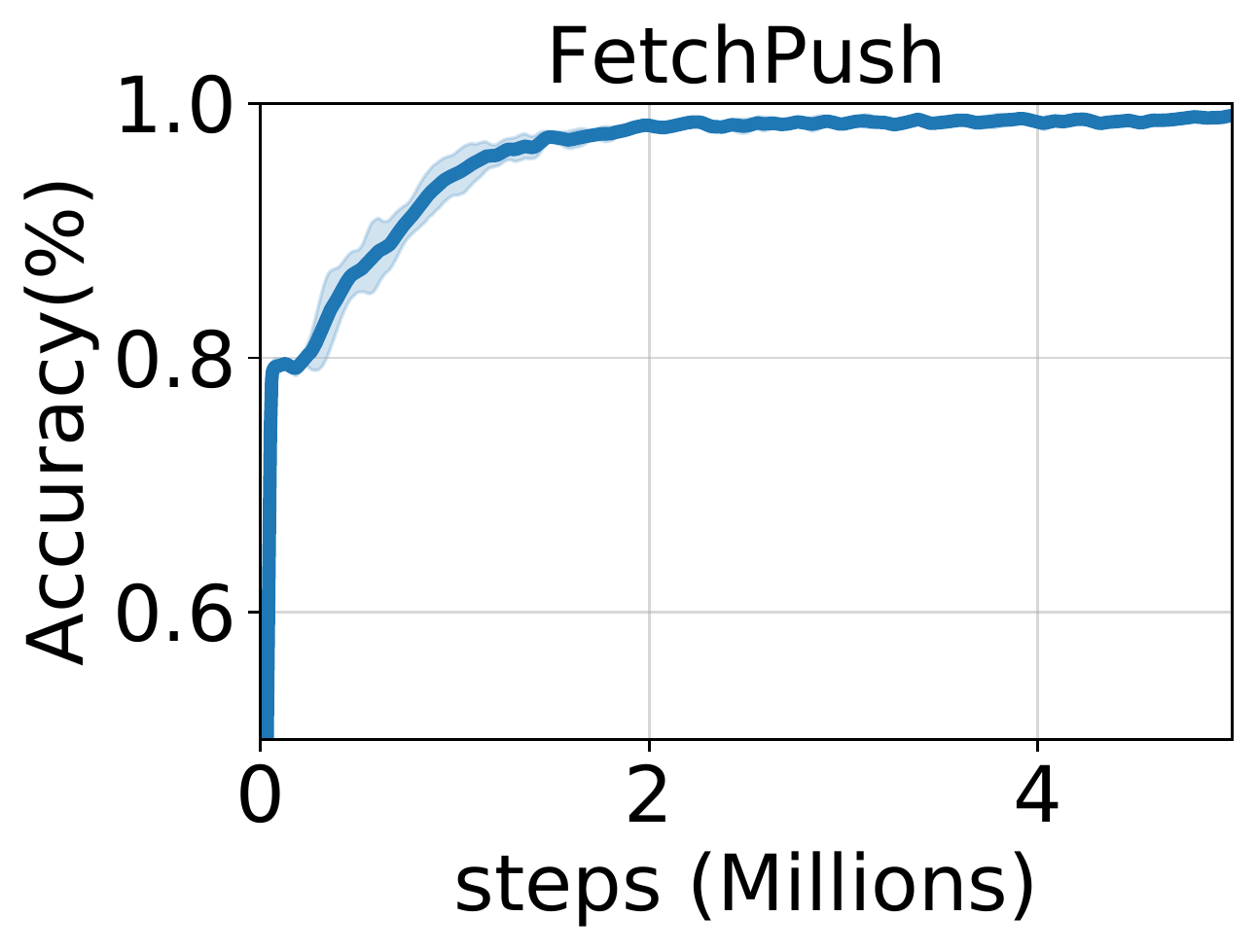}
    \includegraphics[draft=false, height=6.5\baselineskip, valign=b]{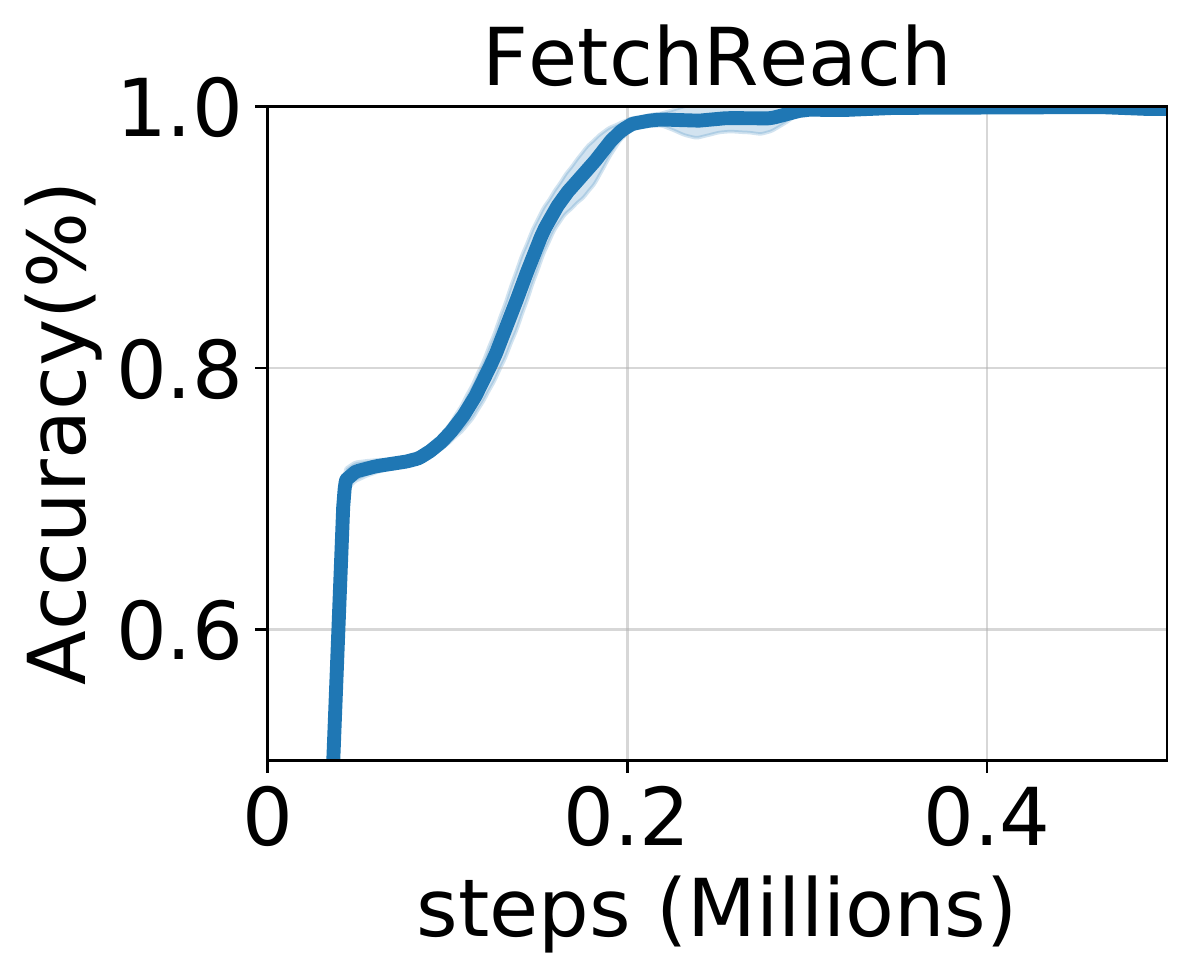}
    \includegraphics[draft=false, height=6.5\baselineskip, valign=b]{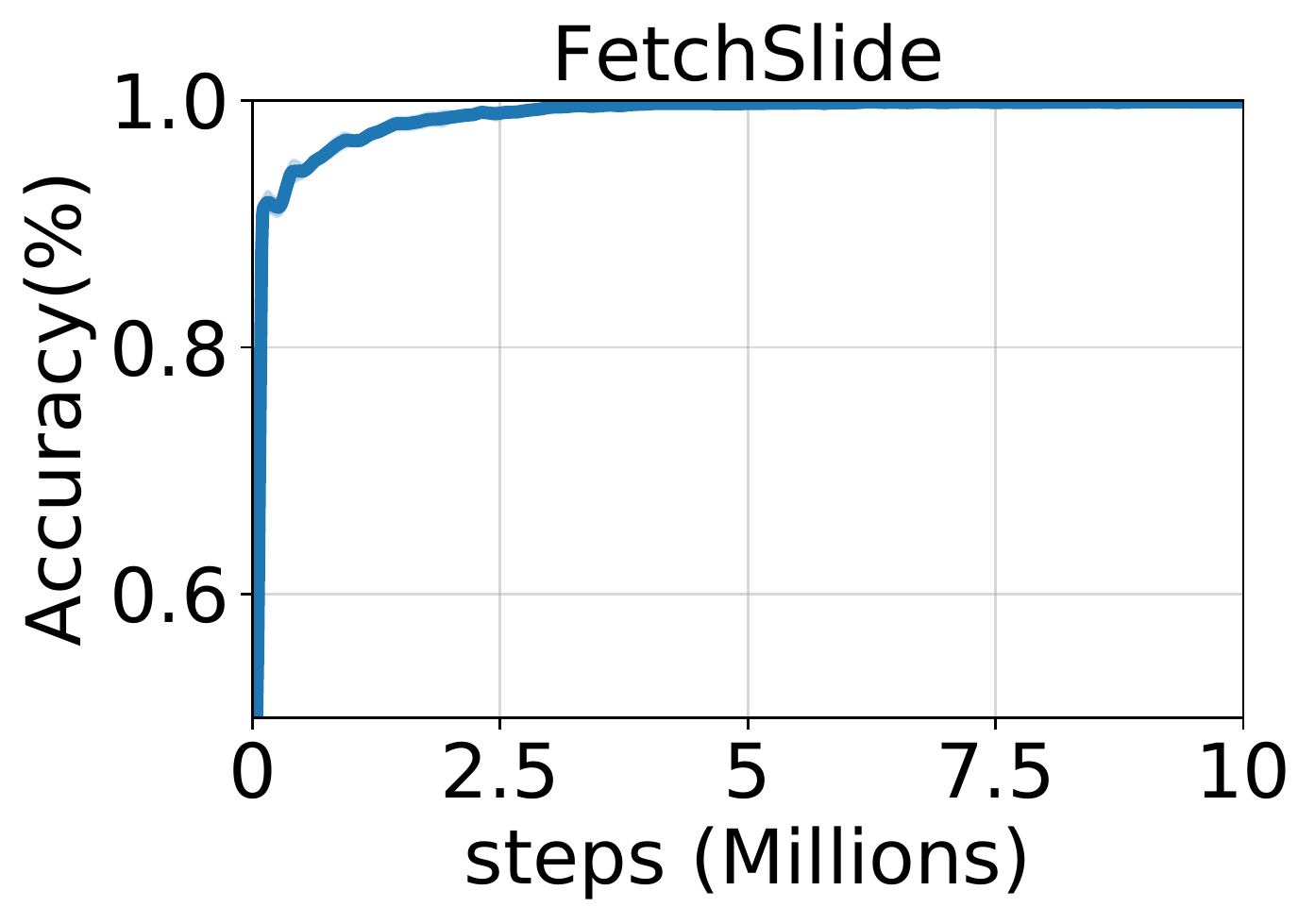}
    \includegraphics[draft=false, height=6.5\baselineskip, valign=b]{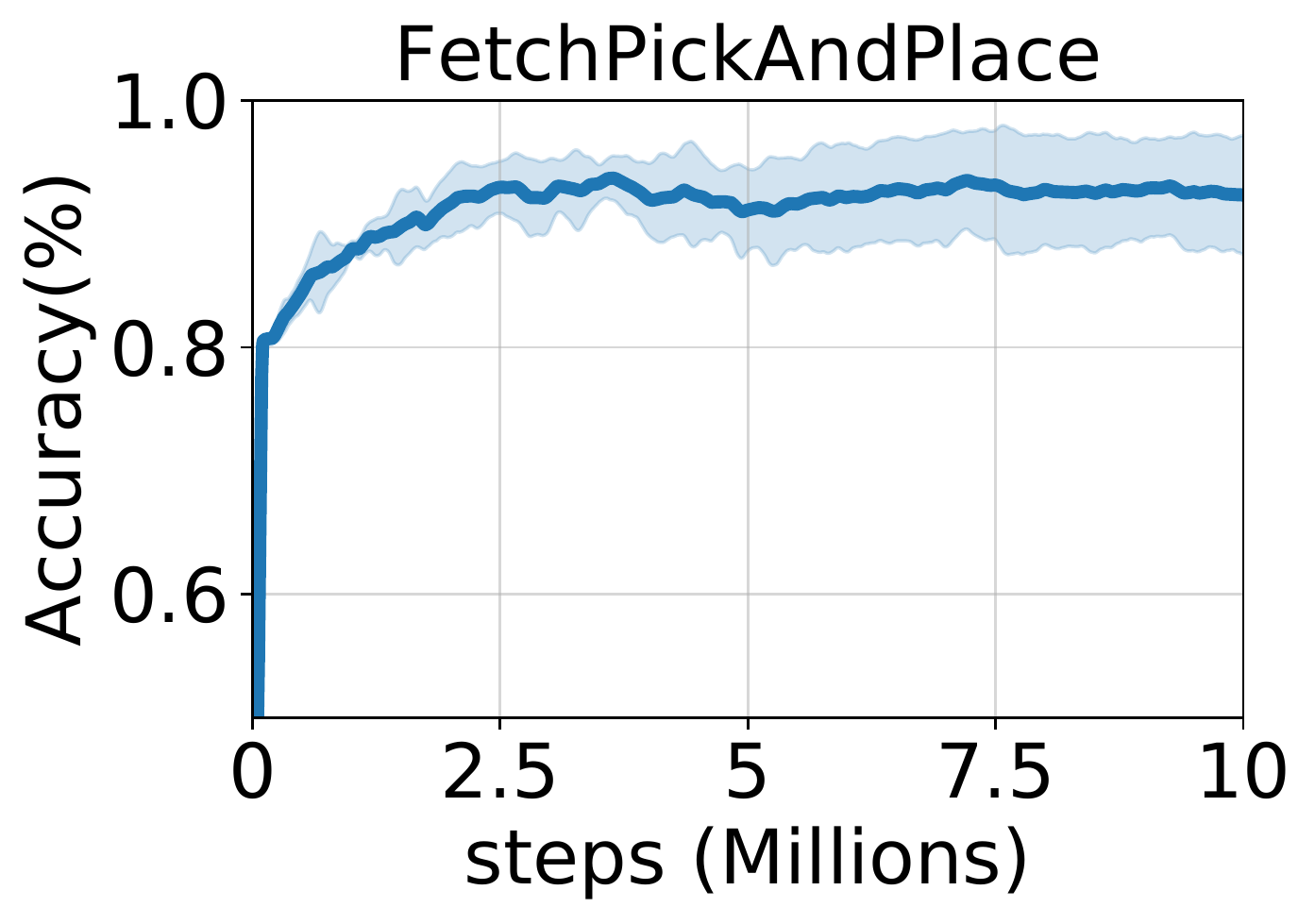}
    {
    \parbox{0.20\linewidth}{\hspace{+50pt} \small\centering (a)}
    \parbox{0.25\linewidth}{\hspace{+50pt} \small\centering (b)}
    \parbox{0.20\linewidth}{\hspace{+25pt} \small\centering (c)}
    \parbox{0.25\linewidth}{\hspace{+30pt} \small\centering (d)}
    }
    \caption{
        The accuracy of the learned reachability network on (a) \fpush{} (b) \freach{}, (c) \fslide{} and (d) \fpnp{} in \fetch{} in terms of environment steps.
    }
    \label{fig:fetch_rnet}
\end{figure}

\cutsectiondown
\subsection{Accuracy of the reachability network}
We measured the accuracy of the reachability network on \grid{}, \dmlab{}, \atari, and \fetch{} in~\Cref{fig:grid_rnet},~\Cref{fig:dmlab_rnet},~\Cref{fig:atari_rnet}, and~\Cref{fig:fetch_rnet}.
The accuracy was measured on the validation set, where we constructed the validation set by sampling 15,000 positive and negative samples respectively from the replay buffer of size 60,000 for \grid{}, \dmlab{}, and 7,500 positive and negative samples respectively from the replay buffer of size 30,000 for \atari, \fetch{}. Specifically, for an anchor $s_t$, we sample the positive sample $s_{t'}$ from $t'\in [t+1,t+k+\Delta t]$, and the negative sample $s_{t''}$ from $t''\in[t+k+\Delta t + \Delta_{-},t+k+\Delta t + 2\Delta_{-}]$.

The RNet reaches an accuracy higher than 80\% in only 0.4M steps in both \grid{} and \dmlab{}. We note that this is quite high considering the unavoidable noise in the negative samples; since the negative samples are sampled based on the temporal distance, not based on the actual reachability, they have a non-zero probability of being reachable, in which case they are in fact the positive samples. For most of the games of \atari, the accuracy of RNet was above 90\% except for \montezuma where the distribution shift in the state space occurs within a task. We note that RNet maintains a high accuracy above 80\% throughout the whole training process due to the additional stabilizing techniques such as weight decay. See~\Cref{appendix:atari-rnet-detail} for more details.
~\Cref{fig:fetch_rnet} summarizes the performance of RNet in \fetch{} tasks. Overall, RNet achieves more than 95\% accuracy in less than 1M steps on all four tasks. Considering that predicting the shortest-path distance is much harder in the continuous action domain than the discrete action domains, this result indicates that RNet can be efficiently trained from the time contrastive learning objective even in challenging continuous action domains.
\subsection{Ablation study: comparison between the learned RNet and the GT-RNet}
In this section, we study the effect of RNet's accuracy on the \sprl{}'s performance. To this end, we implement and compare the ground-truth reachability network by computing the ground-truth distance between a pair of states in \grid{}.

\begin{figure}[!h]
    \centering
    \includegraphics[draft=false, height=6.5\baselineskip, valign=b]{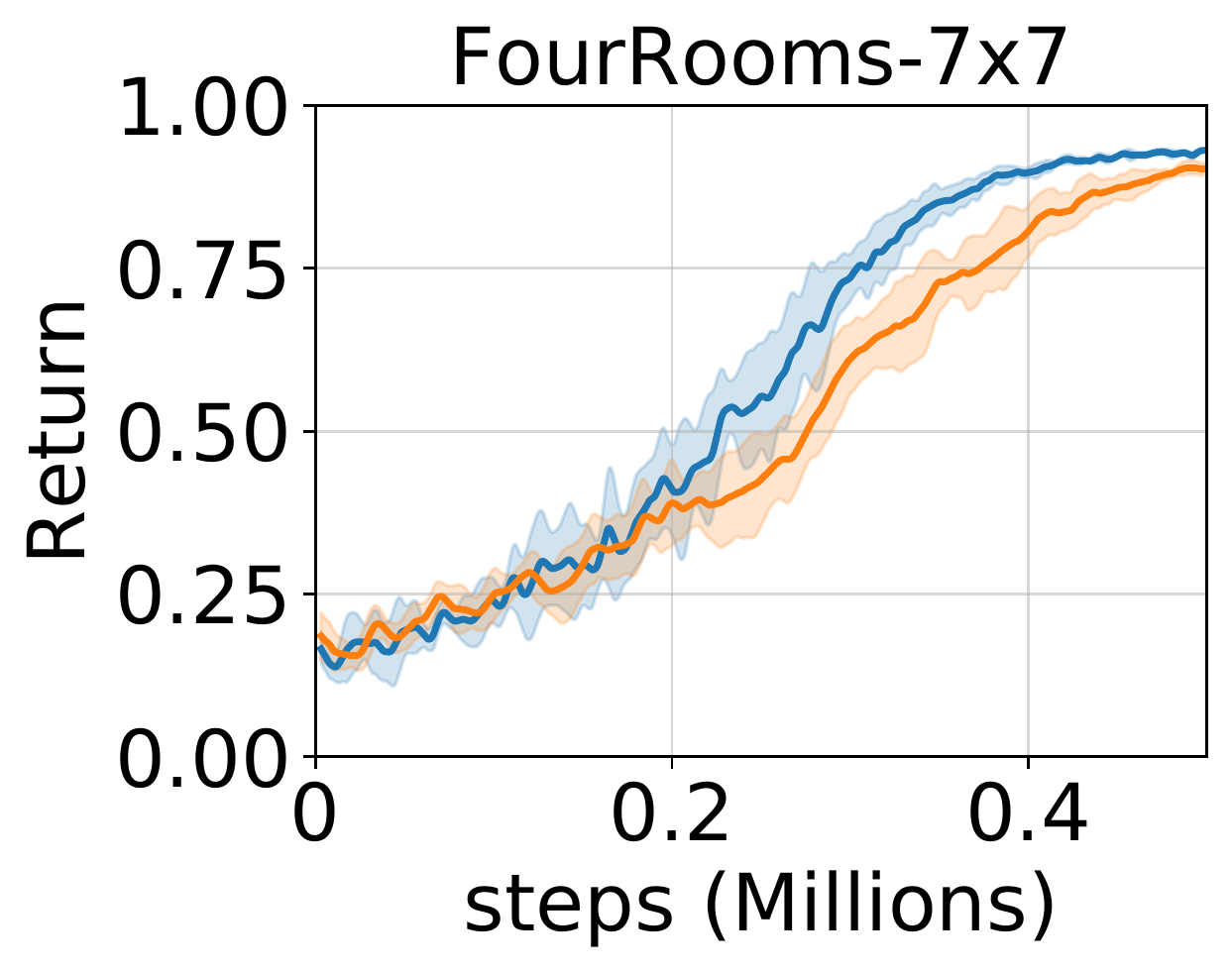}
    \includegraphics[draft=false, height=6.5\baselineskip, valign=b]{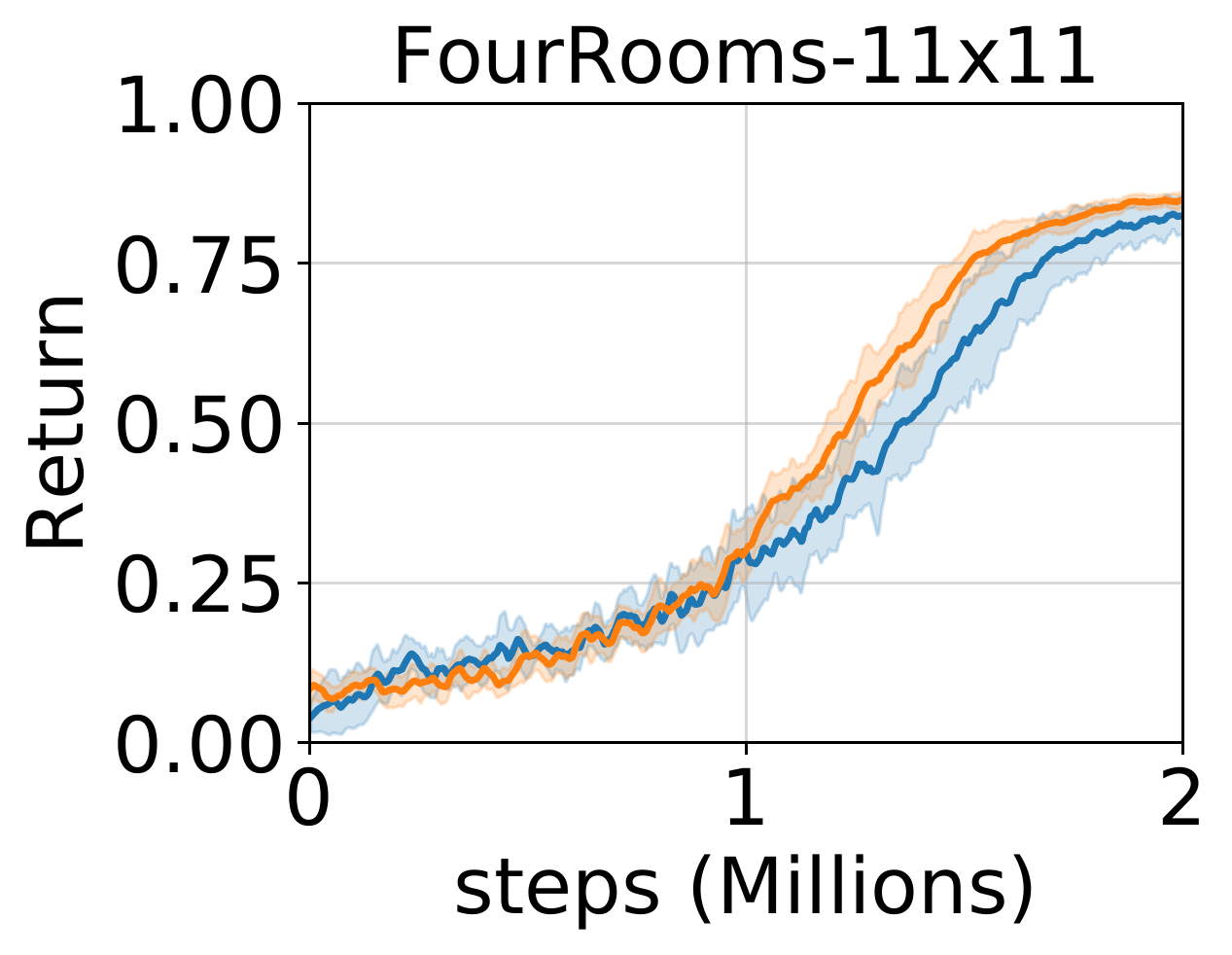}
    \includegraphics[draft=false, height=6.5\baselineskip, valign=b]{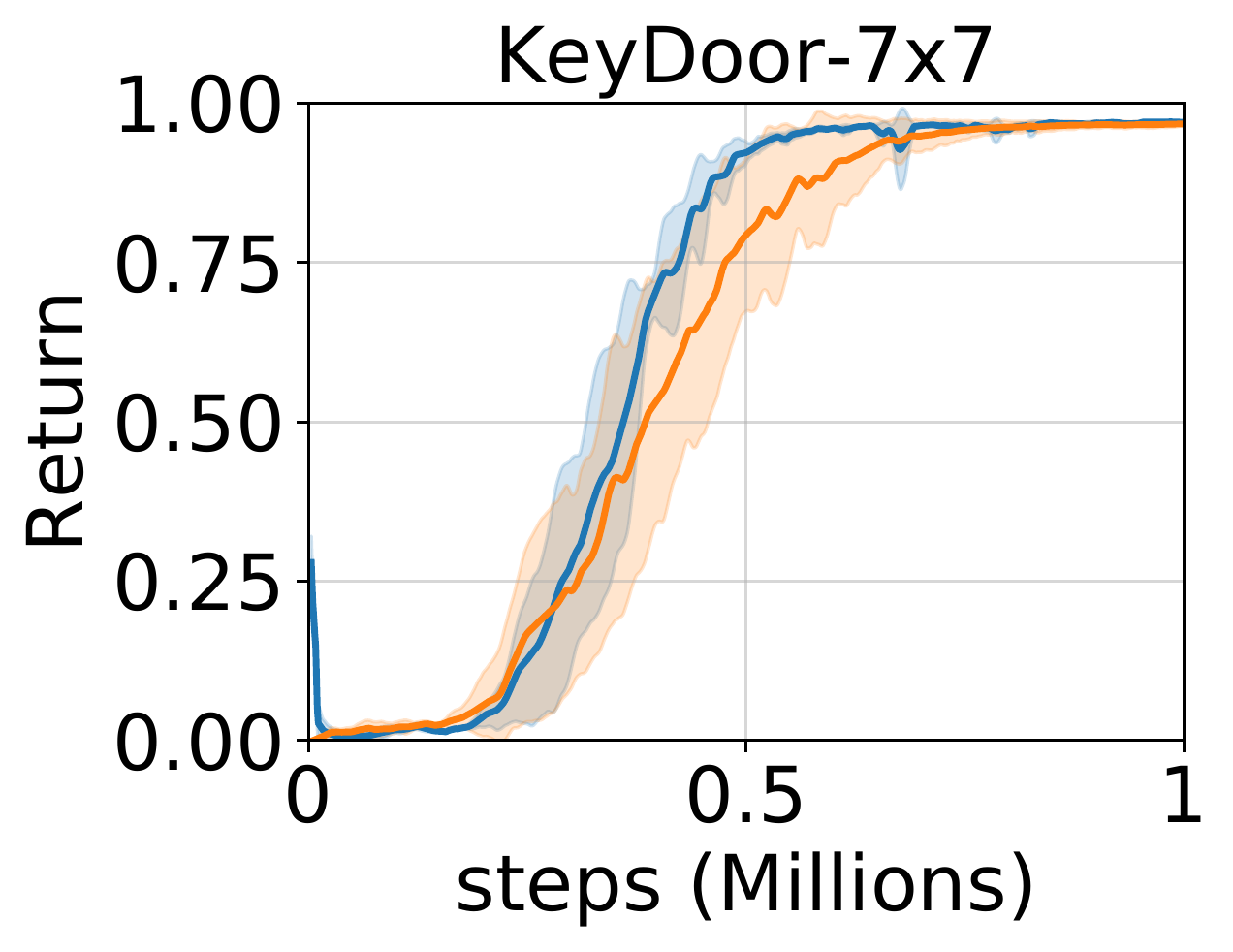}
    \includegraphics[draft=false, height=6.5\baselineskip, valign=b]{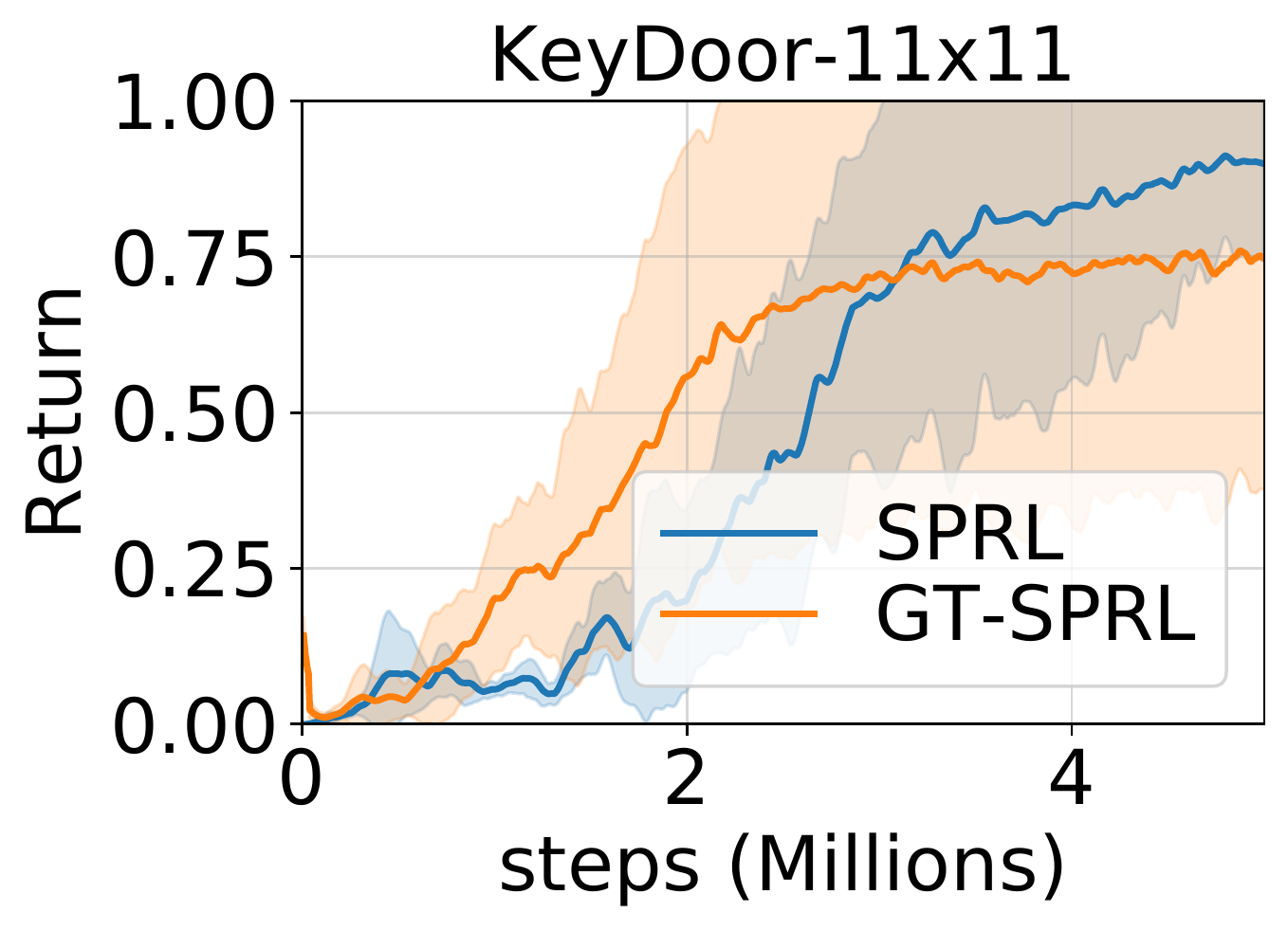}
    {
    \parbox{0.20\linewidth}{\hspace{+60pt} \small\centering (a)}
    \parbox{0.25\linewidth}{\hspace{+40pt} \small\centering (b)}
    \parbox{0.20\linewidth}{\hspace{+25pt} \small\centering (c)}
    \parbox{0.25\linewidth}{\hspace{+20pt} \small\centering (d)}
    }
    \caption{
        The accuracy of the learned reachability network on (a) \fours{} (b) \fourl{}, (c) \keys{} and (d) \keyl{} in \grid{} in terms of environment steps.
    }
    \label{fig:grid_rnet_ablation}
\end{figure}

\paragraph{Ground-truth reachability network}
Ground-truth reachability network was implemented by computing the distance between the two-state inputs, and comparing it with $k$. For the state inputs $s$ and $s'$, we roll out all possible $k$-step trajectories starting from the state $s$ using the ground-truth single-step forward model. If $s'$ is ever visited during the roll-out, the output of the $k$-reachability network is 1 and otherwise, the output is 0.
\paragraph{Result.}
We compared the performance of our \sprl{} with the learned RNet and the ground-truth RNet (\textbf{GT-SPRL}) in~\Cref{fig:grid_rnet_ablation} with the best hyperparameters.
Overall, the performance of \sprl{} and \textbf{GT-SPRL} are similar. This is partly because the learned RNet achieves quite high accuracy in the early stage of learning (see~\Cref{fig:grid_rnet}). Interestingly, we can observe that our SPRL with learned RNet performs better than SPRL with GT-RNet on \fours{} and \keys{}. This is possible since a small noise in RNet output can have a similar effect to the increased tolerance $\Delta t$ on RNet, which makes the resulting cost denser, which may be helpful depending on the tasks and hyperparameters.

%% file: 8_appendix_bonus.tex
\cutsectionup
\section{Qualitative analysis on \ksp cost}

\begin{figure}[!h]
    \centering
    \hspace{-10pt}
    \includegraphics[draft=false, width=0.25\linewidth, valign=t]
    {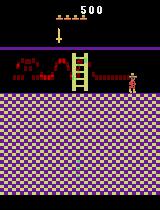}
    \hspace{+10pt}
    \includegraphics[draft=false, width=0.25\linewidth, valign=t]
    {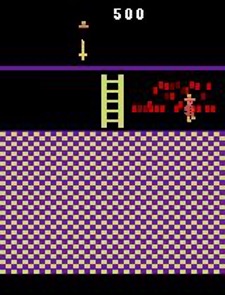}
    \hspace{-10pt}
    \\
    {
    \hspace{-10pt}
    \parbox{0.10\linewidth}{\hspace{-5pt} \small\centering (a)}
    \hspace{+75pt}
    \parbox{0.10\linewidth}{\hspace{+5pt} \small\centering (b)}
    }
    \caption{
    The visualization of \ksp and non-\ksp trajectories in \textit{Room 6} of \montezuma.
    We visualized the agent's trajectories where it receives (a) low \ksp cost (\ie, lowest 25\%) and (b) high \ksp cost (\ie, highest 25\%) for 60 steps. The intensity of the red color shows the count of the visited coordinate of the agent. In high \ksp cost trajectory, agent tends to move back and forth between the ladder and the right corner of room 6 while in low \ksp cost trajectory, agent moves forward without redundancy in the transition. %
    }
    \label{fig:sprl_bonus}
\end{figure}

We visually inspected the agent's trajectory with high and low \ksp cost to see whether \sprl{} can correctly differentiate the shortest-path trajectory.
~\Cref{fig:sprl_bonus} (a) is the visualization of the agent's trajectory with low \ksp cost (\ie, shortest-path) trajectory and~\Cref{fig:sprl_bonus} (b) shows the trajectory with high \ksp cost (\ie, non-shortest path). 
We can observe that when the agent's trajectory has many redundancies (\eg, moving back-and-forth or staying still in the same place), a high \ksp cost is given to the agent, while the agent is less penalized when it takes the shortest path.
\cutsectiondown

%% file: 8_appendix_experiment_detail.tex
\cutsectionup
\newpage
\section{Experiment details of \grid{} domain}\label{appendix:minigrid-detail}
\cutsectiondown
\subsection{Environment}
\grid{} is a 2D grid-world environment with diverse predefined tasks \citep{gym_minigrid}.
It has several challenging features such as pictorial observation, random initialization of the agent and the goal, complex action space, and transition dynamics involving the agent's orientation of movement and changing object status via interaction (e.g., key-door). \paragraph{State Space.}
An observation $\mb{s}_t$ is represented as $H\times W\times C$ tensor, where $H$ and $W$ are the height and width of the map respectively, and $C$ is the features of the objects in the grid. The $(h,w)$-th element of observation tensor is $(type,\;color,\;status)$ of the object and for the coordinate of the agent, the $(h,w)$-th element is $(type,\;0 ,\;direction)$. The map size (\ie, $H\times W$) varies depending on the task; \eg, for \fours task, the map size is $7\times 7$.
\paragraph{Action Space and transition dynamics}
The episode terminates in 100 steps, and the episode may terminate earlier if the agent reaches the goal before 100 steps.
The action space consists of seven discrete actions with the following transitions.
\vspace{-5pt}
\begin{itemize}
    \item \texttt{Turn-Counter-Clockwise}: change the $direction$ counter-clockwise by 90 degree.
    \item \texttt{Turn-Clockwise}: change the $direction$ clockwise by 90 degree.
    \item \texttt{Move-Forward}: move toward $direction$ by 1 step unless blocked by other objects.
    \item \texttt{Pick-up-key}: pickup the key if the key is in front of the agent.
    \item \texttt{Drop-the-key}: drop the key in front of the agent.
    \item \texttt{Open/Close-doors}: open/close the door if the door is in front of the agent.
    \item \texttt{Optional-action}: not used
\end{itemize}
\cutitemizedown

\paragraph{Reward function.}
The reward is given only if the agent reaches the goal location, and the reward magnitude is $1 - 0.9(\text{length of episode} / \text{maximum step for an episode})$. Thus, the agent can maximize the reward by reaching the goal location in the shortest time.

\subsection{Tasks}
In \fours and \fourl, the map structure has four large rooms, and the agent needs to reach the goal. In \keys and \keyl, the agent needs to pick up the key, go to the door, and open the door before reaching the goal location.

\subsection{Architecture and hyper-parameters}
We used a simple CNN architecture similar to~\citep{mnih2015human} for the policy network.
The network consists of \texttt{Conv1(16x2x2-1/SAME)}-\texttt{CReLU}-\texttt{Conv2(8x2x2-1/SAME)}-\texttt{CReLU}-\texttt{Conv3(8x2x2-1/SAME)}-\texttt{CReLU}-\texttt{FC(512)}-\texttt{FC(action-dimension)}, where \texttt{SAME} padding ensures the input and output have the same size (\ie, width and height) and CReLU~\citep{shang2016understanding} is a non-linear activation function applied after each layer. We used Adam~\citep{kingma2014adam} optimizer to optimize the policy network.

For hyper-parameter search, we swept over a set of hyper-parameters specified in~\Cref{table:param-grid}, and chose the best one in terms of the mean AUC over all the tasks, which is also summarized in~\Cref{table:param-grid}.

\begin{table}[]
\centering
\begin{tabular}{l|c|c}
\toprule
\multicolumn{3}{c}{\ppo{}}\\
\midrule
\textbf{Hyperparameters} & \textbf{Sweep range} & \textbf{Final value} \\ [0.5ex]
\midrule
Learning rate           & 0.001, 0.002, 0.003 & 0.003 \\
Entropy                 & 0.003, 0.005, 0.01, 0.02, 0.05 & 0.01 \\
\toprule
\multicolumn{3}{c}{\icm{}}\\
\midrule
\textbf{Hyperparameters} & \textbf{Sweep range} & \textbf{Final value}\\ [0.5ex]
\midrule
Learning rate           & 0.001, 0.002, 0.003 & 0.003 \\
Entropy                 & - & 0.01 \\
Forward/Inverse model loss weight ratio & 0.2, 0.5, 0.8, 1.0 & 0.8 \\
Curiosity module loss weight & 0.03, 0.1, 0.3, 1.0 & 0.3 \\
\icm{} bonus weight & 0.1, 0.3, 1.0, 3.0 & 0.1 \\
\toprule
\multicolumn{3}{c}{\oracle{}}\\
\midrule
\textbf{Hyperparameters} & \textbf{Sweep range} & \textbf{Final value}\\ [0.5ex]
\midrule
Learning rate           & 0.001, 0.002, 0.003 & 0.003 \\
Entropy                 & - & 0.01 \\
\oracle{} bonus weight  & 0.003, 0.01, 0.03, 0.1, 0.3 & 0.01\\
\toprule
\multicolumn{3}{c}{\eco{}}\\
\midrule
\textbf{Hyperparameters} & \textbf{Sweep range} & \textbf{Final value}\\ [0.5ex]
\midrule
Learning rate           & - & 0.003 \\
Entropy                 & - & 0.01 \\
$k$                     & 3, 5, 7 & 3\\
\eco{} bonus weight     & 0.001, 0.002, 0.005, 0.01 & 0.001\\
\toprule
\multicolumn{3}{c}{\sprl{}}\\
\midrule
\textbf{Hyperparameters} & \textbf{Sweep range} & \textbf{Final value}\\ [0.5ex]
\midrule
Learning rate           & 0.003, 0.01 & 0.01 \\
Entropy                 & - & 0.01 \\
$k$                     & 2, 5 & 2\\
Tolerance ($\Delta t$)  & - & 1 \\
Negative bias ($\Delta^-$)     & 10, 20 & 20 \\
Positive bias ($\Delta^+$)     & - & 5 \\
Cost scale ($\lambda$)  & 0.001, 0.002, 0.005 & 0.002 \\
$N_{\Delta t}$ & 30, 60 & 60 \\
\bottomrule
\end{tabular}
\medskip
\caption{The range of hyperparameters sweeped over and the final hyperparameters used in \grid{} domain.}
\label{table:param-grid}
\end{table}

\cutsectionup
\section{Experiment details of \dmlab{} domain}\label{appendix:dmlab-detail}
\cutsectiondown
\subsection{Environment}
\dmlab{} is a 3D-game environment with a first-person view. Along with random initialization of the agent and the goal, complex action space including directional change, random change of texture, color, and maze structure are features that make tasks in \dmlab{} hard to be learned.
\paragraph{State Space.}
A state $\mb{s}_t$ has the dimension of $84\times 84\times 3$. The state is given as a first-person view of the map structure. We resized the $RGB$ image into $84\times 84\times 3$ $RGB$ image and normalized it by dividing the pixel value by 255.
\paragraph{Action Space and transition dynamics}
The episode terminates after the fixed number of steps regardless of the goal being achieved.
The original action space consists of seven discrete actions: \texttt{Move-Forward}, \texttt{Move-Backward}, \texttt{Strafe Left}, \texttt{Strafe Right}, \texttt{Look Left}, \texttt{Look Right}, \texttt{Look Left, and Move-Forward}, \texttt{Look Right and Move-Forward}. In our experiment, we used eight discrete actions with the additional action \texttt{Fire} as in~\citet{higgins2017darla, vezhnevets2017feudal, savinov2018episodic, espeholt2018impala, khetarpal2018attend}.

\subsection{Tasks}
We tested our agent and compared methods on three standard tasks in \dmlab{}: \dmgs{}, \dmgl{}, and \dmom{} which correspond to \texttt{explore\_goal\_locations\_small}, \texttt{explore\_goal\_locations\_large}, and \texttt{explore\_object\_rewards\_many}, respectively.
\dmgs{} and \dmgl{} have a single goal in the maze, but the size of the maze is larger in \dmgl{} than \dmgs{}. The agent and goal locations are randomly set at the beginning of the episode and the episode length is fixed to 1,350 steps for \dmgs{} and 1,800 steps for \dmgl{}. When the agent reaches the goal, it positively rewards the agent and the agent is re-spawned in a random location without terminating the episode, such that the agent can reach the goal multiple times within a single episode. Thus, the agent's goal is to reach the goal location as many times as possible within the episode length.
\dmom{} has multiple objects in the maze, where reaching the object positively rewards the agent and the object disappears. The episode length is fixed to 1,800 steps. The agent's goal is to gather as many objects as possible within the episode length.
\begin{algorithm}[t]
\caption{Sampling the triplet data from an episode for RNet training}\label{alg:sample}
\begin{algorithmic}[1]
\REQUIRE{ Hyperparameters: $k\in\mbb{N}$, Positive bias $\Delta^{+}\in\mbb{N}$, Negative bias $\Delta^{-}\in\mbb{N}$ }
\STATE Initialize $t_{\text{anc}}\leftarrow 0$.
\STATE Initialize $S_{\text{anc}}=\emptyset$, $S_{+}=\emptyset$, $S_{-}=\emptyset$.
\WHILE{$t_{\text{anc}} < T$}
    \STATE $S_{\text{anc}} = S_{\text{anc}} \cup \{s_{t_{\text{anc}}}\}$.
	\STATE $t_{+}=\text{Uniform}(t_{\text{anc}}+1, t_{\text{anc}}+k)$.
	\STATE $t_{-}=\text{Uniform}(t_{\text{anc}}+k+\Delta^{-}, T)$.
	\STATE $S_{+} = S_{+} \cup \{s_{t_{+}}\}$.
	\STATE $S_{-} = S_{-} \cup \{s_{t_{-}}\}$.
	\STATE $t_{\text{anc}}=\text{Uniform}(t_{+}+1, t_{+}+\Delta^{+})$.
\ENDWHILE
\STATE Return $S_{\text{anc}}, S_{+}, S_{-}$
\end{algorithmic}
\end{algorithm}

\subsection{Reachability network Training}
\label{sec:appendix-reachability-training}
Similar to~\citet{savinov2018episodic}, we used the following contrastive loss for training the reachability network:
\begin{align} %
    \mathcal{L}_{\text{Rnet}} = - \log\left(\text{Rnet}_{k-1}(s_{\text{anc}}, s_+) \right) - \log\left(1 - \text{Rnet}_{k-1}(s_{\text{anc}}, s_-) \right),\label{eq:rnet-loss2}
\end{align}
where $s_{\text{anc}}, s_+, s_-$ are the anchor, positive, and negative samples, respectively. The anchor, positive and negative samples are sampled from the same episode, and their time steps are sampled according to~\Cref{alg:sample}.
The RNet is trained in an off-policy manner from the replay buffer with the size of 60K environment steps collecting agent's online experience.
We found that adaptive scheduling of RNet is helpful for faster convergence of RNet. Out of 20M total environment steps, for the first 1M, 1M, and 18M environment steps, we updated RNet every $6K$, $12K$, and $36K$ environment steps, respectively. For all three environments of \dmlab{}, RNet accuracy was $\sim0.9$ after 1M steps.

\paragraph{Multiple tolerance.}
In order to improve the stability of Reachability prediction, we used the statistics over multiple samples rather than using a single-sample estimate as suggested in~\Cref{eq:cost-tol}.
As a choice of sampling method, we simply used multiples of tolerance. In other words, given $s_{t-(k+\Delta{}t)}$ and $s_t$ as inputs for reachability network, we instead used $s_{t-(k+n\Delta{}t)}$ and $s_t$ where $1 \le n \le N_{\Delta t}$, $n \in \mathbb{N}$ and $N_{\Delta t}$ is the number of tolerance samples. We used 90-percentile of $N_{\Delta t}$ outputs of reachability network, Rnet$_{k-1}(s_{t-(k+n\Delta{}t)}, s_t)$, as in \citep{savinov2018episodic} to get the representative of the samples.

\subsection{Architecture and hyper-parameters}
Following~\citep{savinov2018episodic}, we used the same CNN architecture used in~\citep{mnih2015human}. %

For \sprl{}, we used a smaller reachability network (RNet) architecture compared to \eco{} to reduce the training time.
The RNet is based on the siamese architecture with two branches. 
Following~\citep{savinov2018episodic}, \eco{} used Resnet-18~\citep{he2016deep} architecture with 2-2-2-2 residual blocks and 512-dimensional output fully connected layer to implement each branch. For \sprl{}, we used Resnet-12 with 2-2-1 residual blocks and 512-dimensional output fully connected layer to implement each branch.
The RNet takes two-states as inputs, and each state is fed into each branch. The outputs of the two branches are concatenated and forwarded to three\footnote{\citet{savinov2018episodic} used four 512-dimensional fully-connected layers.} 512-dimensional fully-connected layers to produce one-dimensional sigmoid output, which predicts the reachability between two state inputs. We also resized the observation to the same dimension as policy (\ie, $84\times 84\times 3$, which is smaller than the original $120\times 160\times 3$ used in~\citep{savinov2018episodic}).

For all the baselines (\ie, \ppo{}, \eco{}, \icm{}, and \oracle{}), we used the best hyperparameter used in~\citep{savinov2018episodic}. For \sprl{}, we searched over a set of hyperparameters specified in~\Cref{table:param-dmlab}, and chose the best one in terms of the mean AUC over all the tasks, which is also summarized in~\Cref{table:param-dmlab}.
\begin{table}[]
\centering
\begin{tabular}{l|c|c}
\toprule
\textbf{Hyperparameters for \sprl{}} & \textbf{Sweep range} & \textbf{Final value}\\ [0.5ex]
\midrule
Learning rate           & - & 0.0003 \\
Entropy                 & - & 0.004 \\
$k$                     & 3, 10, 30 & 10\\
Tolerance ($\Delta t$)  & 1, 3, 5 & 1 \\
Negative bias ($\Delta^-$)     & 5, 10, 20 & 20 \\
Positive bias ($\Delta^+$)     & - & 5 \\
Cost scale ($\lambda$)  & 0.02, 0.06, 0.2 & 0.06 \\
Optimizer               & - & Adam \\
$N_{\Delta t}$ & - & 200 \\
\bottomrule
\end{tabular}
\medskip
\caption{The range of hyperparameters sweeped over and the final hyperparameters used for our \sprl{} method in \dmlab{} domain.}
\label{table:param-dmlab}
\end{table}

\cutsectionup
\section{Experiment details of \atari domain}\label{appendix:atari-detail}
\cutsectiondown

\subsection{Environment}
\atari is an important and prominent benchmark in deep reinforcement learning with a high-dimensional visual input. One of the main benefits of using \atari as a testbed is that it covers not only navigational tasks but various tasks such as avoiding and destroying enemies by firing a bullet or changing the map structure using bombs as explained in~\citep{bellemare2013arcade}. Because of the diversity of the task \atari is covering, solving \atari shows that the algorithm has a certain degree of generality.
There exists a variety of preprocessing details for \atari. We mostly followed the implementation of OpenAI Baselines~\citep{baselines}. For detailed information, see~\Cref{table:param-atari-env}.

\begin{table}[!ht]
\centering
\begin{tabular}{l|c}
\toprule
\textbf{Parameter} & \textbf{Value}\\ [0.5ex]
\midrule
Image Width           & 84 \\
Image Height          & 84 \\
Grayscaling           & Yes\\
Number of Actions   & 18\\
Action Repetitions  & 4 \\
Frame Stacking     & 4 \\
End of episode when life lost     & No \\
Reward Clipping  & [-1,1] \\
Discount($\gamma$)               & 0.99 \\
Max Episode Length & 10000 \\
Number of parallel workers          & 12 \\
\bottomrule
\end{tabular}
\medskip
\caption{Preprocessing details for \atari}
\label{table:param-atari-env}
\end{table}

\paragraph{State Space.}
A state $\mb{s}_t$ is represented as $84\times 84\times 4$. We stacked 4 consecutive frames achieved by taking the same action 4 times in a row and resized $RGB$ image into $84\times 84\times 1$ gray image and normalized by dividing the pixel value by 255.

\paragraph{Action Space and transition dynamics}
The episode terminates when the agent loses all of the lives given.
The action space consists of eighteen discrete actions as in~\citep{bellemare2013arcade}.

\subsection{Various Tasks of \atari: Navigational and Non-navigational}
We tested our agent and compared methods on navigational and \textbf{non-navigational tasks} in \atari: \montezuma, \freeway, \mspacman, \gravitar, \seaquest, \hero.
\montezuma is a famous game as a hard exploration game in \atari. An agent should pick up the items such as a key to open the door or a knife to destroy the enemy. In \freeway, an agent should cross the road while avoiding the car. \mspacman is a game where an agent should eat the items and avoid the enemy. Three games mentioned until now have a  navigational feature meaning that the agent can move toward a certain coordinate to get the score. However, the three games to be mentioned have a non-navigational feature meaning that the agent should not only get to a certain coordinate but also use specific action to get the score. In \gravitar and \seaquest, an agent should shoot a bullet. In \hero, an agent can install a bomb to break the wall and move forward. By adding \gravitar, \seaquest, and \hero to our testbed, we evaluated how \sprl{} performs in non-navigational tasks.

\subsection{Reachability network Training}\label{appendix:atari-rnet-detail}

We followed the details of reachability network training for \dmlab{} except for 1) replay buffer size and 2) reachability network training frequency. We changed the size of the replay buffer for the reachability network from $60K$ environment steps to $30K$ environment steps. To avoid overfitting of the reachability network, we enlarged the reachability network training frequency from 6K environment steps to 150K environment steps after the initial 1M environment steps.

\paragraph{Stabilizing Reachability Network.}
In some of the games of \atari, within a task exists a distributional shift in the state space that hinders stable reachability network training. Therefore we had to use some techniques to mitigate the instability problem of the reachability network: Weight decay and Label smoothing. We used weight decay with the factor of 0.03 and label smoothing of 0.1. Also instead of using current and future states as inputs, we used current and pixel-level subtraction of current and future states which stabilized the learning.

\subsection{Architecture and hyper-parameters}
For the policy architecture, we used the same CNN architecture used in~\citet{mnih2015human}.
For the reachability network, we used the same architecture used for \dmlab{} except for the input layer. For the input layer, we used $(\mb{s}_t, \mb{s}_t - \mb{s}_{t-k})$ instead of $(\mb{s}_t, \mb{s}_{t-k})$. This change helped the reachability network to avoid suffering from overfitting when the distribution shift in the state space occurs.

For hyper-parameter search, we swept over a set of hyper-parameters specified in~\Cref{table:param-atari} and chose the best one in terms of the mean AUC over all the tasks, which is also summarized in~\Cref{table:param-atari}.

\begin{table}[]
\centering
\begin{tabular}{l|c|c}
\toprule
\multicolumn{3}{c}{\ppo{}}\\
\midrule
\textbf{Hyperparameters} & \textbf{Sweep range} & \textbf{Final value} \\ [0.5ex]
\midrule
Learning rate           & 0.0001, 0.0002, 0.0003, 0.0005 & 0.0005 \\
Entropy                 & 0.001, 0.003, 0.005, 0.01, 0.03 & 0.01 \\
\toprule
\multicolumn{3}{c}{\icm{}}\\
\midrule
\textbf{Hyperparameters} & \textbf{Sweep range} & \textbf{Final value}\\ [0.5ex]
\midrule
Learning rate           & 0.0005 & 0.0005 \\
Entropy                 & - & 0.01 \\
Forward/Inverse model loss weight ratio & 1.0 & 1.0 \\
Curiosity module loss weight & 1.0 & 1.0 \\
\icm{} bonus weight & 0.0001, 0.0003, 0.001, 0.003, 0.01 & 0.01 \\
\toprule
\multicolumn{3}{c}{\eco{}}\\
\midrule
\textbf{Hyperparameters} & \textbf{Sweep range} & \textbf{Final value}\\ [0.5ex]
\midrule
Learning rate           & 0.0001, 0.0003, 0.0005 & 0.0005 \\
Entropy                 & - & 0.01 \\
\eco{} bonus weight     & 0.001, 0.003, 0.01, 0.03, 0.1 & 0.001\\
\toprule
\multicolumn{3}{c}{\sprl{}}\\
\midrule
\textbf{Hyperparameters} & \textbf{Sweep range} & \textbf{Final value}\\ [0.5ex]
\midrule
Learning rate           &  0.0003, 0.0005 & 0.0005 \\
Entropy                 & - & 0.01 \\
\sprl{} cost scale ($\lambda$)  & 0.01, 0.03, 0.05, 0.1 & 0.05 \\
\toprule
\multicolumn{3}{c}{Reachability network (for \eco{} and \sprl{})}\\
\midrule
$k$                     & 5, 8, 12, 15 & 12\\
Tolerance ($\Delta t$)  & - & 1 \\
Negative bias ($\Delta^-$)     & 80, 100, 120 & 80 \\
Positive bias ($\Delta^+$)     & - & 5 \\
$N_{\Delta t}$ & 30, 50, 100, 200, 400 & 200 \\
\bottomrule
\end{tabular}
\medskip
\caption{The range of hyperparameters sweeped over and the final hyperparameters used in \atari domain.}
\label{table:param-atari}
\end{table}

\cutsectionup
\section{Experiment details of \fetch domain}\label{appendix:fetch-detail}
\cutsectiondown

\subsection{Environment}
\fetch{} is a continuous control environment with a two-fingered gripper. By controlling the gripper, specific interaction between the gripper and the object needs to be accomplished for a given task. The agent and the goal locations are randomly initialized at each episode and these features make the environment difficult to solve.

\paragraph{State Space.}
At each time step $t$, the state input $s_t$ is a vector ($16$-dimensional for \freach{} and $31$-dimensional for \fpush{}, \fslide{}, and \fpnp{}) consisting of the location and the velocity of the gripper. When the object exists in the task, the location, rotation, linear and angular velocities of the object are also included in the state $s_t$.

\paragraph{Action Space}
The action space is continuous. Actions are total 4 dimensional: 3 dimensions for gripper movement and 1 dimension for the opening of the gripper. 

\paragraph{Modification in the environment}
We made two modifications in the episode termination and the reward function to make it a sparse-reward task.
The agent receives +1 reward if the agent reaches the goal and 0 rewards otherwise. Also, the episode terminates when the agent reaches the goal such that the agent can receive a non-zero reward at most once in an episode.

\subsection{Tasks}
Four tasks are available in the \fetch{} domain: \fpush{}, \freach{}, \fslide{}, \fpnp{}
\begin{itemize}
  \item \fpush{} : The goal is to push a box to a target location. The gripper can only push in this task since fingers are not controllable.
  \item \freach{} : The goal is to move the gripper to a target location. This is the easiest task in \fetch{} since moving the gripper is a fundamental skill required in all four tasks.
  \item \fslide{} : The goal is to hit the object, let the object slide, and stop at the desired location by friction.
  \item \fpnp{} : The goal is to grasp a box and move the box to the target location.
\end{itemize}

\subsection{Reward Normalization}\label{appendix:fetch-reward-normalization}

We used reward normalization analogous to~\citet{burda2018large}, \ie, ``dividing the rewards by a running estimate of the standard deviation of the sum of discounted rewards''. Reward normalization has been particularly effective in the \fetch{} environment on every algorithm. We leave the analysis of this phenomenon as future work.

\subsection{Architecture and hyper-parameters}
We used a simple MLP architecture for policy network.
The network consists of \texttt{FC(256)}-\texttt{ReLU}-\texttt{FC(256)}-\texttt{ReLU}-\texttt{FC(256)}-\texttt{ReLU}-\texttt{FC(256)}-\texttt{ReLU}-\texttt{FC(action-dimension)}. The reachability network was also a simple MLP network. The network consists of \texttt{FC(512)}-\texttt{BatchNorm}-\texttt{ReLU}-\texttt{FC(512)}-\texttt{BatchNorm}-\texttt{ReLU}-\texttt{FC(512)}-\texttt{BatchNorm}-\texttt{ReLU}-\texttt{FC(512)}-\texttt{BatchNorm}-\texttt{ReLU}-\texttt{FC(1)}-\texttt{Sigmoid}. We used Adam~\citep{kingma2014adam} optimizer to optimize both networks.

For hyper-parameter search, we swept over a set of hyper-parameters specified in~\Cref{table:param-fetch} and chose the best one in terms of the mean AUC over all the tasks, which is also summarized in~\Cref{table:param-fetch}.

\begin{table}[ht]
\centering
\begin{tabular}{l|c|c}
\toprule
\multicolumn{3}{c}{\ppo{}}\\
\midrule
\textbf{Hyperparameters} & \textbf{Sweep range} & \textbf{Final value} \\ [0.5ex]
\midrule
Learning rate           & - & 0.0005 \\
Entropy                 & 0.001, 0.003, 0.01 & 0.01 \\
Action Noise            & - & 0.003 \\
\toprule
\multicolumn{3}{c}{\icm{}}\\
\midrule
\textbf{Hyperparameters} & \textbf{Sweep range} & \textbf{Final value}\\ [0.5ex]
\midrule
Learning rate           & - & 0.0005 \\
Entropy                 & - & 0.001 \\
Action Noise            & - & 0.003 \\
Forward/Inverse model loss weight ratio & - & 1.0 \\
Curiosity module loss weight & - & 1.0 \\
\icm{} bonus weight & 0.0001, 0.0002, 0.0003, 0.0005, 0.001, 0.003, 0.01 & 0.001 \\
\toprule
\multicolumn{3}{c}{\eco{}}\\
\midrule
\textbf{Hyperparameters} & \textbf{Sweep range} & \textbf{Final value}\\ [0.5ex]
\midrule
Learning rate           & - & 0.0005 \\
Entropy                 & - & 0.001 \\
Action Noise            & - & 0.003 \\
\eco{} bonus weight     & 0.0001, 0.0002, 0.0003, 0.0005, 0.001, 0.003, 0.01 & 0.0001 \\
Bias                    & 0.5, 1.0 & 1.0 \\
\toprule
\multicolumn{3}{c}{\sprl{}}\\
\midrule
\textbf{Hyperparameters} & \textbf{Sweep range} & \textbf{Final value}\\ [0.5ex]
\midrule
Learning rate           &  - & 0.0005 \\
Entropy                 & - & 0.001 \\
\sprl{} cost scale ($\lambda$)  & 0.0001, 0.0002, 0.0003, 0.0005, 0.001, 0.003, 0.01 & 0.05 \\
Bias                    & 0.5, 1.0 & 1.0 \\
\toprule
\multicolumn{3}{c}{Reachability network (for \eco{} and \sprl{})}\\
\midrule
$k$                     & - & 10\\
Tolerance ($\Delta t$)  & - & 1 \\
Negative bias ($\Delta^-$)     & 1, 2, 3, 5, 8, 12 & 12 \\
Positive bias ($\Delta^+$)     & - & 5 \\
$N_{\Delta t}$ & - & 20 \\
\bottomrule
\end{tabular}
\medskip
\caption{The range of hyperparameters swept over and the final hyperparameters used in \fetch{} domain.}
\label{table:param-fetch}
\end{table}

%% file: 8_appendix_option_view.tex
\cutsectionup
\section{Option framework-based formulation}
\label{appendix:option-framework}
\cutsectiondown
\subsection{Preliminary: option framework}
Options framework~\citep{sutton1998between} defines options as a generalization of actions to include a temporally extended series of action. Formally, options consist of three components: a policy $\pi: \mathcal{S} \times \mathcal{A} \rightarrow$ $[0,1],$ a termination condition $\beta: \mathcal{S}^{+} \rightarrow[0,1],$ and an initiation set $\mathcal{I} \subseteq \mathcal{S} .$ An option $\langle\mathcal{I}, \pi, \beta\rangle$ is available in state $s$ if and only if $s \in \mathcal{I}$.
If the option is taken, then actions are selected according to $\pi$ until the option terminates stochastically according to $\beta $. 
Then, the option-reward and option-transition models are defined as
\begin{align}
r_{s}^{o} &= \mathbb{E}\left\{r_{t+1}+\gamma r_{t+2}+\cdots+\gamma^{k-1} r_{t+k} \mid E(o, s, t)\right\}\label{eq:opt-rew}\\
P_{s s^{\prime}}^{o}&=\sum_{k=1}^{\infty} p\left(s^{\prime}, k\right) \gamma^{k}\label{eq:opt-trans}
\end{align}
where $t+k$ is the random time at which option $o$ terminates, $E(o, s, t)$ is the event that option $o$ is initiated in state $s$ at time $t$, and $p(s', k)$ is the probability that the option terminates in $s'$ after $k$ steps. 
Using the option models, we can re-write Bellman equation as follows:
\begin{align}
V^\pi(s) &= \mbb{E}\left[
r_{t+1}+\cdots+\gamma^{k-1}r_{t+k}+\gamma^k V^\pi(s_{t+k})
\right],\\
&=\sum_{o\in\mc{O}} Pr[ E(o, s) ] \left[ r_s^{o} + \sum_{s'} P_{ss'}^oV^\pi(s') \right].\label{eq:opt-bellman}
\end{align}
where $t+k$ is the random time at which option $o$ terminates and $E(o, s)$ is the event that option $o$ is initiated in state $s$.

\subsection{Option-based view-point of shortest-path constraint}
In this section, we present an option framework-based viewpoint of our shortest-path (SP) constraint. We will first show that a (sparse-reward) MDP can be represented as a weighted directed graph where nodes are rewarding states, and edges are options. Then, we show that a policy satisfying SP constraint also maximizes the option-transition probability $P_{s s^{\prime}}^{o}$.

For a given MDP $\mathcal{M}=(\mc{S, A, R, P},\rho, \bar{\mc{S}})$, let $\mc{S}^{R}=\{s|R(s)\neq 0\}\subset\mc{S}$ be the set of all rewarding states, where $R(s)$ is the reward function upon arrival to state $s$. In sparse-reward tasks, it is assumed that $|\mc{S}^{R}| << |\mc{S}|$.
Then, we can form a weighted directed graph $G^\pi=(\mc{V}, \mc{E})$ of policy $\pi$ and given MDP. 
The vertex set is defined as $\mc{V}=\mc{S}^{R}\cup\rho_0\cup\bar{\mc{S}}$ where $\mc{S}^{R}$ is rewarding states, $\rho_0$ is the initial states, and $\bar{\mc{S}}$ is the terminal states. 
Similar to the path set in~\Cref{def:path-set-pi-nr}, let $\mc{T}_{s\rightarrow s'}$ denotes a set of paths transitioning from one vertex $s\in \mc{V}$ to another vertex $s'\in \mc{V}$:
\begin{align}
\mc{T}_{s\rightarrow s'}=\{\tau|s_0=s, s_{\ell(\tau)}=s', \{s_t\}_{0<t<\ell(\tau)}\cap \mc{V}=\emptyset  \}.
\end{align}
Then, the edge from a vertex $s\in \mc{V}$ to another vertex $s'\in \mc{V}$ is defined by an (implicit) option tuple: $o(s, s') = (\mc{I}, \pi, \beta)_{(s,s')}$, where $\mc{I}=\{s\}$, $\beta(s)=\mbb{I}(s=s')$, and
\begin{align}
     \pi^{(s, s')}(\tau)= 
     \begin{cases}
        \frac{1}{Z}\pi (\tau)   &   \text{for }\tau\in\mc{T}_{s\rightarrow s'}\\
        0                       &   \text{otherwise}
     \end{cases},
\end{align}
where $Z$ is the partition function to ensure $\int\pi^{(s, s')}(\tau)d\tau=1$. 
Following \Cref{eq:opt-rew}, the option-reward is given as
\begin{align}
  r^{\pi}_{s, s'}
&=\mbb{E}^{\pi^{(s,s')}}\left[ r_{t+1}+\gamma r_{t+2}+\cdots+\gamma^{k-1}r_{t+k} \mid E(o^{(s,s')}, s, t) \right],\\
&=\mbb{E}^{\pi^{(s,s')}}\left[ \gamma^{k-1}r_{t+k} \mid E(o^{(s,s')}, s, t)  \right]\label{eq:opt-rew2},
\end{align}
where $t+k$ is the random time at which option $o(s, s')$ terminates, and $E(o, s, t)$ is the event that option $o(s, s')$ is initiated in state $s$ at time $t$. Note that in the last equality, $r_{t+1}=\cdots=r_{t+k-1}=0$ holds since $\{s_{t+1},\ldots, s_{t+k-1}\} \cap \mc{V}=\emptyset$ from the definition of option policy $\pi^{(s, s')}$.
Following \Cref{eq:opt-trans}, the option transition is given as
\begin{align}
\trans
&=\sum_{k=1}^{\infty} p(s', k)\gamma^k\\
&=\mathbb{E}^\pi\left[\gamma^k|s_0=s, s_{k}=s', \{r_t\}_{t< k}=0\right] \label{eq:opt-trans2}\\
&=\gamma^{\distpi}.
\end{align}
where $p(s', k)$ is the probability that the option terminates in $s'$ after $k$ steps, and $\distpi$ is the $\pi$-distance in~\Cref{def:ex-path-len}.
Then, we can re-write the shortest-path constraint in terms of $\trans$ as follows:
\begin{align}
\Pisp &= \{\pi|\forall(s, s'\in\mc{T}^{\pi}_{\hat{s}, \hat{s}', \text{nr}} \text{~s.t.~} (\hat{s}, \hat{s}')\in\nrstatepair),\ \distpi = \min_{\pi} \distpi\}\\
&=  \{\pi|\forall(s, s'\in\mc{T}^{\pi}_{\hat{s}, \hat{s}', \text{nr}} \text{~s.t.~} (\hat{s}, \hat{s}')\in\nrstatepair),\ \trans = \max_{\pi} \trans \}
\label{eq:option-shortest-path-policy}
\end{align}
Thus, we can see that the policy satisfying SP constraint also maximizes the option-transition probability. We will use this result in Appendix~\ref{appendix:g-sp-optimal}.

%% file: 8_appendix_proof_single_goal.tex
\bigskip
\section{Shortest-Path Constraint: A Single-goal Case}
\label{appendix:g-sp-specialcase}
\label{appendix:g-sp-optimal}

In this section, we provide more discussion on a special case of the shortest-path constraint~(\Cref{sec:shortest-path-general}), when the (stochastic) MDP defines a single-goal task:
\ie, there exists a unique initial state $s_{\text{init}}\in\mc{S}$ and a unique goal state $s_g\in\mc{S}$ such that $s_g$ is a terminal state, and $R(s) > 0$ if and only if $s = s_g$.

We first note that the non-rewarding path set is identical to the path set in such a setting, because the condition $r_t = 0 (t<\ell(\tau))$ from \Cref{def:path-set-pi-nr} is always satisfied as $R(s) > 0 \Leftrightarrow s = s_g$ and $s_{\ell(\tau)} = s_g$:
\begin{align}
    \pathsetpi = \gpathsetpi=\{\tau \mid s_0=s, s_{\ell(\tau)}=s', p_\pi(\tau)>0, \{s_t\}_{t<\ell(\tau)}\neq s' \}
\end{align}

Again, $\gpathsetpi$ is a set of all path starting from $s$ (\ie, %
and ending at $s'$ (\ie, $s_{\ell(\tau)}=s'$)
where the agent visits $s'$ \emph{only} at the end (\ie, $\{s_t\}_{t<\ell(\tau)}\neq s'$),
that can be rolled out by policy with a non-zero probability (\ie, $p_\pi(\tau)>0$).

We now claim that an optimal policy satisfies the shortest-path constraint.
The idea is that, since $s_g$ is the only rewarding and terminal state, maximizing $R(\tau) = \gamma^T R(s_g)$ where $s_T = s_g$ corresponds to minimizing the number of time steps $T$ to reach $s_g$. 
In this setting, a shortest-path policy is indeed optimal.

\medskip
\begin{lemma}\label{lem:g-sp-optimal}
For a single-goal MDP, any optimal policy satisfies the shortest-path constraint.
\end{lemma}

\begin{proof}
Let $s_{\text{init}}$ be the initial state and $s_g$ be the goal state. We will prove that any optimal policy is a shortest-path policy from the initial state to the goal state. %
We use the fact that $s_g$ is the only rewarding state, \ie, $R(s) > 0$ entails $s = s_g$.
\begin{align}
\pi^* &= \argmax_\pi \mbb{E}^{\tau\sim\pi}_{s\sim\rho} \left[\textstyle\sum_{t}{\gamma^t r_t\biggmid s_0=s}\right]\\
&=\argmax_\pi \mbb{E}^{\tau\sim\pi}\left[\textstyle\sum_{t}{\gamma^t r_t\biggmid s_0=s_{\text{init}}}\right]\\
&=\argmax_\pi \mbb{E}^{\tau\sim\pi} \left[\gamma^T R(s_g) \mid s_0=s_{\text{init}}, s_{\ell(\tau)}=s_g \right]\\
&=\argmax_\pi \mbb{E}^{\tau\sim\pi} \left[\gamma^T \mid s_0=s_{\text{init}}, s_{\ell(\tau)}=s_g \right]\label{eq:g-sp-optimal}\\
&=\argmin_\pi 
\log_\gamma \left( \mbb{E}^{\tau\sim\pi}\left[\gamma^{T} \mid s_0=s_{\text{init}}, s_{\ell(\tau)}=s_g\right] \right)\\
&=\argmin_\pi D_{\text{nr}}^{\pi}(s_{\text{init}},s_g),\label{eq:sp-optimal}
\end{align}
where~\Cref{eq:g-sp-optimal} holds since $R(s_g)>0$ from our assumption that $R(s)+V^*(s)>0$.
\end{proof}

\clearpage

%% file: 8_appendix_proof_theorems.tex
\section{Proof of \Cref{thm:sp-optimal}}
\label{appendix:sp-optimal}

\cutsectiondown

We make the following assumptions on the Markov Decision Process (MDP) $\mathcal{M}$: namely \emph{mild stochasticity} (\Cref{def:mildstochasticity1,def:mildstochasticity2}).

\begin{definition}[Mild stochasticity (1)]
\label{def:mildstochasticity1}
In MDP $\mathcal{M}$, there exists an optimal policy $\pi^*$ and the corresponding shortest-path policy $\pi^{sp}\in\Pisp$ such that for all $s,s'\in\nrstatepair$, it holds $p_{\pi^*}(\bar{s}=s'|s_0=s)=p_{\pi^{sp}}(\bar{s}=s'|s_0=s)$.
\end{definition}

\begin{definition}[Mild stochasticity (2)]
\label{def:mildstochasticity2}
In MDP $\mathcal{M}$, the optimal policy $\pi^*$ does not visit the same state more than once:
For all $s \in \mathcal{S}$ such that $\rho_{\pi^*}(s)>0$,
it holds $\rho_{\pi^*}(s)=1$, where $\rho_{\pi}(s) \triangleq \mathbb{E}_{s_0\sim\rho_{0}(S), a \sim \pi(A|s), s'\sim(S|s,a) }\left[\sum_{t=1}^{T} \mathbb{I}\left(s_{t}=s\right)\right]$ is the state-visitation count.
\end{definition}
In other words, we assume that the optimal policy does not have a cycle. One common property of MDP that meets this condition is that the reward disappearing after being acquired by the agent.
We note that this assumption holds for many practical environments. In fact, in many cases as well as \atari{}, \dmlab{}, etc.

\medskip
\begin{theorem*}[\ref{thm:sp-optimal}]
For any MDP with the mild stochasticity condition, an optimal policy $\pi^*$ satisfies the shortest-path constraint: $\pi^*\in\Pisp$.
\end{theorem*}
\begin{proof} 
For simplicity, we prove this based on the option-based view point (see \Cref{appendix:option-framework}).
By plugging~\Cref{eq:opt-rew2} and~\Cref{eq:opt-trans2} into~\Cref{eq:opt-bellman}, we can re-write the Bellman equation of the value function $V^\pi(s)$ as follows:
\begin{align}
V^\pi(s)
&=\sum_{o\in\mc{O}} Pr[ E(o, s) ] \left[ 
	r_s^{o} + \sum_{s'} P_{ss'}^oV^\pi(s') 
\right]\\
&=\sum_{s' \in \nrstate} p_{\pi}( \bar{s}=s'|s_0=s)\left[ 
	R(s')\mbb{E}^{\tau\sim\pi}(\gamma^{\ell(\tau)}|s_0=s, \bar{s}=s') +  \gamma\trans V^\pi(s') 
\right]\\
&=\sum_{s' \in \nrstate} p_{\pi}( \bar{s}=s'|s_0=s)\left[ 
	R(s') \trans +  \gamma\trans V^\pi(s') 
\right],\\
&=\sum_{s' \in \nrstate} p_{\pi}( \bar{s}=s'|s_0=s)\trans\left[ 
	R(s') + \gamma V^\pi(s') 
\right],\label{eq:opt-value}
\end{align}
where $\bar{s}$ is the first rewarding state that agent encounters. Intuitively, $p_{\pi}( \bar{s}=s'|s_0=s)$ means the probability that the $s'$ is the first rewarding state that policy $\pi$ encounters when it starts from $s$.
From~\Cref{eq:option-shortest-path-policy}, our goal is to show:
\begin{align}
\pi^*\in \Pisp = \{\pi \mid \forall(s,s')\in \pathsetpiphi, \trans = \transshort\},
\label{eq:pi-in-pisp-proof}
\end{align}
where $ \transshort=\max_{\pi}\trans$.

We will prove \Cref{eq:pi-in-pisp-proof} by contradiction. Suppose $\pi^*$ is an optimal policy such that $\pi^*\not\in \Pisp$. Then, 
\begin{align}
\exists (\widehat{s}, \widehat{s}'\in \pathsetoptpiphi) &\text{ s.t. } \transopthat \neq \transshorthat.
\end{align}
Recall the definition: $\transshort=\max_{\pi} \trans$. Then, for any $\pi$, the following statement is true.
\begin{align}
\trans \neq \transshort &\leftrightarrow \trans < \transshort.
\end{align}
Thus, we have 
\begin{align}
\transopthat < \transshorthat
\end{align}
Let $\pi_{sp}\in\Pisp$ be a shortest path policy that preserves stochastic dynamics from \Cref{def:mildstochasticity1}.
Then, we have 
\begin{align}
\transopthat < \transshorthat = P^{\pi_{sp}}_{\hat{s}, \hat{s}'}.\label{ineq:PP}
\end{align}
Then, let's compose a new policy $\widehat{\pi}$:
\begin{align}
    \widehat{\pi}(a|s)=
    \begin{cases}
    \pi_{sp}(a|s) & \text{if } \exists \ \tau\in\mc{T}_{\hat{s}, \hat{s}', \text{nr}}^{\pi_{sp}}\text{ s.t. }s\in\tau \\
    \pi^{*}(a|s) & \text{otherwise}
    \end{cases}.\label{def:compose_policy}
\end{align}
Now consider a path $\tau_{\hat{s}\rightarrow \hat{s}'}$ that agent visits $\hat{s}$ at time $t=i$ and transitions to $\hat{s}'$ at time $t=j > i$ while not visiting any rewarding state from $t=i$ to $t=j$ with non-zero probability (\ie, $p_{\pi_{\text{sp}}}(\tau)>0$). We can define a set of such paths as follows:
\begin{align}
    \pathsethathat=\{\tau \mid \exists(i<j), s_i=\hat{s}, s_{j}=\hat{s}', \{s_t\}_{i<t<j}\cap \mc{\nrstate}=\emptyset, p_{\pi_{\text{sp}}}(\tau)>0  \}.
\end{align}

To reiterate the definitions from \Cref{def:shortest-path-constraint}:
$\nrstate=\{s \mid R(s)>0\text{ or }\rho(s)>0 \}$ is the union of all initial and rewarding states,
and $\nrstatepair=\{(s, s') \mid s, s' \in \nrstate, \rho(s)>0, \pathsetpi\neq\emptyset\}$ is the subset of $\nrstate$ such that agent may roll out.

From \Cref{def:mildstochasticity2} and \Cref{def:compose_policy}, the likelihood of a path $\tau$ under policy $\hat{\pi}$ is given as follows:
\begin{align}
    p_{\widehat{\pi}}(\tau)=
    \begin{cases}
    p_{\pi^*}(\tau\in\pathsethathat)p_{\pi_{\text{sp}}}(\tau|\tau\in\pathsethathat) & \text{for }\tau\in\pathsethathat\\
    p_{\pi^*}(\tau) & \text{otherwise}
    \end{cases},\label{eq:likelihood-comp}
\end{align}

where $p_{\widehat{\pi}}(\tau)$ is the likelihood of trajectory $\tau$ under policy $\widehat{\pi}$, $p_{\widehat{\pi}}(\tau\in\pathsethathat)=\int_{\tau\in\pathsethathat}{p_{\widehat{\pi}}(\tau)d\tau}$ ensures the likelihood $\widehat{\pi}(\tau)$ to be a valid probability density function (\ie, $\int p_{\widehat{\pi}}(\tau)d\tau = 1$).
From the path $\tau_{\hat{s}\rightarrow \hat{s}'}$ and $i, j$, we will choose two states $s_{\text{ir}}, s'_{\text{ir}}\sim\tau_{\hat{s}\rightarrow \hat{s}'}$, where
\begin{align}
    s_{\text{ir}}=\max_t (s_t| s_t\in\nrstate, t \leq i),
    \quad 
    s'_{\text{ir}}=\min_t (s_t| s_t\in\nrstate, j\leq t).
\end{align}
Note that such $s_{\text{ir}}$ and $s'_{\text{ir}}$  always exist in $\tau_{\hat{s}\rightarrow \hat{s}'}$ since the initial state and the terminal state satisfy the condition to be $s_{\text{ir}}$ and $s'_{\text{ir}}$.

Then, we can show that the path between $s_{\text{ir}}$ and $s'_{\text{ir}}$ is \tb{not} a shortest-path.
Recall the definition of $\distpi$ (\Cref{def:ex-path-len}):
\begin{align}
D_{\text{nr}}^{\pi^*}(s_{\text{ir}}, s_{\text{ir}}')
&\defeq \log_\gamma \left(\mathbb{E}_{\tau \sim \pi^* :~ \tau \in \mc{T}^{\pi^*}_{\initstir, \laststir, \text{nr}}} \left[ \gamma^{\ell(\tau)} \right]\right)\\
&= \log_\gamma \left(\mathbb{E}_{\tau\sim\pi^*} \underbrace{ \left[
        \gamma^{\ell(\tau)} \mid \tau \in \mathcal{T}^{\pi^*}_{s_{\text{ir}}, s_{\text{ir}}', \text{nr}}
    \right]
    }_\clubsuit
\right)
\end{align}
where we will use $\clubsuit \defeq \gamma^{\ell(\tau)} \mid \tau \in \mathcal{T}^{\pi^*}_{s_\text{ir}, s'_\text{ir}, \text{nr}}$ ~for a shorthand notation.
Then, we have
\begin{align}
\gamma^{D_{\text{nr}}^{\pi^*}(s_{\text{ir}}, s_{\text{ir}}')}
&\defeq \mathbb{E}_{\tau \sim \pi^*} \left[ \clubsuit \right]\\
&=p_{\pi^*}(\tau\in\pathsethathat) \mathbb{E}_{\tau \sim \pi^*} \left[ 
\clubsuit \mid \tau\in\pathsethathat \right]\nn\\
    &\hspace{5em}+p_{\pi^*}(\tau\notin\pathsethathat)\mathbb{E}_{\tau \sim \pi^*} \left[ 
\clubsuit \mid \tau\notin\pathsethathat \right]
\hspace*{5em}{\color{white}.}\\
\text{(From~\Cref{def:sp-policy})}\ 
&<p_{\pi^*}(\tau\in\pathsethathat) \mathbb{E}_{\tau\sim\pi_{\text{sp}}}\left[ \clubsuit \mid \tau\in \pathsethathat\right]\nn\\
        &\hspace{5em}+p_{\pi^*}(\tau\notin\pathsethathat)\mathbb{E}_{\tau\sim\pi^*}\left[ \clubsuit \mid \tau\notin \pathsethathat\right]\\
\text{(From~\Cref{eq:likelihood-comp})}\ 
&=p_{\widehat{\pi}}(\tau\in\pathsethathat) \mathbb{E}_{\tau\sim\widehat{\pi}}\left[ \clubsuit \mid \tau\in \pathsethathat\right]\nn\\
        &\hspace{5em}+p_{\widehat{\pi}}(\tau\notin\pathsethathat)\mathbb{E}_{\tau\sim\widehat{\pi}}\left[ \clubsuit \mid  \tau\notin \pathsethathat\right]\\
&=\mathbb{E}_{\tau \sim \widehat{\pi}} \left[ \clubsuit \right]
=\gamma^{D_{\text{nr}}^{\widehat{\pi}}(s_{\text{ir}}, s_{\text{ir}}')}\label{ineq:lemma}\\
&\Longleftrightarrow D_{\text{nr}}^{\pi^*}(s_{\text{ir}}, s'_{\text{ir}})>D_{\text{nr}}^{\widehat{\pi}}(s_{\text{ir}}, s'_{\text{ir}}) \label{ineq:lemma-final}
\end{align}
where Ineq.~(\ref{ineq:lemma-final}) is given by the fact that $\gamma<1$. Then, $P^{\pi^*}_{s_{\text{ir}}, s_{\text{ir}}'} < P^{\widehat{\pi}}_{s_{\text{ir}}, s_{\text{ir}}'}$.

From~\Cref{eq:opt-value}, we have
\begin{align}
\hspace*{-5em}
V^{\widehat{\pi}}(\initstir)
    &=\sum_{s'\in\nrstate}p_{\widehat{\pi}}( \bar{s}=s' \mid s_0=\initstir)P^{\widehat{\pi}}_{\initstir, s'}\left[R(s')+\gamma V^{\widehat{\pi}}(s')\right]\\
	&=p_{\widehat{\pi}}( \bar{s}=\initstir' \mid  s_0=\initstir)P^{\widehat{\pi}}_{\initstir,\laststir}\left[R(\laststir)+\gamma V^{\widehat{\pi}}(\laststir)\right] \nonumber\\
	    &\hspace{5em} + \sum_{s'\in\nrstate\setminus\laststir}p_{\widehat{\pi}}( \bar{s}=s' \mid  s_0=\initstir)P^{\widehat{\pi}}_{\initstir, s'}\left[R(s')+\gamma V^{\widehat{\pi}}(s')\right]\\
	&=p_{\pi^*}( \bar{s}=\initstir' \mid  s_0=\initstir)P^{\widehat{\pi}}_{\initstir,\laststir}\left[R(\laststir)+\gamma V^{\pi^*}(\laststir)\right] \nonumber\\
	    &\hspace{5em} + \sum_{s'\in\nrstate\setminus\laststir}p_{\pi^*}( \bar{s}=s' \mid s_0=\initstir)P^{\pi^*}_{\initstir, s'}\left[R(s')+\gamma V^{\pi^*}(s')\right]\label{eq:mild2}\\
    &>p_{\pi^*}( \bar{s}=\initstir' \mid  s_0=\initstir)P^{\pi^{*}}_{\initstir,\laststir}\left[R(\laststir)+\gamma V^{\pi^*}(\laststir)\right] \nonumber\\
	    &\hspace{5em} + \sum_{s'\in\nrstate\setminus\laststir}p_{\pi^*}( \bar{s}=s' \mid s_0=\initstir)P^{\pi^*}_{\initstir, s'}\left[R(s')+\gamma V^{\pi^*}(s')\right]\label{ineq:opt-value}\\
	&=\sum_{s'\in\nrstate}p_{\pi^*}( \bar{s}=s' \mid  s_0=\initstir)P^{\pi^{*}}_{\initstir, s'}\left[R(s')+\gamma V^{\pi^{*}}(s')\right]\\
	&=V^{*}(\initstir),
\end{align}
where 
Eq.~(\ref{eq:mild2}) holds from the \textit{mild-stochasticity (1)} and \textit{mild-stochasticity (2)} assumption, and Ineq.~(\ref{ineq:opt-value}) holds because $P^{\widehat{\pi}}_{\initstir, \laststir}>P^{\pi^{*}}_{\initstir, \laststir}$ and $R(s')+\gamma V^{\pi^{*}}(s')>0$ from the non-negative optimal value assumption (See~\Cref{sec:pre}).
Finally, this is a contradiction since the optimal value function $V^{*}(s)$ should be the maximum.
\end{proof}

%% file: 8_appendix_relwork.tex
\cutsectionup
\section{Extended related works}
\label{appendix:extended-related-works}
\cutsectiondown

\paragraph{Approximate state abstraction.}
The approximate state abstraction approaches investigate partitioning an MDP’s state space into clusters of similar states while preserving the optimal solution. Researchers have proposed several state similarity metrics for MDPs.
\citet{dean2013model} proposed to use the bisimulation metrics~\citep{givan2003equivalence,ferns2004metrics}, which measures the difference in transition and reward function.
\citet{bertsekas1988adaptive} used the magnitude of Bellman residual as a metric.
\citet{abel2016near, abel2018state, li2006towards} used the different types of distance in optimal Q-value to measure the similarity between states to bound the sub-optimality in optimal value after the abstraction.
Recently, \citet{castro2019scalable} extended the bisimulation metrics to the approximate version for the deep-RL setting where the tabular representation of state is not available.

Our shortest-path constraint can be seen as a form of state abstraction, in that, ours also aim to reduce the size of MDP (\ie, state and action space) while preserving the ``solution quality''. However, our method does so by removing sub-optimal policies, not by aggregating similar states (or policies).

\paragraph{Connection to Option framework}
Our shortest-path constraint constrains the policy space to a set of shortest-path policies (See~\Cref{def:sp-policy} for definition) between initial and rewarding states. It can be seen as a set of options~\citep{sutton1998between} transitioning between initial and rewarding states. We refer the readers to Appendix~\ref{appendix:option-framework} for the detailed description of the option framework-based formulation of our framework.